\documentclass[twoside,11pt]{article}

\usepackage[preprint]{jmlr2e}
\hypersetup{colorlinks=true,allcolors=blue}
\hypersetup{
  pdftitle={Near-optimal Delta-convex Estimation of Lipschitz Functions},
  pdfauthor={G\'abor Bal\'azs},
  pdfkeywords={nonparametric regression, Lipschitz function, squared loss, minimax rate, function approximation, delta-convex function, empirical risk minimization}
}

\usepackage{amsmath}
\usepackage{ntheorem}
\usepackage[capitalize,noabbrev]{cleveref}

\usepackage[all]{hypcap}
\usepackage[multiple]{footmisc}

\usepackage{algorithm}
\usepackage{algorithmic}
\makeatletter
\def\cleartheorem#1{\expandafter\let\csname#1\endcsname\relax
    \expandafter\let\csname c@#1\endcsname\relax
}
\makeatother
\cleartheorem{theorem}
\cleartheorem{lemma}
\cleartheorem{corollary}
\newtheorem{theorem}{Theorem}
\newtheorem{lemma}[theorem]{Lemma} 
\newtheorem{corollary}[theorem]{Corollary}

\newcommand{\defeq}{\doteq} 
\newcommand{\E}{\mathbb{E}} 
\newcommand{\K}{\mathbb{K}} 
\newcommand{\argmin}{\mathop{\mathrm{argmin}}}
\newcommand{\argmax}{\mathop{\mathrm{argmax}}}
\newcommand{\Ordo}{O} 
\newcommand{\as}{a.s.} 
\newcommand{\wrt}{w.r.t.} 
\newcommand{\iid}{i.i.d.} 
\newcommand{\norm}[1]{\lVert#1\rVert} 
\newcommand{\T}{\top} 
\newcommand{\ASSIGN}{\leftarrow} 
\newcommand{\ind}{\mathbb{I}} 
\newcommand{\event}{\mathcal{E}} 
\newcommand{\Prob}{\mathbb{P}} 

\newcommand{\setN}{\mathbb{N}} 
\newcommand{\setR}{\mathbb{R}} 
\newcommand{\setF}{\mathcal{F}} 
\newcommand{\setG}{\mathcal{G}} 
\newcommand{\setH}{\mathcal{H}} 
\newcommand{\setX}{\mathcal{X}}
\newcommand{\setY}{\mathcal{Y}}
\newcommand{\setXhat}{\hat{\setX}}
\newcommand{\setXtilde}{\tilde{\setX}}
\newcommand{\setC}{\mathcal{C}} 
\newcommand{\setI}{\mathcal{I}} 
\newcommand{\setT}{\mathcal{T}} 
\newcommand{\setB}{\mathcal{B}}
\newcommand{\setZ}{\mathcal{Z}}
\newcommand{\setP}{\mathcal{P}}
\newcommand{\setM}{\mathcal{M}}

\renewcommand{\vec}[1]{{\boldsymbol{#1}}} 
\newcommand{\vzero}{\vec{0}} 
\newcommand{\vone}{\vec{1}} 
\newcommand{\vx}{\vec{x}}
\newcommand{\vX}{\vec{X}}
\newcommand{\vxhat}{\hat{\vx}}
\newcommand{\vxtilde}{\tilde{\vx}}
\newcommand{\vXhat}{\vec{\hat{X}}}
\newcommand{\vXtilde}{\vec{\tilde{X}}}
\newcommand{\vXbar}{\vec{\bar{X}}}
\newcommand{\Ybar}{\bar{Y}}
\newcommand{\vu}{\vec{u}}
\newcommand{\vv}{\vec{v}}
\newcommand{\ve}{\vec{e}}
\newcommand{\vw}{\vec{w}}
\newcommand{\vz}{\vec{z}}
\newcommand{\vtheta}{\vec{\theta}}

\newcommand{\data}{{\mathcal{D}_n}} 
\newcommand{\lip}{\lambda} 
\newcommand{\lipf}{\tilde\lambda} 
\newcommand{\somenorm}{p} 
\newcommand{\phin}{\phi_\somenorm} 
\newcommand{\setFn}{\setF_\somenorm} 
\newcommand{\lipphi}{\lip_{\phin}} 
\newcommand{\Eapprox}{E_{\textrm{approx}}} 
\newcommand{\reg}{z} 
\newcommand{\regz}{\mathcal{R}} 
\newcommand{\risk}{\mathcal{L}} 
\newcommand{\fni}{f_n^+} 
\newcommand{\fnih}{\hat{f}_n^+} 
\newcommand{\inds}{\setI_n^+} 
\newcommand{\Khat}{\hat{k}} 
\newcommand{\proj}{\pi} 
\newcommand{\capx}{\tilde{\tau}_\somenorm} 
\newcommand{\clip}{\tau_\somenorm} 
\newcommand{\cphi}{\tau_{\phin}} 
\newcommand{\offsetc}{\tilde{c}_0} 
\newcommand{\rrho}{r_{\hspace{-0.5mm}\rho}}
\newcommand{\rsigma}{r_{\hspace{-0.5mm}\sigma}}
\newcommand{\setFbar}{\overline{\setF}} 
\newcommand{\setGbar}{\overline{\setG}} 
\newcommand{\bnd}{\beta} 
\newcommand{\fniapprox}{\hat{f}_n^{++}}
\newcommand{\setFcvx}{\setF^{\textrm{cvx}}} 

\usepackage{lastpage}
\jmlrheading{27}{2026}{1-\pageref{LastPage}}{11/25; Revised 4/26}{5/26}{25-2864}{G\'abor Bal\'azs}
\ShortHeadings{Near-optimal Delta-convex Estimation of Lipschitz Functions}{Bal\'azs}
\firstpageno{1}

\usepackage{fancyhdr}
\fancypagestyle{firstpage}{%
  \fancyhf{}%
  \fancyhead[L]{\footnotesize Published in \href{https://jmlr.org/papers/v27/25-2864.html}{JMLR 27(134):1-41, 2026}}%
}
\begin{document}

\title{Near-optimal Delta-convex Estimation of Lipschitz Functions}

\author{\name G\'abor Bal\'azs \email gabalz@gandg.ai \\
       \textsc{G\&G}, Cartagena, Spain}

\editor{Aryeh Kontorovich}

\maketitle

\begin{abstract}
  This paper presents a tractable algorithm for estimating an unknown Lipschitz function from noisy observations and establishes an upper bound on its convergence rate.
  The approach extends max-affine methods from convex shape-restricted regression to the more general Lipschitz setting.
  A key component is a nonlinear feature expansion that maps max-affine functions into a subclass of delta-convex functions, which act as universal approximators of Lipschitz functions while preserving their Lipschitz constants.
  Leveraging this property, the estimator attains the minimax convergence rate (up to logarithmic factors) with respect to the intrinsic dimension of the data under squared loss and subgaussian distributions in the random design setting.
  The algorithm integrates adaptive partitioning to capture intrinsic dimension, a penalty-based regularization mechanism that removes the need to know the true Lipschitz constant, and a two-stage optimization procedure combining a convex initialization with local refinement.
  The framework is also straightforward to adapt to convex shape-restricted regression.
  Experiments demonstrate competitive performance relative to other theoretically justified methods, including nearest-neighbor and kernel-based regressors.
\end{abstract}


\section{Introduction}

This paper considers the fundamental problem of estimating an unknown regression function from noisy observations in the random design setting.
Suppose we observe $n$ independent and identically distributed (\iid)~samples, $\data \defeq \langle (\vX_i,Y_i) : i \in [n] \rangle$, for an unknown real-valued regression function $f_* : \setX_* \to \setR$ on some unknown domain $\setX_* \subseteq \setR^d$, such that
\begin{equation} \label{eq:data-model}
  \vX_i \in \setX_* \textrm{ almost surely (\as)}, \quad\qquad
  Y_i \defeq f_*(\vX_i) + \xi_i .
\end{equation}
The noise $\xi_i$ is centered, satisfying $\E[\xi_i|\vX_i] = 0$ \as~for all $i \in [n]$, where $[m] \defeq \{1,\ldots,m\}$ for any positive integer $m$.
We assume that the regression function $f_*$ is $\lip_*$-Lipschitz on $\setX_*$ with respect to (\wrt)~the Euclidean norm $\norm{\cdot}_2$ for some unknown Lipschitz constant $\lip_* \in (0,\infty)$.
We evaluate the estimators using the excess risk under squared loss, for which the minimax (convergence) rate is known to be $\Theta\big(n^{-2/(2+d_*)}\big)$ in terms of the sample size $n$ and the intrinsic data dimension $d_*$ \citep{Stone1982}.
Throughout the paper, we use the standard asymptotic order of growth notations: $\Omega(\cdot)$, $\Theta(\cdot)$, and $\Ordo(\cdot)$.

In convex (shape-restricted) regression \citep[e.g.,][]{Lim2014,HanWellner2016,Balazs2016,KurEtAl2020}, the regression function $f_*$ is known to be convex on a convex domain~$\setX_*$, and the goal is to estimate $f_*$ by a convex function.
In this setting, it is common to choose the estimator from the class of max-affine functions (functions defined as the maximum of finitely many affine functions) because they approximate any convex function at the worst-case optimal rate \citep{BalazsGyorgySzepesvari2015}.
Moreover, empirical risk minimization over max-affine functions using $n$ hyperplanes can be reformulated as a tractable convex optimization problem (solvable in polynomial time \wrt\ $d$ and $n$; \citealp[Section~6.5.5]{BoydVandenberghe2004}).
Although several extensions of max-affine functions were proposed \citep{BagirovEtAl2010,SunYu2019,SiahkamariEtAl2020}, none have been shown to achieve the minimax rate up to logarithmic factors (i.e., near-minimax rate) in the general Lipschitz setting of~\eqref{eq:data-model}.
In this paper, we fill this gap by using the following extension of max-affine functions:
\begin{equation} \label{eq:dcf-funcs1} \begin{split}
  \setFn(\setXhat) \defeq
  \Big\{f : \setR^d \to \setR \,\big|\, &f(\vx) = \max_{k\in[k_0]}b_{f,k} + \vu_{f,k}^\T(\vx-\vxhat_k) + v_{f,k}\norm{\vx-\vxhat_k}_\somenorm , \\
  & \vx\in\setR^d ,\, b_{f,k} \in \setR ,\, \vu_{f,k} \in \setR^d ,\, v_{f,k} \in \setR ,\, k\in[k_0]\Big\} ,
\end{split} \end{equation}
where $\setXhat \defeq \{\vxhat_1,\ldots,\vxhat_{k_0}\} \subset \setR^d$ is a nonempty, finite set of size $k_0 \defeq |\setXhat|$, and $\norm{\cdot}_\somenorm$ is the usual $p$-norm on $\setR^d$ for $p \in \{1, 2, \infty\}$.

The key observation, based on the function representation in the extension theorem of \citet{McShane1934}, is that when $\setXhat$ is chosen to be an $\epsilon$-cover of the covariate data $\setX_n \defeq \{\vX_1,\ldots,\vX_n\}$ \wrt~$\norm{\cdot}_2$, there exists a function within $\setFn(\setXhat)$ for any $p \in \{1,2,\infty\}$ that is uniformly $\Ordo(\epsilon)$-close to the Lipschitz regression function $f_*$ on $\setX_n$ (see \cref{thm:approxL}).
We use this fact to bound the approximation error of our estimators to $f_*$ on $\setX_n$, where the estimators ``approximately'' minimize the empirical risk over the training data $\data$ within $\setFn(\setXhat)$.
A tradeoff arises in selecting the size of the cover $\setXhat$: increasing $|\setXhat|$ improves approximation accuracy (i.e., smaller $\epsilon$) but results in a more complex representation (i.e., more parameters), and vice versa.
To balance this tradeoff, we construct $\setXhat$ using the adaptive farthest-point clustering algorithm of \citet{Balazs2022}, which achieves $|\setXhat| \approx n^{d_*/(2+d_*)}$ and $\epsilon \approx n^{-1/(2+d_*)}$.
Our main result, stated in \cref{thm:near-minimax-rate}, shows that this choice of $\setXhat$ together with the class $\setFn(\setXhat)$ yields estimators that achieve a near-minimax rate under \eqref{eq:data-model} for subgaussian distributions.

We predominantly use the Euclidean norm $\norm{\cdot}_2$ with the shorthand notation $\norm{\cdot}$.
We only allow the norm to vary in the estimator design in \eqref{eq:dcf-funcs1}, using this flexibility to highlight connections between our method and two other techniques in the literature.
With the max-norm ($\somenorm = \infty$), the elements of the set $\{ f \in \setF_\infty(\setXhat) : v_{f,k} \le 0, k\in[k_0]\}$ are max-min-affine functions (functions defined as the maximum of minima of finitely many affine functions).
Since \citet{Ovchinnikov2002} showed that max-min-affine functions can represent any continuous piecewise-linear function, various max-min-affine estimators have been developed \citep{BagirovEtAl2010,TorielloVielma2012,BagirovEtAl2022}.
However, to the best of our knowledge, none of these come with theoretical guarantees.
In \cref{sec:max-min-affine}, we address this gap with a tractable max-min-affine estimator based on $\setF_\infty(\setXhat)$, to which the near-minimax guarantee of \cref{thm:near-minimax-rate} applies.
In \cref{sec:dcf-alg}, we further discuss an extension to $\setF_1(\setXhat)$, which can be computed by maxout neural networks \citep{GoodfellowEtAl2013}.

Our algorithm regularizes the uniform Lipschitz constant of the estimator and does not require knowledge of the Lipschitz constant $\lip_*$ of the regression function $f_*$.
In \cref{sec:convex-regression}, we further show how our method can be easily adapted to the convex regression setting.
Although \citet{BlanchetEtAl2019} proposed a similar uniform Lipschitz regularization for convex regression and proved a convergence rate in probability, their result only holds for $d > 4$ and for sufficiently large $n$ satisfying $\ln(n) \ge \lip_*$.
\citet{Lim2025} extended this result to $d \le 4$, but it still only provides a convergence rate in probability as $n \to \infty$.
In contrast, we establish a probably approximately correct (PAC) bound that holds for all $n \ge 2$.
To the best of our knowledge, this adapted variant of our estimator is the first tractable method for convex regression to enjoy a PAC guarantee in the random design setting (albeit not necessarily near-minimax) without requiring knowledge of $\lip_*$ or a uniform bound on $f_*$.

The result of \citet[Section~III.2]{HiriartUrruty1985} shows that the class $\setFn(\setXhat)$ lies within the class of delta-convex (DC) functions, whose elements can be expressed as the difference of two convex functions \citep{Hartman1959}.
The classes of max-min-affine, weakly max-affine, and delta-max-affine functions are also contained within the class of DC functions, and all of them have been studied for estimator design \citep[e.g.,][]{BagirovEtAl2010,SunYu2019,SiahkamariEtAl2020}.
However, none of these approaches have achieved a near-minimax rate guarantee in the setting of \eqref{eq:data-model}.
In \cref{sec:max-min-affine,sec:weakly-convex}, we provide approximation results for all of these classes, which may be of independent interest.
To achieve the near-minimax rate in \cref{thm:near-minimax-rate}, we work with $\setFn(\setXhat)$ for two reasons.
First, empirical risk minimization (ERM) over $\setFn(\setXhat)$ leads to a tractable convex optimization problem, as discussed in \cref{sec:analysis}.
In contrast, we are not aware of any tractable ``approximation'' of the ERM problem over the entire class of max-min-affine functions.
Second, for the worst-case optimal approximation $f \in \setFn(\setXhat)$ to a $\lip_*$-Lipschitz regression function $f_*$, the parameter magnitudes $\max_{k\in[k_0]}\max\{\norm{\vu_{f,k}},|v_{f,k}|\}$ are provably bounded above by $\Ordo(\lip_*)$, as shown in \cref{thm:approxL}.
Importantly, this upper bound does not depend on the approximation accuracy.
Since we cannot establish a similar bound for weakly max-affine and delta-max-affine functions, our proof technique does not apply to those cases in general (i.e., without further assuming smoothness of $f_*$).

Finally, we note that several other methods have been shown to achieve a near-minimax rate in the regression setting of \eqref{eq:data-model} with respect to the intrinsic data dimension.
Specifically, these are the nearest-neighbor estimator \citep[Corollary~3]{KulkarniPosner1995}, certain tree-based predictors \citep[Theorem~9]{KpotufeDasgupta2012}, and the Nadaraya-Watson estimator \citep[Theorem~21]{Kpotufe2010}.
Unlike typically discontinuous partitioning estimators, such as nearest-neighbor and tree-based methods, our estimator is always continuous.
While the Nadaraya-Watson estimator also produces a continuous function by forming a weighted average over the entire sample using a continuous kernel, its structural behavior differs from ours.
In particular, the max operator underlying the representation of functions in $\setFn(\setXhat)$ induces a partition of $\setR^d$, making our estimator more closely resemble partitioning regression methods than kernel-based approaches such as Nadaraya-Watson.
Moreover, in contrast to our algorithm and to certain tree-based methods \citep[Theorem~4]{Kpotufe2009b}, both partitioning and kernel regression techniques generally require external model selection procedures \citep[e.g.,][]{BartlettEtAl2002} such as sample splitting or cross-validation \citep[e.g.,][Chapters~7 and 8]{GyorfiEtAl2002} to adapt to the intrinsic dimension and attain near-optimal guarantees.
Although a comprehensive empirical evaluation is beyond the scope of this paper, we include experimental comparisons with these methods in \cref{sec:experiments}.

\section{The Proposed Algorithm}
\label{sec:dcf-alg}

Define the feature vector $\phin(\vx,\vxhat) \defeq \big[(\vx-\vxhat)^\T\,\,\norm{\vx-\vxhat}_\somenorm\big]^\T \in \setR^{d+1}$ for all $\somenorm \in \{1,2,\infty\}$ and $\vx,\vxhat\in\setR^d$. Then, for any nonempty, finite set $\setXhat \defeq \{\vxhat_1,\ldots,\vxhat_{k_0}\} \subset \setR^d$, we can rewrite \eqref{eq:dcf-funcs1} in the compact form:
\begin{equation*} \begin{split}
  \setFn(\setXhat) = \big\{f : \setR^d \to \setR \,\big|\,
  & f(\vx) = \max_{k\in[k_0]}b_{f,k} + \vw_{f,k}^\T\phin(\vx,\vxhat_k), \\
  & \vx\in\setR^d,\, b_{f,k}\in\setR ,\, \vw_{f,k} \in \setR^{d_\somenorm} ,\, k \in [k_0]
  \big\}
\end{split} \end{equation*}
where $d_\somenorm \defeq d+1$.
This formulation allows us to analogously introduce the set $\setF_+(\setXhat)$ by defining $\phi_+(\vx,\vxhat) \defeq \big[(\vx-\vxhat)_+^\T\,\,(\vxhat-\vx)_+^\T\big]^\T \in \setR^{d_+}$ and $d_+ \defeq 2d$, where the ReLU operation $(\vz)_+ \defeq \max\{\vzero_d,\vz\}$ is applied elementwise to any vector $\vz \in \setR^d$, with $\vzero_d$ denoting the zero vector of size $d$.

The functions of $\setF_+(\setXhat)$ can be computed by the so-called maxout neural networks \citep{GoodfellowEtAl2013}.
Further, $\setF_1(\setXhat) \subset \setF_+(\setXhat)$ holds because $\vone_{2d}^\T\,\phi_+(\vx,\vxhat) = \norm{\vx-\vxhat}_1$ and $[\vw^\T\,-\vw^\T] \phi_+(\vx,\vxhat) = \vw^\T(\vx - \vxhat)$ for all $\vx,\vxhat,\vw \in \setR^d$, where $\vone_s$ denotes the all-ones vector of size~$s$.
In the paper, we consider $\setFn(\setXhat)$ for all $\somenorm \in \{1, 2, \infty, +\}$, and discuss their approximation guarantees for Lipschitz functions in \cref{sec:approxL}.

Let $k_0 \in \setN$, and $\{\setC_1(\setXhat),\ldots,\setC_{k_0}(\setXhat)\}$ denote the Voronoi partition of the entire space $\setR^d$ induced by the center points $\setXhat \defeq \{\vxhat_1,\ldots,\vxhat_{k_0}\} \subset \setR^d$ \wrt~the Euclidean norm $\norm{\cdot}$.
Formally, define $\setC_k(\setXhat) \defeq \big\{\vx\in\setR^d : \norm{\vx-\vxhat_k} = \min_{\vxhat\in\setXhat}\norm{\vx - \vxhat}\big\}$ for all $k\in[k_0]$, with ties broken arbitrarily but consistently (e.g., by selecting the center with the smaller index).

\subsection{Delta-convex Fitting (DCF)}
\label{sec:dcf}

Fix $p \in \{1,2,\infty,+\}$ and suppose we are given the training data $\data$ from \eqref{eq:data-model}.
First, we use the adaptive farthest-point clustering (AFPC; see \cref{alg:AFPC} below) method of \citet{Balazs2022} to compute a finite set of distinct center points $\setXhat_K \defeq \{\vXhat_1,\ldots,\vXhat_K\}\subseteq \setX_n$ for some $K \in [n]$.
Then, the core of our algorithm is the following convex optimization problem:
\begin{equation} \label{eq:erm} \begin{split}
  &\min_{\substack{\reg \in \setR, \\ b_1,\ldots,b_K \in \setR, \\ \vw_1,\ldots,\vw_K \in \setR^{d_\somenorm}}} \hspace{-3mm} \theta_1 \reg^2 + \hspace{-2mm}\sum_{k\in[K]} \hspace{-1mm}\theta_2\norm{\vw_k}^2 + \frac1n \sum_{i\in[n]}\ind\big\{\vX_i \in \setC_k(\setXhat_K)\big\}\Big(b_k + \vw_k^\T\phin(\vX_i,\vXhat_k) - Y_i\Big)^2 \\
  &\qquad \textrm{such that }
  b_k \ge b_l + \vw_l^\T\phin(\vXhat_k,\vXhat_l) ,\quad \norm{\vw_k} \le \reg + \theta_0 ,\quad k,l\in[K] .
\end{split} \end{equation}
where $\ind\{\cdot\}$ is the $\{0,1\}$-valued indicator function, and $\theta_0, \theta_1, \theta_2 \ge 0$ are fixed regularization parameters.
In particular, the role of $\theta_0$ is to mitigate the effect of conservative regularization on the uniform Lipschitz constant of the estimator (i.e., the $\theta_1\reg^2$ term).
Let $(\reg_n,\langle(b_{n,k},\vw_{n,k}) : k\in[K]\rangle)$ be a solution to \eqref{eq:erm}, and define the (initial) estimator as $f_n(\vx) \defeq \max_{k\in[K]}b_{n,k} + \vw_{n,k}^\T\phin(\vx,\vXhat_k)$ for all $\vx\in\setR^d$.
Clearly, $f_n \in \setFn(\setXhat_K)$.

Denote the empirical risk of any function $f : \setR^d \to \setR$ by $\risk_n(f) \defeq \frac1n\sum_{i\in[n]}\big(f(\vX_i)-Y_i\big)^2$.
Additionally, for all $f \in \setFn(\setXhat_K)$ and $c_0,c_1,c_2 \ge 0$, define the regularization term
\[
  \regz_{c_0,c_1,c_2}(f) \defeq c_1\max_{k\in[K]}\big(\norm{\vw_{f,k}} - c_0\big)_+^2 + c_2\sum_{k\in[K]}\norm{\vw_{f,k}}^2 ,
\]
and the largest slope parameter magnitude by $\lipf_f \defeq \max_{k\in[K]}\norm{\vw_{f,k}}$.
We then define $\fnih$ to be a refinement of $f_n$ in the sense that
\begin{equation} \label{eq:erm-local} \begin{split}
  \fnih \in \big\{f \in \setFn(\setXhat_K) : \risk_n(f) + \regz_n(f) \le \risk_n(f_n) + \regz_n(f_n)\big\}
  ,\quad \regz_n(\cdot) \defeq \regz_{\theta_3\lipf_{f_n},\theta_{f_n},\theta_2}(\cdot) ,
\end{split} \end{equation}
where $\theta_{f_n} \defeq \lipf_{f_n}^{-2}\big(\risk_n(f_n) + \regz_{0,0,\theta_2}(f_n)\big)$ if $\lipf_{f_n} > 0$, and $\theta_3 \ge 1$ is a fixed regularization parameter.
Note that $\lipf_{f_n} = \max_{k\in[K]}\norm{\vw_{n,k}}$ by definition.
In the degenerate cases, when $\lipf_{f_n} = 0$ or $\theta_{f_n} = 0$, we set $\fnih = f_n$ and $\theta_{f_n} = 0$.

To define the final estimator, we prune all parameters of $\fnih$ which do not affect its empirical risk, and center its average response over $\setX_n$.
Formally, we define the final estimator for all $\vx\in\setR^d$ as
\begin{equation} \label{eq:dcf-estimator} \begin{split}
  \fni(\vx) &\defeq C_n^+ + \max_{k\in\inds} b_{\fnih,k} + \vw_{\fnih,k}^\T\phin(\vx,\vXhat_k) , \qquad \inds \defeq \setI_n(\fnih) ,
  \\
  \setI_n(f) &\defeq \big\{k\in[K] : f(\vX_i) = b_{f,k} + \vw_{f,k}^\T\phin(\vX_i,\vXhat_k) \textrm{ for some } i\in[n]\big\} ,
\end{split} \end{equation}
where the index set $\setI_n(\cdot)$ is defined for all $f \in \setFn(\setXhat_K)$, and the centering constant $C_n^+$ is given by $C_n^+ \defeq \Ybar - \frac1n\sum_{i=1}^n\fnih(\vX_i)$ with $\Ybar \defeq \frac1n\sum_{i=1}^n Y_i$.

We call our algorithm delta-convex fitting (DCF) and summarize it in \cref{alg:DCF}.
For completeness, \cref{alg:AFPC} also presents the AFPC method of \citet{Balazs2022}.
To describe it, let the covariate data radius be defined as $R_{\setX_n} \defeq \max_{i\in[n]}\norm{\vX_i-\vXbar}$ with $\vXbar \defeq \frac1n\sum_{i\in[n]}\vX_i$.
Define the clustering objective as $\epsilon_n(\setXhat) \defeq \max_{\vX\in\setX_n}\min_{\vXhat\in\setXhat}\norm{\vX-\vXhat}$, and the partition size limit as $\Khat(\setXhat) \defeq \min\big\{n (\epsilon_n(\setXhat)/R_{\setX_n})^2, n^{d/(2+d)}\big\}$ for any set of center points $\setXhat \subseteq \setX_n$.

\begin{center}
\begin{minipage}{0.44\textwidth}
\begin{algorithm}[H]
  \caption{$\setXhat \defeq \textbf{AFPC}(\setX_n)$}
  \label{alg:AFPC}
  \begin{algorithmic}[1]
    \STATE $\setXhat \ASSIGN \{\vXhat\}$ with arbitrary $\vXhat \in \setX_n$
    \WHILE{$|\setXhat| < \Khat(\setXhat)$}
    \STATE \mbox{$\vXtilde \hspace{-1mm}\in\hspace{-0.5mm} \argmax_{\vX\in\setX_n}\hspace{-0.5mm}\min_{\vXhat\in\setXhat}\norm{\hspace{-0.25mm}\vX\hspace{-0.75mm}-\hspace{-0.75mm}\vXhat\hspace{-0.25mm}}$}\\
    \STATE $\setXhat \ASSIGN \setXhat \cup \{\vXtilde\}$
    \ENDWHILE
    \RETURN{$\setXhat$}
  \end{algorithmic}
\end{algorithm}
\end{minipage}
\hfill
\begin{minipage}{0.53\textwidth}
\begin{algorithm}[H]
  \caption{$\fni \defeq \textbf{DCF}(\data, \phin)$}
  \label{alg:DCF}
  \begin{algorithmic}[1]
    \STATE $\setXhat_K \defeq \textbf{AFPC}(\setX_n)$, $K \defeq |\setXhat_K|$
    \STATE $(\reg_n, \hspace{-0.5mm}\{\hspace{-0.5mm}(b_{n,k},\vw_{n,k}) \hspace{-0.5mm}:\hspace{-0.5mm} k\in[K]\}\hspace{-0.5mm}) \hspace{-0.25mm}\ASSIGN\hspace{-0.25mm}$ solution to \eqref{eq:erm}, using $\data$, $\phin$, $\setXhat_K$, and $\theta_0, \theta_1, \theta_2$
    \STATE $f_n(\cdot) \defeq \max_{k\in[K]} b_{n,k}+\vw_{n,k}^\T\phin(\,\cdot\,,\vXhat_k)$
    \STATE $\fni \ASSIGN $ refinement of $f_n$ via \eqref{eq:erm-local} and \eqref{eq:dcf-estimator}, \\ using $\data$, $\phin$, $\setXhat_K$, $f_n$, and $\theta_0, \theta_1, \theta_2, \theta_3$
    \RETURN $\fni$
  \end{algorithmic} 
\end{algorithm}
\end{minipage}
\end{center}

We consider feature maps $\phin$ for all $\somenorm \in \{1,2,\infty,+\}$, and analyze the DCF algorithm under the following choice of regularization parameters:
\begin{equation} \label{eq:dcf-params} \begin{aligned}
  0 \le \theta_0 &= \Ordo\big((R_{\setY_n}/\max\{1,R_{\setX_n}\})\ln(n)\big) , &
  \theta_1 &= \Theta\big(\max\big\{1, R_{\setX_n}^2\big\} \, (d K / n)\big) , \\
  0 \le \theta_2 &\le \theta_1 / K , &
  1 \le \theta_3 &= \Ordo\big(\ln(n)\big),
\end{aligned} \end{equation}
where the response data radius is defined as $R_{\setY_n} \defeq \max_{i\in[n]}|Y_i - \Ybar|$ with $\setY_n \defeq \{Y_1,\ldots,Y_n\}$, and $K$ is the size of the AFPC-computed partition as defined in \cref{alg:DCF}.

\subsection{Theoretical Guarantees of DCF}
\label{sec:near-minimax-rate}

Let $\setB(\vx_0,r) \defeq \{\vx\in\setR^d : \norm{\vx-\vx_0} \le r\}$ denote the closed ball in $\setR^d$ centered at $\vx_0 \in \setR^d$ with radius $r > 0$.
We write $Z \sim P$ to indicate that the random variable $Z$ is sampled from the distribution~$P$.
We consider the statistical model \eqref{eq:data-model}, where the regression function $f_* : \setX_* \to \setR$ is $\lip_*$-Lipschitz on its domain $\setX_* \subseteq \setR^d$.

Let $d_\circ$ denote the \emph{doubling dimension} of $\setX_*$ \citep[e.g.,][]{Gupta2003}.
That is, $d_\circ$ is the smallest number such that
for any $\vx \in \setR^d$ and $r > 0$, the set $\setB(\vx,r) \cap \setX_*$ can be covered by the union of at most $2^{d_\circ}$ balls of radius $r/2$.
Since $d_\circ$ can exceed $d$ by a constant factor,\footnote{It is known that exactly 7 discs of radius 1/2 are needed to cover the unit disc \citep[e.g.,][Section~I.2]{Zahn1962}, so the doubling dimension for any set of positive area in $\setR^2$ is at least $\log_2(7) > 2$.} we define the intrinsic dimension as $d_* \defeq \min\{d_\circ, d\}$ to ensure that the convergence rate is bounded by $n^{-2/(2+d)}$ in the worst case when $d < d_\circ$.
On the other hand, the doubling dimension $d_\circ$ (and thus $d_*$) can be significantly smaller than $d$, helping to mitigate the curse of dimensionality.
Practical examples where this occurs include affine subspaces, Riemannian manifolds \citep[Theorem~22]{DasguptaFreund2008}, sparse data, and unions of these \citep[Lemmas~3 and 4]{KpotufeDasgupta2012}.

\cref{thm:near-minimax-rate} presents the main result of the paper, providing an adaptive near-minimax rate for DCF estimators with respect to (\wrt)~the intrinsic dimension~$d_*$.
\begin{theorem} \label{thm:near-minimax-rate}
  For $n \ge 2$, consider the estimation problem \eqref{eq:data-model}, where the $n$ \iid~samples $\data$ are drawn from an unknown distribution $P_*$ over $\setX_* \times \setR$, and the regression function $f_*$ is $\lip_*$-Lipschitz on $\setX_*$ \wrt\ $\norm{\cdot}$.
  Suppose the covariate and noise distributions are subgaussian with parameters $\rho, \sigma > 0$ such that
  \begin{equation} \label{eq:subgaussian} \begin{split}
    \E\big[e^{\norm{\vX-\E[\vX]}^2/\rho^2}\big] \le 2 \,,\qquad
    \E\big[e^{(f_*(\vX)-Y)^2/\sigma^2}\big|\vX\big] \le 2 \,\,\, \textrm{\as} \,,\qquad
    (\vX, Y) \sim P_* .
  \end{split} \end{equation}
Let $\somenorm \in\{1,2,\infty,+\}$, and $\fni$ be the DCF estimator computed by \cref{alg:DCF} using regularization parameters satisfying \eqref{eq:dcf-params}.
Then, for all $\gamma \in (0,1)$, it holds with probability at least $1-\gamma$ \wrt\ the randomness of the sample $\data$ and the estimator (i.e., choosing the initial point of AFPC) that
\begin{equation*}
  \E_{(\vX,\cdot) \sim P_*}\big[\big(\fni(\vX) - f_*(\vX)\big)^2\big] = \Ordo\Big( d\big(1+d\,\ind\{\somenorm \ne 2\}\big) \, n^{-2/(2+d_*)} \bnd \Big) ,
\end{equation*}
where $\bnd \defeq \theta_3^2 \big(1+\rho^2\ln(n/\gamma)\big) \big(\lip_*^2 + \sigma^2\ln(\bnd_{\ln}/\gamma)\big) \ln^3(n) \ln(dn/\gamma)$, and $\bnd_{\ln} \defeq n (1+\lip_*\rho/\sigma)$.
\end{theorem}
\begin{proof}
  See \cref{sec:proof-near-minimax-rate}.
\end{proof}

\citet[Theorem~1]{Stone1982} showed that the minimax rate for the estimation problem~\eqref{eq:data-model} under the squared loss is $\Theta(n^{-2/(2+d_*)})$ whenever $[0,1]^{d_*}\times\{0\}^{d-d_*} \subseteq \setX_*$.
Therefore, the convergence rate established in \cref{thm:near-minimax-rate} is near-minimax, since $\bnd = \Ordo\big(\ln^8(n)\big)$ with $\theta_3 = \Ordo\big(\ln(n)\big)$.
Furthermore, our result provides a PAC bound, which can be converted into an expectation bound via integration \citep[e.g.,][Eq.~2]{BalazsGyorgySzepesvari2016}.

The DCF algorithm runs in polynomial time with respect to both $d$ and $n$.
Since the AFPC algorithm differs from farthest-point clustering only in its stopping rule (using $|\setXhat| < \Khat(\setXhat)$ in \cref{alg:AFPC} instead of a fixed cardinality threshold), it can be computed in $\Ordo(dKn)$ time \citep{Gonzalez1985,HochbaumShmoys1985}.
Constructing the second-order cone program (SOCP) of \eqref{eq:erm} takes $\Ordo(d^2 n + K^2)$ time and yields a problem with $\Ordo(dK)$ variables and $\Ordo(K^2)$ constraints.
Its solution can be approximated to accuracy $\delta > 0$ using interior-point methods \citep[e.g.,][Section~6.2]{NesterovNemirovskii1994}, yielding a conservative worst-case runtime $\Ordo\big(d^2n + d^2 K^5 \ln(K/\delta)\big)$, ignoring sparsity.
Lastly, the refinement step \eqref{eq:erm-local} is optional with respect to theoretical guarantees (i.e., one can always use $\fnih = f_n$) and can be performed to a desired accuracy using smoothing techniques \citep{Nesterov2005} in a tractable manner.
Evaluating the DCF estimators $f_n$ and $\fni$ at a new point for inference takes $\Ordo(dK)$ and $\Ordo(d|\setI_n^+|)$ time, respectively, where $|\setI_n^+| \le K$.

\subsection{Discussion of DCF and Its Guarantees}
\label{sec:dcf-discuss}

The minimization problem in \eqref{eq:erm} is similar to the APCNLS algorithm of \citet{Balazs2022}, which is designed for training max-affine estimators.
Like APCNLS, our approach trains a partitioning estimator and maps it into a target function class, $\setFn(\setXhat_K)$ in our case, instead of the class of max-affine functions.
A key distinction, however, is that we impose the constraints only at the center points, resulting in just $K^2$ constraints, which is considerably fewer than the $nK$ constraints used in APCNLS.
As shown in \cref{sec:dcf-derivation}, despite this reduced constraint set, our algorithm preserves theoretical properties analogous to those of APCNLS.

The stopping condition of AFPC (\cref{alg:AFPC}) ensures that $K-1 < \Khat(\setXhat_{K-1})$ and $\Khat(\setXhat_K) \le K$.
Using this, \citet[Lemma~4.2]{Balazs2022} showed that AFPC guarantees both a complexity bound of $K = \Ordo\big(n^{d_*/(2+d_*)}\big)$ and an accuracy bound of $\epsilon_n^2(\setXhat_K) = \Ordo(K/n) = \Ordo\big(n^{-2/(2+d_*)}\big)$.
These bounds balance the tradeoff between complexity and accuracy in our setting, as discussed under \eqref{eq:dcf-funcs1}.
The term $n^{d/(2+d)}$ inside $\Khat(\setXhat)$ serves as a straightforward upper bound for the ``worst-case'' scenario when $d_* \approx d$.
Furthermore, AFPC stops immediately if $\epsilon_n(\setXhat) = 0$, ensuring that it always produces distinct center points, justifying the representation of $\setXhat$ as a set.

Clearly, the choice $\fnih = f_n$ always works for \eqref{eq:erm-local}.
However, based on the experimental results in \cref{sec:experiments}, we strongly recommend choosing $\fnih$ as an approximate local solution to the non-convex optimization problem $\min_{f\in\setFn(\setXhat_K)} \risk_n(f) + \regz_n(f)$, initialized at $f_n$.
Our near-minimax convergence rate guarantee in \cref{thm:near-minimax-rate} holds in both cases.

Denote the parameters of the final estimator $\fni$ by $b_{\fni,k} \defeq b_{\fnih,k} + C_n^+$ and $\vw_{\fni,k} \defeq \vw_{\fnih,k}$ for all $k \in \inds$.
Besides potentially reducing inference-time computational costs, the final step \eqref{eq:dcf-estimator} also facilitates bounding the magnitude of the unregularized bias parameters $\{b_{\fni,k} : k \in \inds\}$, which is required for applying the concentration inequality used to prove the near-minimax rate of the estimator.
We also naturally extend $\lipf_f$ and $\regz_n(f)$ to $f = \fni$ by replacing $[K]$ with $\inds$ in their definitions.
For the constant function, $f_c^{\textrm{const}}(\vx) \defeq c$ for all $\vx\in\setR^d$ and $c\in\setR$, the empirical risk $\risk_n(f_c^{\textrm{const}} + \fnih) = c^2 - 2c C_n^+ + \risk_n(\fnih)$ is minimized at $c = C_n^+$ with minimum value $\risk_n(\fni)$, hence the centering in \eqref{eq:dcf-estimator} ensures $\risk_n(\fni) \le \risk_n(\fnih)$.
Moreover, $\regz_n(\fni) \le \regz_n(\fnih)$ since the slope parameters of $\fni$ are a subset of those of $\fnih$.
Therefore, the estimator $\fni$ is also a refinement of $f_n$ that satisfies
\begin{equation} \label{eq:estimator-improvement}
  \risk_n(\fni) + \regz_n(\fni) \le \risk_n(\fnih) + \regz_n(\fnih) \le \risk_n(f_n) + \regz_n(f_n) .
\end{equation}
In \cref{sec:analysis}, we show that the initial solution $f_n$ already approximates $\risk_n(f_*) + \theta_1 \lip_*^2$ with accuracy $\Ordo(K/n)$, which is upper bounded by the near-minimax rate.
We rely on \eqref{eq:estimator-improvement} to show that $\fni$ inherits the near-minimax rate guarantee of $f_n$.
To this end, we also use the regularizer $\regz_n(\cdot)$ in \eqref{eq:erm-local}, which prevents the largest slope magnitude of $\fnih$ (and hence of $\fni$) from exceeding that of $f_n$ by more than a constant factor of $\theta_3$. That is, $\lipf_{\fni} \le \lipf_{\fnih} = \Ordo\big(\theta_3\lipf_{f_n}\big)$, as explained in \cref{thm:fniL}.

The DCF algorithm can be adapted to use alternative function classes.
As discussed in \cref{sec:variants}, it can be applied with the ``complementary'' set $\setF^-_\somenorm(\setXhat_K) \defeq \{f : -f \in\setFn(\setXhat_K)\}$ or the ``symmetric'' set $\setF^\Delta_\somenorm(\setXhat_K) \defeq \{f_1 - f_2 : f_1, f_2 \in \setFn(\setXhat_K)\}$.
The convergence rate established in \cref{thm:near-minimax-rate} extends to both cases.
In \cref{sec:max-min-affine}, we further describe how the DCF algorithm yields max-min-affine estimators and extend its theoretical guarantees to this setting.
Finally, by restricting $\setFn(\setXhat_K)$ to convex functions (including max-affine functions), the DCF algorithm specializes to convex regression, as discussed in \cref{sec:convex-regression}.
This generalization subsumes the APCNLS algorithm of \citet{Balazs2022}, matching its convergence-rate bound while eliminating the need to know the Lipschitz constant $\lip_*$ of the regression function.

Scaling all the elements of $\setX_n$ and $\setY_n$ by the same positive constant can alter the slope variables $\vw_1,\ldots,\vw_K$ returned by the DCF algorithm.
This is due to the $\max\{1,R_{\setX_n}\}$ terms in the definitions of the parameters $\theta_0$ and $\theta_1$ in \eqref{eq:dcf-params}.
The positive lower bound on these parameters is necessary to keep the regularization active, thereby preserving the guarantee of \cref{thm:near-minimax-rate} in degenerate cases where $\E\big[\norm{\vX-\E[\vX]}^2\big]$ converges to zero as $n$ grows while the noise level $\sigma > 0$ remains fixed.
The choice of $1$ as the lower bound can be relaxed, for example to $1/\ln(n)$ at the cost of introducing additional $\ln(n)$ factors in the bound of \cref{thm:near-minimax-rate}.

\section{Experiments}
\label{sec:experiments}

We compared DCF (\cref{alg:DCF}) with other theoretically justified estimators that achieve near-minimax rates, namely the $k$-nearest neighbors estimator ($k$-NN; \citealp[e.g.,][Chapter~6]{GyorfiEtAl2002}) and the Nadaraya-Watson kernel regressor (NW; \citealp{Nadaraya1964,Watson1964}).
As baselines, we also evaluated ordinary least squares regression (OLS) and state-of-the-art tree-based estimators, namely random forests (RF; \citealp{Breiman2001}) and the XGBoost gradient boosting algorithm (XGB; \citealp{XGBoost2016}).

We present experimental results on three public datasets from the Delve Project (University of Toronto).\footnote{\url{https://www.cs.toronto.edu/~delve/data/datasets.html}}
The \texttt{cpusmall} (comp-active/cpuSmall) dataset consists of 8192 samples, where the task is to predict the portion of time that CPUs run in user mode based on 12 system activity measures.
The \texttt{pumadyn} datasets also contain 8192 samples each and involve predicting the angular acceleration of one of the links of a simulated Puma 560 robot arm.
We selected the versions designed for highly nonlinear estimation with 8 input dimensions and varying noise levels: \texttt{pumadyn-8nm} (medium noise) and \texttt{pumadyn-8nh} (high noise).
Despite their small size, these problems effectively illustrate the strengths and weaknesses of DCF estimators.

For each experiment, we drew $n \in \{1024, 2048, 4096\}$ training samples from the datasets and used the remaining data for evaluation, measuring the mean squared error (MSE) of the estimators on the test set.
Each experiment was repeated 20 times, and we report the average results along with standard deviation error bars.
All algorithms except the tree-based ones are sensitive to feature scaling, so we tested both min-max scaling (MM) and Z-score normalization (STD).
The features of the \texttt{pumadyn} datasets are already reasonably scaled, so we also conducted experiments on these without applying any additional scaling (noFS).
To ensure comparability across problems, we also standardized the response variables in each dataset by centering and scaling them to have unit variance.

\cref{fig:prbKs} shows the partition size and the average cell size distribution of the Voronoi partitions computed by AFPC.
\begin{figure}[b]
  \begin{center}
    \includegraphics[height=0.169\textwidth]{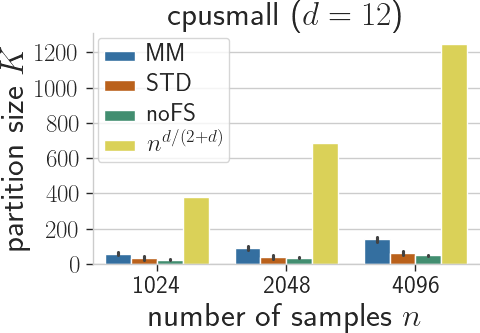}
    \includegraphics[height=0.169\textwidth]{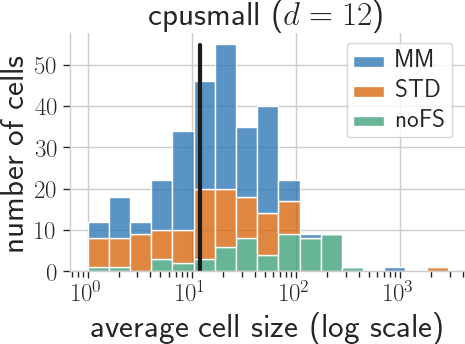}
    \includegraphics[height=0.169\textwidth]{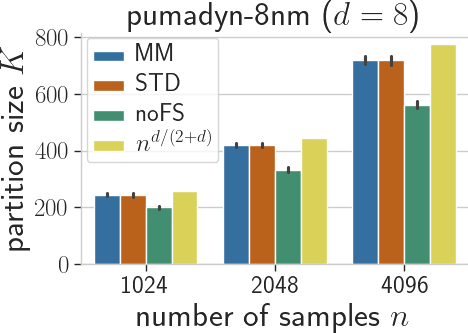}
    \includegraphics[height=0.169\textwidth]{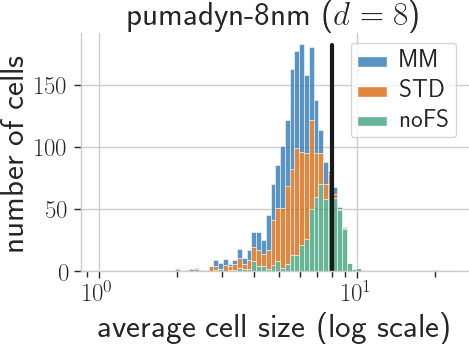}
  \end{center}
  \vspace{-5mm}
  \caption{AFPC partition size ($K$) for sample sizes $n \in \{1024, 2048, 4096\}$, and average cell size distribution for $n = 4096$.
    The upper bound of $K$ is $n^{d/(2+d)}$.
    The black vertical lines on the average cell size axes mark the value of $d$.
  The plots for \texttt{pumadyn-8nh} are similar to those of \texttt{pumadyn-8nm} and are omitted for brevity.}
  \label{fig:prbKs}
  \vspace{-5mm}
\end{figure}
On the \texttt{cpusmall} dataset, AFPC terminates relatively early under both scaling methods, with roughly half as many cells in the STD case compared with MM.
The \texttt{pumadyn} data includes noise in both the covariates and response variables, causing AFPC to return partition sizes close to the upper bound, even at the medium noise level.
In all cases, AFPC produces many cells with fewer points than the domain dimension $d$, making regression underdetermined within those cells.
Minimizing the slope of the estimator along unconstrained directions within such cells is an effective safeguard against overfitting.
In DCF, this is enforced by the regularizer $\theta_2\sum_{k\in[K]}\norm{\vw_k}^2$ in \eqref{eq:erm}, and analogously in \eqref{eq:erm-local}, as has been done in convex regression practice \citep{AybatWang2016,ChenMazumder2024}.

For the regression experiments, we used the default parameter settings for XGB and RF.
The implementations of RF and $k$-NN were taken from scikit-learn \citep{ScikitLearn}, while NW was implemented using scikit-fda \citep{ScikitFDA}.
The number of neighbors for $k$-NN was selected via $5$-fold cross-validation (CV) from the range $1$ to $\ln(n) n^{2/(2+d)}$ as motivated by \citet[Theorem~6.2]{GyorfiEtAl2002}.
For NW, we used the Gaussian and triweight kernels (referred to as NW-G and NW-T, respectively), selecting the bandwidth via $5$-fold CV among 100 equidistant values up to $(2R_{\setX_n})^{d/(2+d)}\big(R_{\setY_n}^2/n\big)^{1/(2+d)}$ as motivated by \citet[Theorem~21]{Kpotufe2010}.

DCF was implemented in Python, using local optimization for its refinement step as recommended in \cref{sec:dcf}.
We employed a quadratic penalty method \citep[e.g.,][Section~17.1]{NocedalWright2006} with a penalty parameter of $10^6$ to transform the SOCP initialization problems \eqref{eq:erm} and \eqref{eq:erm-symm}, as well as the refinement steps \eqref{eq:erm-local} and \eqref{eq:erm-local-symm}, into unconstrained minimization problems, which were then solved using L-BFGS \citep[e.g.,][Section~7.2]{NocedalWright2006}.%
\footnote{Our implementation is available at \url{https://github.com/gabalz/dcfit/releases/tag/v0.1.0}.}%
\footnote{We also implemented the SOCP problems \eqref{eq:erm} and \eqref{eq:erm-symm} using the Clarabel interior-point solver \citep{Clarabel2024}, which yielded similar results but with significantly longer computational time.}
The objective functions in the refinement steps were smoothed using the soft maximum approximation $\max_{k\in[K]}\alpha_k \approx \mu\ln\big(\sum_{k\in[K]} e^{\alpha_k/\mu}\big)$ for all $\alpha_1,\ldots,\alpha_K\in\setR$, with smoothing parameter $\mu \defeq 10^{-6}$.\footnote{More precisely, smoothing was applied only to the gradient computations, which perturbs the objective by at most $\mu\ln(K)$, a quantity we ignored.}
Since the gradients of these objective functions are either not Lipschitz or have very large Lipschitz constants (on the order of $1/\mu$), we stabilized the L-BFGS algorithm by reverting to a gradient step and resetting the L-BFGS memory whenever the backtracking line search failed.
To further improve numerical stability, the training data was centered and scaled to unit variance.

We evaluated DCF using the sets $\setFn(\cdot)$, $\setFn^-(\cdot)$, and $\setFn^\Delta(\cdot)$ for $\somenorm \in \{1,2,\infty,+\}$.
We denote the corresponding estimators by $\textrm{DCF}_\somenorm$, $\textrm{DCF}_\somenorm^-$, and $\textrm{DCF}_\somenorm^\Delta$, respectively.
As the results for $\somenorm \in \{1,2\}$ were similar to those for $\somenorm = \infty$, we only report the latter in the plots.
For reference, we denote by $\textrm{i-DCF}_+^\Delta$ the initial estimator based on $\setF_+^\Delta(\cdot)$, i.e., before applying the refinement step \eqref{eq:erm-local-symm}.
We also compare against the max-min-affine estimator based on $\setF_\infty(\cdot)$, denoted $\textrm{MMA}$, and its symmetrized variant, denoted $\textrm{MMA}^\Delta$, which are described in \cref{sec:max-min-affine}.
We use regularization parameters $\theta_0 = (R_{\setY_n}/\max\{1,R_{\setX_n}\})\ln(n)$, $\theta_1 = \max\{1,R_{\setX_n}^2\}(dK/n)$, $\theta_2 \in \big\{(R_{\setX_n}/n)^2, R_{\setX_n}^2/n\big\}$, and $\theta_3 = \ln(n)$, which satisfy \eqref{eq:dcf-params}.
Unless otherwise indicated, we use the stronger regularizer $\theta_2 = R_{\setX_n}^2/n$.

The results are summarized in \cref{fig:results}.
Across all problems, DCF (and MMA) using the symmetric sets ($\textrm{DCF}_\somenorm^\Delta$) achieved lower and more stable MSEs than their other variants with less expressive function representations.
\begin{figure}[t]
  \begin{center}
    \includegraphics[width=0.32\textwidth]{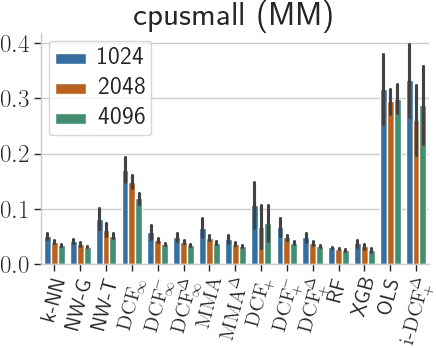}
    \includegraphics[width=0.32\textwidth]{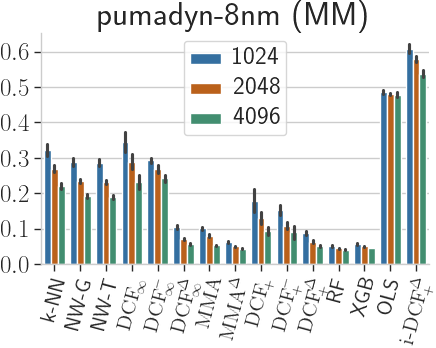}
    \includegraphics[width=0.32\textwidth]{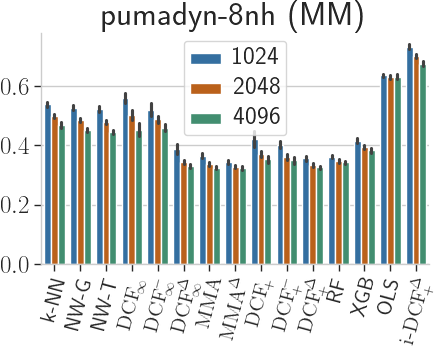}
    \includegraphics[width=0.32\textwidth]{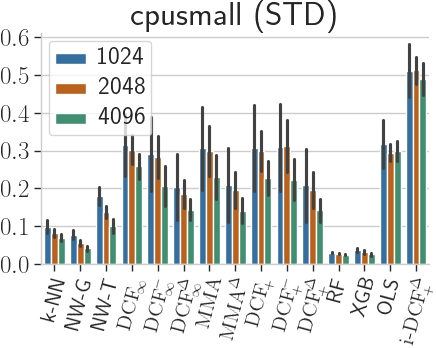}
    \includegraphics[width=0.32\textwidth]{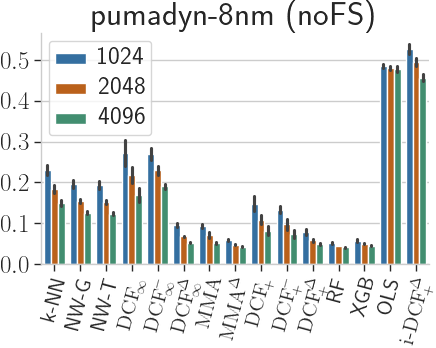}
    \includegraphics[width=0.32\textwidth]{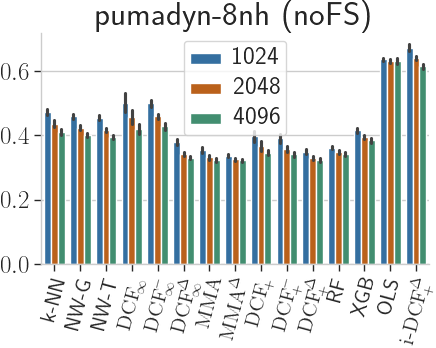}
  \end{center}
  \vspace{-5mm}
  \caption{Test MSEs of the estimators trained on sample sizes $n \in \{1024, 2048, 4096\}$. The performance is very similar across all estimators for both MM and STD scalings of the \texttt{pumadyn} datasets; therefore, the plots for the latter are omitted for brevity.}
  \label{fig:results}
  \vspace{-5mm}
\end{figure}
These symmetrized DCF variants performed at least as well as the other theoretically justified methods ($k$-NN and NW) on all problems, except for the \texttt{cpusmall} dataset with STD scaling.
In that case, \cref{fig:prbKs} shows that AFPC stopped earlier, producing only half as many cells as in the MM case, which led DCF to underfit using the stronger regularizer $\theta_2 = R_{\setX_n}^2/n$.
The tree-based methods (RF and XGB) performed quite well on these datasets, and DCF nearly matched their performance in many cases, unlike $k$-NN and NW.
Since the initial estimator $\textrm{i-DCF}_+^\Delta$ is forced to fit a continuous function only over the center points in \eqref{eq:erm} and \eqref{eq:erm-symm}, it performed worse than the OLS baseline on these sample sizes, making the refinement step strongly recommended.

The training times of the estimators are shown in the left panel of \cref{fig:dcf-stats} for the \texttt{pumadyn-8nm} dataset,\footnote{We used an Intel Core i5-2400S CPU with 4 cores at 2.50 GHz and ran two experiments in parallel.} which represents the least favorable case for DCF due to the large AFPC-computed partition size $K$ (see \cref{fig:prbKs}).
On this dataset, the training time of the most expensive DCF variants ($\textrm{DCF}_\infty^\Delta$ and $\textrm{DCF}_+^\Delta$) is close to the training time of the NW algorithms (around 5 minutes for $n = 4096$).
The refinement step in training MMA estimators takes significantly longer, as they use $d$ times more parameters than the corresponding DCF variants.
On the \texttt{cpusmall} dataset with MM scaling, where AFPC yields relatively few cells compared to the upper bound, DCF trains faster than NW (under 2 minutes for DCF, versus about 6 minutes for NW).
\begin{figure}[t]
  \begin{center}
    \includegraphics[width=0.32\textwidth]{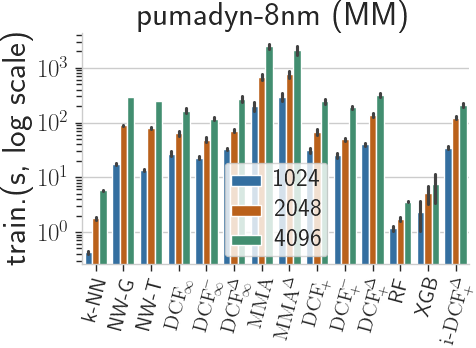}
    \includegraphics[width=0.32\textwidth]{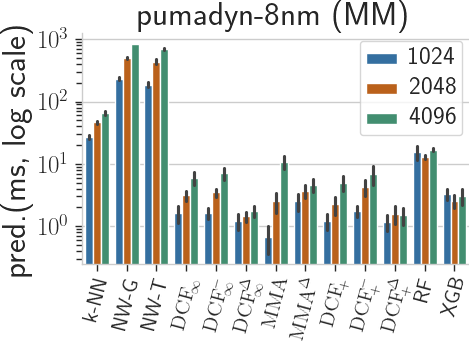}
    \includegraphics[width=0.32\textwidth]{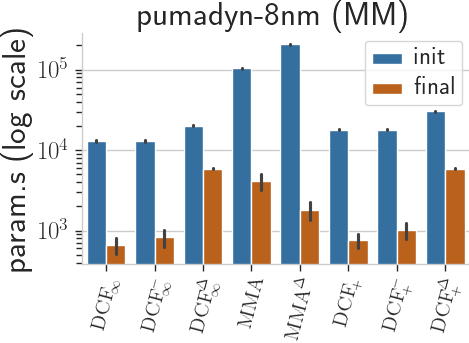}
  \end{center}
  \vspace{-5mm}
  \caption{Training and prediction times (in seconds and milliseconds, respectively) are shown for the \texttt{pumadyn-8nm} dataset with MM scaling in the left and center panels.
    Prediction times are measured on the entire test set (whose size varies with $n$) and normalized to $1000$ samples.
    The right panel shows the number of parameters used by the initial ($f_n$) and final ($\fni$) DCF estimators. For MMA, we show the parameter count of $m_n$ and of its final refinement (\cref{sec:max-min-affine}).}
  \label{fig:dcf-stats}
\end{figure}
The center panel of \cref{fig:dcf-stats} shows that the inference time of DCF on the \texttt{pumadyn-8nm} dataset is quite fast.
The right panel explains why:~a large portion of DCF's original parameters remain unused and are pruned in the final step of constructing the estimator, as described in \eqref{eq:dcf-estimator} and \eqref{eq:dcf-estimator-symm}.

Although \cref{thm:near-minimax-rate} holds for all $\theta_2 \in \big[0,\max\{1,R_{\setX_n}^2\}(d/n)\big]$, as defined in \eqref{eq:dcf-params}, \cref{fig:results-small} shows that the practical performance of DCF is sensitive to this parameter.
\begin{figure}[t!]
  \begin{center}
    \includegraphics[width=0.32\textwidth]{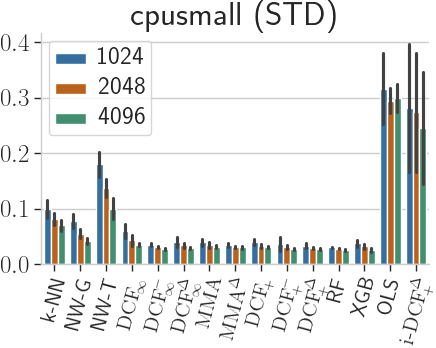}
    \includegraphics[width=0.32\textwidth]{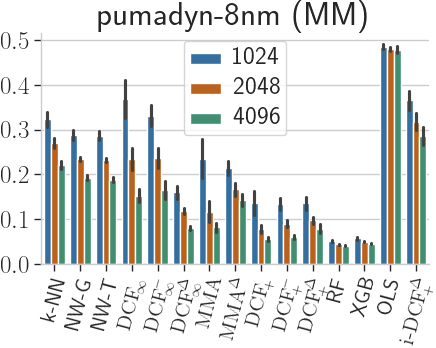}
    \includegraphics[width=0.32\textwidth]{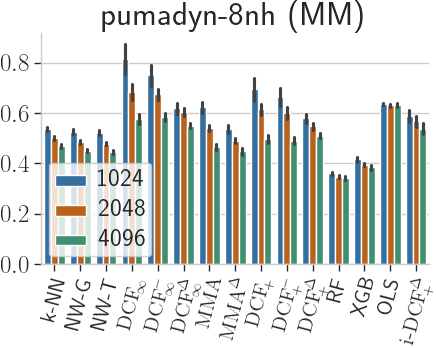}
  \end{center}
  \vspace{-5mm}
  \caption{Test MSEs, using the same notations as above, where all DCF models are trained with the weaker regularizer $\theta_2 = (R_{\setX_n}/n)^2$.}
  \label{fig:results-small}
\end{figure}
Using the weaker regularizer $\theta_2 = (R_{\setX_n}/n)^2$, the left panel shows that all DCF methods perform well on the \texttt{cpusmall} dataset with STD scaling (unlike using the stronger regularizer in \cref{fig:results}).
However, as seen in the center and right panels, this choice allows the DCF algorithms to overfit on the \texttt{pumadyn} datasets, most notably on the noisier \texttt{pumadyn-8nh} variant shown in the right panel.

\section{Analysis of DCF}
\label{sec:analysis}

This section is dedicated to the proof of \cref{thm:near-minimax-rate}.
First, \cref{sec:approxL} presents our main approximation result in \cref{thm:approxL}, which underpins the analysis of DCF (\cref{alg:DCF}).
Then, in \cref{sec:dcf-derivation}, we examine the properties of the DCF algorithm by relating it to empirical risk minimization.
Finally, after briefly reviewing the guarantees of AFPC (\cref{alg:AFPC}) in \cref{sec:afpc-guarantees}, we apply techniques from empirical process theory in \cref{sec:proof-near-minimax-rate} to establish the near-minimax rate.

\subsection{Approximation of Lipschitz Functions}
\label{sec:approxL}

For a metric space $(\setZ,\psi)$, some $\epsilon > 0$, and some $\setZ_0 \subseteq \setZ$, the finite set $\{\vz_k \in \setZ_0 : k\in[K]\}$ is called an (internal) \emph{$\epsilon$-cover} of $\setZ_0$ \wrt\ the metric $\psi$ if the union of the $\epsilon$-balls centered at $\vz_k$ covers $\setZ_0$, that is $\setZ_0 \subseteq \cup_{k\in[K]}\{\vz\in\setZ : \psi(\vz_k,\vz) \le \epsilon\}$.
The cardinality of the smallest such cover is called the \emph{$\epsilon$-covering number} of $\setZ_0$ \wrt\ $\psi$ and denoted by $N_\psi(\setZ_0,\epsilon)$.

Let $\setF_{\lip,\setX}$ denote the class of $\lip$-Lipschitz functions on a set $\setX \subseteq \setR^d$ \wrt\ $\norm{\cdot}$, defined~as
\begin{equation*}
  \setF_{\lip,\setX} \defeq \Big\{ f : \setX \to \setR \,\Big|\, \sup_{\vx,\vxhat\in\setX,\vx\ne\vxhat}\norm{\vx-\vxhat}^{-1}\big(f(\vx)-f(\vxhat)\big) \le \lip\Big\} .
\end{equation*}
\citet[Theorem~1]{McShane1934} showed that every $\lip$-Lipschitz function $f : \setX \to \setR$ on some $\setX \subset \setR^d$ can be extended to $\setR^d$ via the function $\tilde{f}(\cdot) \defeq \sup_{\vxhat\in\setX}f(\vxhat)-\lip\norm{\,\cdot\,-\vxhat}$, satisfying $\tilde{f}(\vx) = f(\vx)$ for all $\vx \in \setX$.
The key observation in this paper is the uniform $\Ordo(\lip\epsilon)$ approximation bound in \cref{thm:approxL}, which replaces the supremum in $\tilde{f}$ with a maximum over a finite $\epsilon$-cover of $\setX$.
\begin{theorem} \label{thm:approxL}
  Let $\setX_\epsilon \subseteq \setX$ be an $\epsilon$-cover of $\setX \subset \setR^d$ \wrt\ $\norm{\cdot}$, and $\norm{\cdot}_{\somenorm}$ be a norm on $\setR^d$ such that $t_0\norm{\cdot}_{\somenorm} \le \norm{\cdot} \le t_1\norm{\cdot}_{\somenorm}$ with some constants $t_0, t_1 > 0$.
  Suppose $f \in \setF_{\lip,\setX}$ for some Lipschitz constant $\lip > 0$, and define $\hat{f}(\vx) \defeq \max_{\vxhat\in\setX_\epsilon}f(\vxhat) - t_1\lip\norm{\vx-\vxhat}_\somenorm$ for all $\vx \in \setR^d$.
  Then, for all $\vx\in\setX$ and all $\vxtilde\in\setX_\epsilon$, the following hold:
  \begin{equation*}
    0 \le f(\vx)-\hat{f}(\vx) \le (1 + t_0^{-1}t_1) \lip \epsilon
    ,\qquad
    \hat{f}(\vxtilde) = f(\vxtilde)
    ,\qquad
    \hat{f} \in \setF_{(t_1/t_0)\lip,\setR^d} .
  \end{equation*}
\end{theorem}
\begin{proof}
  Take any $\vx \in \setX$ and $\vxtilde \in \setX_\epsilon$ arbitrarily.
  Since $f \in \setF_{\lip,\setX}$ and $\norm{\cdot} \le t_1 \norm{\cdot}_{\somenorm}$, we get
  $\hat{f}(\vx) - f(\vx) = \max_{\vxhat\in\setX_\epsilon}f(\vxhat) - f(\vx) - t_1\lip\norm{\vx - \vxhat}_{\somenorm} \le 0$.
  Similarly, $\hat{f}(\vxtilde) - f(\vxtilde) = \max_{\vxhat\in\setX_\epsilon}f(\vxhat) - f(\vxtilde) - t_1\lip\norm{\vxtilde - \vxhat}_{\somenorm} \ge 0$, hence $\setX_\epsilon \subseteq \setX$ implies $\hat{f}(\vxtilde) = f(\vxtilde)$.
  By the $\epsilon$-covering condition, we have $\min_{\vxhat\in\setX_\epsilon}\norm{\vx-\vxhat} \le \epsilon$.
  Combining this with $\norm{\cdot}_{\somenorm} \le t_0^{-1} \norm{\cdot}$, we obtain
  $f(\vx) - \hat{f}(\vx) = \min_{\vxhat\in\setX_\epsilon}f(\vx) - f(\vxhat) + t_1\lip\norm{\vx-\vxhat}_{\somenorm} \le (1 + t_0^{-1}t_1) \lip \epsilon$.
  Moreover, $\hat{f}$ is $((t_1/t_0)\lip)$-Lipschitz on $\setR^d$ \wrt\ $\norm{\cdot}$ because the $\max$ function is 1-Lipschitz and $\norm{\cdot}_{\somenorm}$ is $(t_0^{-1})$-Lipschitz on $\setR^d$ \wrt\ $\norm{\cdot}$.
  That is, $\hat{f} \in \setF_{(t_1/t_0)\lip,\setR^d}$.
\end{proof}

\cref{thm:approxL} provides a lower approximation of $f \in \setF_{\lip,\setX}$ by a ``max-concave'' function $\hat{f}$ that satisfies $\hat{f} \le f$ on $\setX$.
Similarly, one can construct an upper approximation of $f$ using a ``min-convex'' function \citep{HiriartUrruty1980}, defined by $\check{f}(\vx) \defeq \min_{\vxhat\in\setX_\epsilon}f(\vxhat) + t_1\lip\norm{\vx-\vxhat}_\somenorm$ for all $\vx\in\setR^d$.
We also present a variant of \cref{thm:approxL} with an improved approximation rate for smooth functions in \cref{sec:approxS}.

The approximation error of $\hat{f}$ and $\check{f}$ \wrt\ $f$ is zero at the points in $\setX_\epsilon$, and increases proportionally with the distance from those points.
\cref{fig:approxL} illustrates several examples of these approximations.
\begin{figure}[b!]
  \begin{center}
    \fbox{\includegraphics[width=0.24\textwidth]{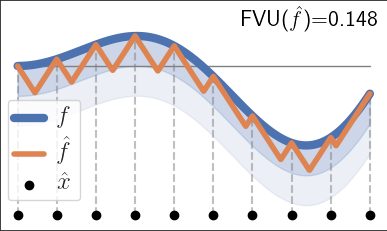}}
    \fbox{\includegraphics[width=0.24\textwidth]{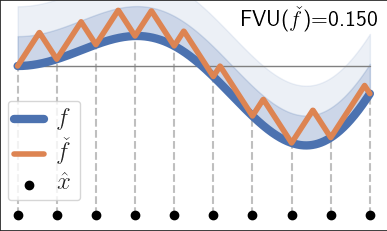}}
    \fbox{\includegraphics[width=0.24\textwidth]{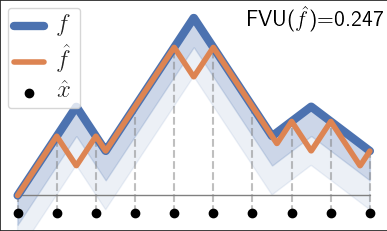}}
    \fbox{\includegraphics[width=0.24\textwidth]{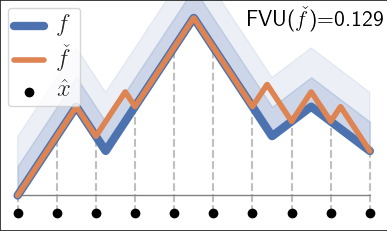}}
  \end{center}
  \vspace{-5mm}
  \caption{
    Approximation of a function $f \in \setF_{\lip,\setX}$ by the max-concave $\hat{f}$ and the min-convex $\check{f}$ as defined above.
    The left two plots use $f(x) = x\sin(x)$, while the right two plots use $f(x) = \max\{1-|x-1|,2-|x-3|,1-|x-5|/2\}$, both on $\setX = [0,6]$.
    The shaded regions represent $\lip\epsilon$ and $2\lip\epsilon$ bounds around $f$.
    Black circles mark the locations of the $10$ equidistant centers $\setX_\epsilon$, forming an $\epsilon$-cover of $\setX$ with $\epsilon = 1/3$.
    The horizontal line is shown at the height of zero.
    FVU (fraction of variance unexplained) is calculated over $n = 1000$ equidistant points \mbox{$\vx_1,\ldots,\vx_n \in \setX$} with $y_i = f(\vx_i)$ as: $\textrm{FVU}(\hat{f}) \defeq \sum_{i\in[n]}(y_i-\hat{f}(\vx_i))^2/\sum_{i\in[n]}(y_i-\bar{y})^2$, where $\bar{y} \defeq (1/n)\sum_{i\in[n]}y_i$.}
  \label{fig:approxL}%
\end{figure}
The functions $\hat{f}$ and $\check{f}$ provide an $\Ordo(\lip\epsilon)$ approximation rate of the $\lip$-Lipschitz function $f$ while maintaining an $\Ordo(\lip)$ Lipschitz constant, as guaranteed by \cref{thm:approxL}.
Their appropriate variants (using the norm $\norm{\cdot}_\somenorm$) are included in the function sets $\setFn(\setXhat)$ and $\setF^-_\somenorm(\setXhat)$ for $\somenorm \in \{1,2,\infty\}$.
Hence, these classes, as well as their superclasses such as $\setF_+(\setXhat)$ and $\setFn^\Delta(\setXhat)$, inherit the same guarantees.

\subsection{Properties of DCF Estimators}
\label{sec:dcf-derivation}

Let $\somenorm \in \{1,2,\infty,+\}$, $k_0 \in \setN$, and for any nonempty, finite set $\setXhat \defeq \{\vxhat_1,\ldots,\vxhat_{k_0}\}$, define the function class
\begin{equation*} \begin{aligned}
  \setG_{\somenorm}(\setXhat)
  \defeq \Big\{ g : \setR^d \to \setR \,\Big|\, g(\vx) &\defeq \hspace{-1mm} \sum_{k\in[k_0]}\ind\big\{\vx\in\setC_k(\setXhat)\big\} \, h_{g,k}(\vx) ,\, \vx\in\setR^d ,\\
  h_{g,k}(\vx) &\defeq b_{g,k} + \vw_{g,k}^\T\,\phin(\vx,\vxhat_k) ,\,
  b_{g,k} \in \setR ,\, \vw_{g,k} \in \setR^{d_\somenorm} ,\, k\in[k_0] \Big\} .
\end{aligned} \end{equation*}
Note that $\setG_{\somenorm}(\setXhat)$ is a vector space over $\setR$,
since for any $g_1,g_2\in\setG_\somenorm(\setXhat)$ and $c_1,c_2 \in \setR$,
the function $c_1 g_1 + c_2 g_2$ also belongs to $\setG_\somenorm(\setXhat)$, with $b_{c_1 g_1 + c_2 g_2, k} = c_1b_{g_1,k}+c_2b_{g_2,k}$ and $\vw_{c_1 g_1 + c_2 g_2,k} = c_1\vw_{g_1,k}+c_2\vw_{g_2,k}$ for all $k\in[k_0]$.

Since the parameter spaces of $\setFn(\setXhat)$ and $\setG_\somenorm(\setXhat)$ coincide, we can extend $\regz_{\theta_0,\theta_1,\theta_2}(f)$ and $\lipf_f$ from \cref{sec:dcf}, originally defined on $\setFn(\setXhat)$, to any function $f \in \setFn(\setXhat) \cup \setG_\somenorm(\setXhat)$.
We also define the map $\proj : \setG_\somenorm(\setXhat) \to \setFn(\setXhat)$ by $\proj(g)(\vx) \defeq \max_{k\in[k_0]}h_{g,k}(\vx)$ for all $\vx\in\setR^d$ and $g \in \setG_\somenorm(\setXhat)$.
Clearly, for all $g \in \setG_\somenorm(\setXhat)$, we have $b_{\proj(g),k} = b_{g,k}$ and $\vw_{\proj(g),k} = \vw_{g,k}$ for all $k\in[k_0]$, and $\proj(g) \ge g$ on $\setR^d$ holds since $\setC_1(\setXhat), \ldots, \setC_{k_0}(\setXhat)$ are pairwise disjoint.

Consider the setting of \cref{thm:near-minimax-rate}, and let $\setXhat_K = \{\vXhat_1,\ldots,\vXhat_K\} \subseteq \setX_n$ be the set of center points computed by AFPC (\cref{alg:AFPC}).
Since $\phin(\vxhat,\vxhat) = \vzero_{d_\somenorm}$ for any $\vxhat \in \setR^d$ and $\vXhat_k \in \setC_k(\setXhat_K)$ for all $k\in[K]$, we obtain $g(\vXhat_k) = h_{g,k}(\vXhat_k) = b_{g,k}$ for all $g \in \setG_\somenorm(\setXhat_K)$ and $k\in[K]$.
Thereby, the constraints $b_{g,k} \ge h_{g,l}(\vXhat_k)$, $k, l \in [K]$, are equivalent to \mbox{$g(\vXhat_k) \ge \proj(g)(\vXhat_k)$}, $k\in[K]$, which together with $\proj(g) \ge g$ implies that the functions $g$ and $\proj(g)$ coincide on the set $\setXhat_K$.
Moreover, since the index sets $\{i \in [n] : \vX_i \in \setC_k(\setXhat_K)\}$, $k\in[K]$, partition $[n]$, the data-fitting term of \eqref{eq:erm} equals $\risk_n(g)$, and $\regz_{\theta_0,\theta_1,\theta_2}(g)$ coincides with the regularization of \eqref{eq:erm} with $z = (\lipf_g - \theta_0)_+$.
Therefore, the convex optimization problem \eqref{eq:erm} in DCF (\cref{alg:DCF}) is equivalent to the following regularized empirical risk minimization task:
{
  \setlength{\abovedisplayskip}{\dimexpr\abovedisplayskip-2mm\relax}
  \setlength{\belowdisplayskip}{\dimexpr\belowdisplayskip-2mm\relax}
  \setlength{\abovedisplayshortskip}{\dimexpr\abovedisplayshortskip-2mm\relax}
  \setlength{\belowdisplayshortskip}{\dimexpr\belowdisplayshortskip-2mm\relax}
  \begin{equation} \label{eq:erm-G}
    \min_{g\in\setG_\somenorm(\setXhat_K)} \risk_n(g) + \regz_{\vtheta}(g) \textrm{ such that $g = \proj(g)$ on $\setXhat_K$} ,
  \end{equation}
}%
where $\regz_{\vtheta}(\cdot) \defeq \regz_{\theta_0,\theta_1,\theta_2}(\cdot)$.
Let $g_n$ be a solution to \eqref{eq:erm-G} that matches the solution of \eqref{eq:erm} in parameters, i.e., $b_{g_n,k} = b_{n,k}$ and $\vw_{g_n,k} = \vw_{n,k}$ for all $k \in [K]$.
Then, by the definition of the initial DCF estimator $f_n$ in \cref{sec:dcf}, we have $f_n = \proj(g_n)$.

We will also need the following properties of the feature vector $\phin$: its output norm is bounded by a constant multiple of the Euclidean distance between its arguments, and $\phin$ is Lipschitz in each of its two arguments, as shown in the next result:

\begin{lemma} \label{thm:phi-props}
  Let $\somenorm \in \{1,2,\infty,+\}$. Then for all $\vx, \vxhat, \vxtilde \in \setR^d$, the following hold:
  {
    \setlength{\abovedisplayskip}{\dimexpr\abovedisplayskip-1mm\relax}
    \setlength{\belowdisplayskip}{\dimexpr\belowdisplayskip-1mm\relax}
    \setlength{\abovedisplayshortskip}{\dimexpr\abovedisplayshortskip-1mm\relax}
    \setlength{\belowdisplayshortskip}{\dimexpr\belowdisplayshortskip-1mm\relax}
    \begin{equation*}
      \norm{\phin(\vx,\vxhat)} \le \cphi\norm{\vx-\vxhat}
      ,\quad
      \max\hspace{-0.5mm} \big\{
      \norm{\phin(\vx,\vxhat) - \phin(\vx,\vxtilde)},
      \norm{\phin(\vxhat,\vx) - \phin(\vxtilde,\vx)}
      \big\}
      \le \lipphi \norm{\vxhat - \vxtilde}
      ,
    \end{equation*}
  }%
  where $\cphi \defeq \sqrt{1 + \ind\{\somenorm \in \{2,\infty\}\} + d\,\ind\{\somenorm = 1\}}$, and $\lipphi \defeq 1 + \ind\{\somenorm \ne 1\} + \sqrt{d}\,\ind\{\somenorm = 1\}$.
\end{lemma}
\begin{proof}
  For $\somenorm \in \{1,2,\infty\}$, the first claim follows from the inequalities $\norm{\cdot}_\infty \le \norm{\cdot}$ and $\norm{\cdot}_1 \le \sqrt{d}\norm{\cdot}$, since $\norm{\phin(\vx,\vxhat)}^2 = \norm{\vx-\vxhat}^2 + \norm{\vx-\vxhat}^2_\somenorm \le \cphi^2 \norm{\vx-\vxhat}^2$.

  For $\somenorm = +$, we have $\norm{\phi_+(\vx,\vxhat)}^2 = \norm{(\vx-\vxhat)_+}^2 + \norm{(\vxhat-\vx)_+}^2 = \norm{\vx-\vxhat}^2$, implying the first claim with $\tau_{\phi_+} = 1$.

  Next we prove that $\phin$ is $\lipphi$-Lipschitz in its second argument.
  The claim that it is also $\lipphi$-Lipschitz in its first argument follows analogously.

  For $\somenorm \in \{1,2,\infty\}$, we have $\phin(\vx,\vxhat)-\phin(\vx,\vxtilde) = \big[(\vxtilde-\vxhat)^\T\,\,\,(\norm{\vx-\vxhat}_\somenorm-\norm{\vx-\vxtilde}_\somenorm)\big]^\T$.
  From this, we prove the second claim using $\norm{[\vu^\T\,s]^\T} \le \norm{\vu} + |s|$ and the reverse triangle inequality $|\norm{\vu}_\somenorm - \norm{\vv}_\somenorm| \le \norm{\vu - \vv}_\somenorm \le (\lipphi-1)\norm{\vu-\vv}$, which hold for all $\vu,\vv \in \setR^d$, $s \in \setR$.

  For $\somenorm = +$, $\phi_+(\vx,\vxhat)-\phi_+(\vx,\vxtilde) = \big[\big((\vx-\vxhat)_+-(\vx-\vxtilde)_+\big)^\T\,\,\,\big((\vxhat-\vx)_+-(\vxtilde-\vx)_+\big)^\T\big]^\T$.
  Then, using $\norm{[\vu^\T\,\vv^\T]^\T} \le \norm{\vu} + \norm{\vv}$ and $\norm{(\vu)_+ - (\vv)_+} \le \norm{\vu - \vv}$ for all $\vu, \vv\in\setR^d$, the second claim follows with $\lip_{\phi_+} = 2$.
\end{proof}

The functions in $\setFn(\setXhat_K)$ depend nonlinearly on the random center points $\setXhat_K$.
In order to apply concentration inequalities, we exploit the Lipschitz property of $\phin$ from \cref{thm:phi-props}, which enables us to ``approximate'' these random centers by fixed (non-random) ones in \cref{sec:conc-ineq}.

Recall that $\setXhat_K$ is an $\epsilon_n(\setXhat_K)$-cover of the covariate data $\setX_n$ by definition.
Then, the following result provides a uniform bound on the distance between $f_n$ and $g_n$ on $\setX_n$:
\begin{lemma} \label{thm:distance-fngn}
  For all $i \in [n]$, it holds that $0 \le f_n(\vX_i) - g_n(\vX_i) \le 2 \lipf_{f_n} \lipphi \epsilon_n(\setXhat_K)$.
\end{lemma}
\begin{proof}
  The lower bound follows directly from $f_n = \proj(g_n) \ge g_n$.
  Fix $i \in [n]$ arbitrarily, and let $k \in [K]$ be such that $\vX_i \in \setC_k(\setXhat_K)$.
  From \eqref{eq:erm-G}, we have $g_n = \proj(g_n)$ on $\setXhat_K$, which implies $0 \le h_{g_n,k}(\vXhat_k) - h_{g_n,l}(\vXhat_k)$ for all $l\in[K]$.
  By the definition of $\proj$, we also have $h_{f_n,l}(\vx) = h_{g_n,l}(\vx)$ for all $l \in [K]$ and $\vx \in \setR^d$.
  Therefore:
  {
    \setlength{\abovedisplayskip}{\dimexpr\abovedisplayskip-2mm\relax}
    \setlength{\belowdisplayskip}{\dimexpr\belowdisplayskip-3mm\relax}
    \setlength{\abovedisplayshortskip}{\dimexpr\abovedisplayshortskip-2mm\relax}
    \setlength{\belowdisplayshortskip}{\dimexpr\belowdisplayshortskip-3mm\relax}
    \begin{equation*} \begin{split}
      f_n(\vX_i) - g_n(\vX_i)
      &= \max_{l\in[K]} h_{f_n,l}(\vX_i) - h_{g_n,k}(\vX_i)
      \\
      &\le \max_{l\in[K]} h_{g_n,l}(\vX_i) - h_{g_n,l}(\vXhat_k) + h_{g_n,k}(\vXhat_k) - h_{g_n,k}(\vX_i)
      \\
      &\le \max_{l\in[K]} \big(\norm{\vw_{n,l}} + \norm{\vw_{n,k}}\big) \lipphi \norm{\vX_i - \vXhat_k} ,
    \end{split} \end{equation*}
}%
where we used the Cauchy-Schwarz inequality and \cref{thm:phi-props} in the last step.
The claim then follows from the definitions of $\lipf_{f_n}$ and $\epsilon_n(\setXhat_K)$ using $\vX_i \in \setC_k(\setXhat_K)$.
\end{proof}

\vspace{-4mm}
\cref{thm:approxL} bounds the uniform approximation error of $\setFn(\setXhat_K)$ to the $\lip_*$-Lipschitz regression function $f_*$ by $\Ordo\big(\lip_* \epsilon_n(\setXhat_K)\big)$.
The next result transfers this bound to $\setG_\somenorm(\setXhat_K)$, thereby justifying its use in the ERM task~\eqref{eq:erm-G}, and allowing us to reformulate the problem as the tractable convex optimization task in \eqref{eq:erm}.

\begin{lemma} \label{thm:gstar}
  For $f_* \in \setF_{\lip_*,\setX_*}$, there exists a function $g_* \in \setG_\somenorm(\setXhat_K)$ such that $f_* \ge \proj(g_*) \ge g_*$ on $\setX_n$, and
  {
    \setlength{\abovedisplayskip}{\dimexpr\abovedisplayskip-2mm\relax}
    \setlength{\belowdisplayskip}{\dimexpr\belowdisplayskip-2mm\relax}
    \setlength{\abovedisplayshortskip}{\dimexpr\abovedisplayshortskip-2mm\relax}
    \setlength{\belowdisplayshortskip}{\dimexpr\belowdisplayshortskip-2mm\relax}
    \begin{equation*}
      \max_{i\in[n]} f_*(\vX_i) - g_*(\vX_i) \le \capx \, \lip_* \, \epsilon_n(\setXhat_K)
      ,\quad\qquad
      \textrm{$g_* = \proj(g_*)$ on $\setXhat_K$}
      ,\quad\qquad
      \lipf_{g_*} \le \clip \lip_*
      ,
    \end{equation*}
  }%
  where $\capx \defeq 1 + \ind\{\somenorm = 2\} + \ind\{\somenorm \ne 2\}\sqrt{d}$, and $\clip \defeq \ind\{\somenorm \in \{1,2\}\} + \ind\{\somenorm = \infty\}\sqrt{d} + \ind\{\somenorm = +\}\sqrt{2d}$.
\end{lemma}
\vspace{-1mm}
\begin{proof}
  Fix the norm $\somenorm \in \{1,2,\infty\}$, and define $t_0 \defeq \ind\{\somenorm \in \{2,\infty\}\} + \ind\{\somenorm = 1\}/\sqrt{d}$ and $t_1 \defeq \ind\{\somenorm \ne \infty\} + \ind\{\somenorm = \infty\}\sqrt{d}$.
  Then, $t_0\norm{\cdot}_\somenorm \le \norm{\cdot} \le t_1\norm{\cdot}_\somenorm$.

  Let $b_{*,k} \defeq f_*(\vXhat_k)$ and $\vw_{*,k} \defeq [\vzero_d^\T\,-t_1\lip_*]^\T$ for all $k\in[K]$.
  Define $\hat{f}_* \in \setFn(\setXhat_K)$ and $g_* \in \setG_\somenorm(\setXhat_K)$ such that $b_{\hat{f}_*,k} = b_{g_*,k} = b_{*,k}$ and $\vw_{\hat{f}_*,k} = \vw_{g_*,k} = \vw_{*,k}$ for all $k\in[K]$.
  Then, we have $\hat{f}_* = \proj(g_*) \ge g_*$.

  Let $\epsilon \defeq \epsilon_n(\setXhat_K)$.
  Since $\setXhat_K$ is an $\epsilon$-cover of $\setX_n$, we can apply \cref{thm:approxL} with $f = f_*$, $\hat{f} = \hat{f}_*$, $\setX = \setX_n$, and $\setX_\epsilon = \setXhat_K$.
  This yields $\hat{f}_*(\vX_i) \le f_*(\vX_i)$ for all $i\in[n]$, proving $f_* \ge \proj(g_*) \ge g_*$ on $\setX_n$.
  Now fix $i \in [n]$, and let $k \in [K]$ be such that $\vX_i \in \setC_k(\setXhat_K)$.
  By the definitions of $g_*$ and $\setC_k(\setXhat_K)$, we have $g_*(\vX_i) = f_*(\vXhat_k) - t_1 \lip_* \norm{\vX_i - \vXhat_k}_\somenorm$.
  Then, using $f_* \in \setF_{\lip_*,\setX_*}$, $\setX_n \subseteq \setX_*$, and $\norm{\cdot}_\somenorm \le t_0^{-1}\norm{\cdot}$, we obtain
  {
    \setlength{\abovedisplayskip}{\dimexpr\abovedisplayskip-1mm\relax}
    \setlength{\belowdisplayskip}{\dimexpr\belowdisplayskip-1mm\relax}
    \setlength{\abovedisplayshortskip}{\dimexpr\abovedisplayshortskip-1mm\relax}
    \setlength{\belowdisplayshortskip}{\dimexpr\belowdisplayshortskip-1mm\relax}
    \begin{equation} \label{eq:ghat-proof}
      f_*(\vX_i) - g_*(\vX_i)
      = f_*(\vX_i) - f_*(\vXhat_k) + t_1 \lip_* \norm{\vX_i - \vXhat_k}_\somenorm
      \le (1 + t_0^{-1}t_1) \lip_* \norm{\vX_i - \vXhat_k} ,
    \end{equation}
  }%
  which implies $f_*(\vX_i) - g_*(\vX_i) \le \capx \lip_* \epsilon$ for all $i\in[n]$ as $(1 + t_0^{-1}t_1) = \capx$ and $\norm{\vX_i - \vXhat_k} \le \epsilon$.

  \noindent
  Additionally, setting $\vX_i = \vXhat_k$ in \eqref{eq:ghat-proof} and using $f_* \ge g_*$ on $\setX_n$ yields $g_*(\vXhat_k) = f_*(\vXhat_k)$ for all $k\in[K]$.
  Then, combining this with $f_*(\vXhat_k) = \hat{f}_*(\vXhat_k)$ from \cref{thm:approxL}, we obtain $g_*(\vXhat_k) = f_*(\vXhat_k) = \hat{f}_*(\vXhat_k) = \proj(g_*)(\vXhat_k)$ for all $k\in[K]$, and thus $g_* = \proj(g_*)$ on $\setXhat_K$.
  Finally, $\lipf_{g_*} = \max_{k\in[K]}\norm{\vw_{*,k}} = t_1 \lip_* = \clip \lip_*$, which proves the claim for all $\somenorm \in \{1,2,\infty\}$.

  The case $\somenorm = +$ follows from the case $\somenorm = 1$ by using $\vw_{*,k}^\T\,\phi_+(\vx,\vxhat) = -\lip_*\norm{\vx-\vxhat}_1$ for all $\vx,\vxhat\in\setR^d$ with $\vw_{*,k} \defeq -\lip_*\vone_{2d}$, which satisfies $\norm{\vw_{*,k}} = \clip \lip_*$ for all $k\in[K]$.
\end{proof}

Note that $f_n$, $\fnih$, and $\fni$ are Lipschitz continuous with Lipschitz constants bounded by $(\lipf_{f_n}\lipphi)$, $(\lipf_{\fnih}\lipphi)$, and $(\lipf_{\fni}\lipphi)$, respectively.
In the refinement step \eqref{eq:erm-local}, it is important to ensure that the Lipschitz constant of $\fnih$ does not scale polynomially in the sample size $n$, as this would deteriorate the convergence rate of the DCF estimator.
In \cref{sec:Eapprox-bound}, we prove that the Lipschitz constant of $f_n$ grows logarithmically with $n$.
Therefore, we use $\lipf_{f_n}$ as reference to regularize in \eqref{eq:erm-local}, and the following result shows that the Lipschitz constant of the refined estimators $\fnih$ and $\fni$ can exceed $\lipf_{f_n}\lipphi$ by at most an $\Ordo(\theta_3)$ factor.

\begin{lemma} \label{thm:fniL}
  Let $\theta_3 \ge 1$. Then, $\lipf_{\fni} \le \lipf_{\fnih} \le (1+\theta_3) \lipf_{f_n}$.
\end{lemma}
\begin{proof}
  By definition $\lipf_{\fni} \le \lipf_{\fnih}$, and $\theta_3 \ge 1$ implies $\regz_n(f_n) = \regz_{0,0,\theta_2}(f_n)$.
  In the degenerate cases when $\lipf_{f_n} = 0$ or $\theta_{f_n} = 0$, we have $\fnih = f_n$ by definition, and the claim is trivial.

  Let $\lipf_{f_n} > 0$ and $\theta_{f_n} > 0$.
  Suppose, to the contrary, that $\lipf_{\fnih} > (1+\theta_3) \lipf_{f_n}$.
  By definition, we have $\lipf_{f_n}^{-2}\big(\risk_n(f_n) + \regz_{0,0,\theta_2}(f_n)\big) = \theta_{f_n} > 0$.
  Then,
  \begin{equation*}
    \regz_n(\fnih)
    \ge \lipf_{f_n}^{-2}\big(\risk_n(f_n) + \regz_{0,0,\theta_2}(f_n)\big)(\lipf_{\fnih}-\theta_3\lipf_{f_n})_+^2
    > \risk_n(f_n) + \regz_n(f_n) ,
  \end{equation*}
  which contradicts \eqref{eq:erm-local}, proving the claim.
\end{proof}

Since \cref{thm:distance-fngn,thm:gstar} both depend on $\epsilon_n(\setXhat_K)$, we first review the guarantees provided by AFPC for this quantity in the next section, before proceeding to the proof of \cref{thm:near-minimax-rate}.

\subsection{AFPC Guarantees}
\label{sec:afpc-guarantees}

Recall the definition of the covering number from \cref{sec:approxL}.
We will often rely on the following well-known result on the covering number of bounded sets:
\begin{lemma}[e.g., \citealt{Wainwright2019}, Lemma~5.7] \label{thm:volume-argument}
  Let $\setZ_0 \subseteq \setB(\vz_0,r) \subseteq \setR^d$ for some \mbox{$d \in \setN$}, $\vz_0 \in \setR^d$, and $r > 0$. Then $N_{\norm{\cdot}}(\setZ_0,\epsilon) \le \max\{1, (3r/\epsilon)^d\}$ for all $\epsilon > 0$.
\end{lemma}

Recall the definition of the covariate data radius $R_{\setX_n}$ from AFPC (\cref{alg:AFPC}), and let $d_\circ$ denote the doubling dimension of the domain $\setX_*$, as introduced in \cref{sec:near-minimax-rate}.
Since $\setX_n \subseteq \setX_*$, we have $N_{\norm{\cdot}}(\setX_n,\epsilon) \le \max\{1,(4R_{\setX_n}/\epsilon)^{d_\circ}\}$ for all $\epsilon > 0$ \citep[e.g.,][Lemmas~6 and 7]{KpotufeDasgupta2012},\footnote{The definition of $d_\circ$ in \cref{sec:near-minimax-rate} does not require the covering balls to have centers within the covered set. Since we use internal covering numbers, our constant 4 in front of $R_{\setX_n}$ is looser by a factor of 2.} which can be combined with \cref{thm:volume-argument} to obtain $N_{\norm{\cdot}}(\setX_n,\epsilon) \le \max\{1,(4R_{\setX_n}/\epsilon)^{d_*}\}$ for all $\epsilon > 0$.
This allows us to apply the next result, which bounds the covering accuracy and the number of the center points returned by AFPC.
\begin{lemma}[\citealt{Balazs2022}, Lemma 4.2]\hspace{-1mm}\footnote{The max cell-diameter objective used by Balázs is within a factor of 2 of the covering-radius $\epsilon_n(\setXhat)$.} \label{thm:AFPC}
  Suppose $\setXhat$ is computed by AFPC (\cref{alg:AFPC}) using the covariate data $\setX_n$, and $N_{\norm{\cdot}}(\setX_n,\epsilon) \le \max\{1,(4R_{\setX_n}/\epsilon)^{d_*}\}$ \as~holds for all \mbox{$\epsilon > 0$}.
  Then there exists a (non-random) positive integer $k_*$ such that $K \defeq |\setXhat| \le k_* = \Ordo\big(n^{d_*/(2+d_*)}\big)$ \as, and $\setXhat$ is an $\epsilon$-cover of $\setX_n$ with $\epsilon \defeq \epsilon_n(\setXhat) = \Ordo(R_{\setX_n} \sqrt{K / n})$.
\end{lemma}
Combining the bounds on $k_*$ and $\epsilon$ from \cref{thm:AFPC} yields $\epsilon^2 = \Ordo(n^{-2/(2+d_*)})$, which matches the minimax rate in the setting of \cref{thm:near-minimax-rate}.

\subsection{Proof of the Near-minimax Rate of DCF}
\label{sec:proof-near-minimax-rate}

Consider the setting of \cref{thm:near-minimax-rate}, i.e., a regression problem as in \eqref{eq:data-model} with an \iid~sample $\data$ drawn from a distribution $P_*$ corresponding to a regression function $f_* \in \setF_{\lip_*,\setX_*}$.

To simplify notation, define $\norm{f}_*^2 \defeq \E_{(\vX,\cdot) \sim P_*}[f^2(\vX)]$, $\langle f, \hat{f} \rangle_n \defeq \frac1n\sum_{i\in[n]}f(\vX_i) \hat{f}(\vX_i)$, and $\norm{f}_n^2 \defeq \langle f, f \rangle_n$ for all functions $f, \hat{f} : \setR^d \to \setR$.
Further, in these contexts, we slightly abuse notation by treating $y$ as a function, defining $y(\vX) \defeq Y$ for $(\vX,Y) \sim P_*$, and $y(\vX_i) \defeq Y_i$ for all $i\in[n]$.
For example, we write $\E_{(\vX,Y) \sim P_*}\big[(f(\vX)-Y)^2\big] = \norm{f - y}^2_*$, $\risk_n(f) = \norm{f - y}^2_n$, and $\E_{(\vX,\cdot) \sim P_*}\big[(f(\vX)-f_*(\vX))^2\big] = \norm{f - f_*}_*^2$.

Let $\fni$ be the estimator computed by DCF (\cref{alg:DCF}), and $\offsetc > 1$ be a constant.
We decompose its expected squared error by following the approach of \citet[Section~11.3]{GyorfiEtAl2002}:
\begin{equation} \label{eq:error-decomp}
  \E_{(\vX,\cdot) \sim P_*}\big[\big(\fni(\vX)-f_*(\vX)\big)^2\big]
  = \big( \norm{\fni - f_*}_*^2 - \offsetc\Eapprox \big) + \offsetc\Eapprox ,
\end{equation}
where the approximation error is defined as $\Eapprox \defeq \risk_n(\fni) - \risk_n(f_*)$.
The technique of offsetting with a factor greater than $1$ allows the derivation of faster convergence rates in the nonparametric setting, provided $\Eapprox$ remains within the minimax rate.
Using \eqref{eq:estimator-improvement}, we can upper bound $\Eapprox$ as
\begin{equation} \label{eq:Eapprox-1} \begin{split}
  \Eapprox
  &\le \risk_n(f_n) - \risk_n(f_*) + \regz_n(f_n) \\
  &= \norm{f_n - y}_n^2 - \norm{f_* - y}_n^2 + \regz_n(f_n) \\
  &= \norm{f_n - g_n}_n^2 + \norm{g_n - y}_n^2 - \norm{f_* - y}_n^2 + \regz_n(f_n) \\
  &\qquad
    + 2\langle f_n - g_n , f_* - y \rangle_n
    + 2\langle f_n - g_n , g_n - f_* \rangle_n .
\end{split} \end{equation}
Recall that $g_*$ of \cref{thm:gstar} is feasible for the minimization in \eqref{eq:erm-G} since $g_* = \proj(g_*)$ holds on $\setXhat_K$.
Because $g_n$ is a solution to \eqref{eq:erm-G}, it follows that
\begin{equation} \label{eq:erm-ineq}
  \risk_n(g_n) + \regz_{\vtheta}(g_n) \le \risk_n(g_*) + \regz_{\vtheta}(g_*) .
\end{equation}
Furthermore, using $2ab \le a^2 + b^2$ for any $a,b\in\setR$, we obtain
\begin{equation} \label{eq:Eapprox-2} \begin{split}
  2\langle f_n - g_n , g_n - f_* \rangle_n
  &= 2\langle f_n - g_n , g_* - f_* \rangle_n + 2\langle f_n - g_n , g_n - g_* \rangle_n \\
  &\le 2\norm{f_n - g_n}_n^2 + \norm{g_* - f_*}_n^2 + \norm{g_n - g_*}_n^2 .
\end{split} \end{equation}
Notice that $\theta_3 \ge 1$ implies $\regz_n(f_n) = \regz_{0,0,\theta_2}(f_n)$, and since $f_n = \proj(g_n)$, we also have $\regz_{\vtheta}(f_n) = \regz_{\vtheta}(g_n)$.
Then, plugging \eqref{eq:erm-ineq} and \eqref{eq:Eapprox-2} into \eqref{eq:Eapprox-1}, and using $\regz_n(f_n) = \regz_{0,0,\theta_2}(f_n) \le \regz_{\vtheta}(f_n) = \regz_{\vtheta}(g_n)$, we get
\begin{equation} \label{eq:Eapprox-3} \begin{split}
  \Eapprox &\le \norm{f_n - g_n}_n^2 + \norm{g_* - y}_n^2 - \norm{f_* - y}_n^2 + \regz_{\vtheta}(g_*) \\
           &\qquad
             + 2\langle f_n - g_n , f_* - y \rangle_n
             + 2\langle f_n - g_n , g_n - f_* \rangle_n
  \\
           &= \norm{f_n - g_n}_n^2 + \norm{g_* - f_*}_n^2 + 2\langle g_* - f_* , f_* - y \rangle_n + \regz_{\vtheta}(g_*)
  \\
           &\qquad
             + 2\langle f_n - g_n , f_* - y \rangle_n
             + 2\langle f_n - g_n , g_n - f_* \rangle_n
  \\
           &\le 3\norm{f_n - g_n}_n^2 + 2\norm{g_* - f_*}_n^2 + \norm{g_n - g_*}_n^2 + \regz_{\vtheta}(g_*) \\
           &\qquad
             + 2\langle f_n - g_n , f_* - y \rangle_n
             + 2\langle g_* - f_* , f_* - y \rangle_n .
\end{split} \end{equation}
Similarly to \citet[Section~4.1]{Balazs2022}, using $2ab = a(2b) \le (a^2/2) + 2b^2$ for any $a,b\in\setR$, we obtain
\begin{equation*} \begin{split}
  \risk_n(g_*) - \risk_n(g_n)
  &= \norm{g_* - y}_n^2 - \norm{g_n - y}_n^2
  \\
  &= -\norm{g_n - g_*}_n^2 + 2\langle g_* - g_n, g_* - y \rangle_n
  \\
  &= -\norm{g_n - g_*}_n^2 + 2\langle g_* - g_n, g_* - f_* \rangle_n + 2\langle g_* - g_n, f_* - y \rangle_n
  \\
  &\le -(1/2) \norm{g_n - g_*}_n^2 + 2\norm{g_* - f_*}_n^2 + 2\langle g_* - g_n, f_* - y \rangle_n ,
\end{split} \end{equation*}
which can be used to rearrange \eqref{eq:erm-ineq} as
\begin{equation} \label{eq:erm-ineq-rearranged}
  (1/2)\norm{g_n - g_*}_n^2 + \regz_{\vtheta}(g_n)
  \le 2\norm{g_* - f_*}_n^2 + 2 \langle g_* - g_n , f_* - y \rangle_n + \regz_{\vtheta}(g_*) .
\end{equation}
Inequality \eqref{eq:erm-ineq-rearranged} is an adaptation of the ``basic inequality'' of \citet[Lemma~10.1]{VanDeGeer2000}, and it plays a key role in our analysis to bound the approximation error $\Eapprox$ via \eqref{eq:Eapprox-3}.

Recall from \cref{thm:distance-fngn} that $\norm{f_n - g_n}_n = \Ordo\big(\lipf_{f_n}\lipphi \epsilon_n(\setXhat_K)\big)$.
Additionally, we also have from \cref{thm:gstar} that $\norm{g_* - f_*}_n = \Ordo\big(\capx \lip_* \epsilon_n(\setXhat_K)\big)$ and $\regz_{\vtheta}(g_*) = \Ordo(\theta_1 \lipf_{g_*}^2) = \Ordo(\theta_1 \clip^2 \lip_*^2)$ by $\theta_1 \ge K\theta_2$ from \eqref{eq:dcf-params}.
Combining these bounds with \eqref{eq:Eapprox-3} and \eqref{eq:erm-ineq-rearranged}, using $\regz_{\vtheta}(g_n) \ge 0$, we finally get
\begin{equation} \label{eq:Eapprox-4} \begin{split}
  \Eapprox
  &= \Ordo\Big((\lipf_{f_n}^2\lipphi^2 + \capx^2 \lip_*^2) \epsilon_n^2(\setXhat_K) + \theta_1 \clip^2 \lip_*^2\Big) \\
  &\qquad
    + 2\langle f_n - g_n , f_* - y \rangle_n
    + 2\langle g_* - f_* , f_* - y \rangle_n
    + 4\langle g_* - g_n , f_* - y \rangle_n .
\end{split} \end{equation}
According to Theorem~1 of \citet{Stone1982} and \cref{thm:AFPC}, the minimax rate is captured by $K/n$.
This rate is also reflected in the asymptotic expression of \eqref{eq:Eapprox-4}, since both $\epsilon_n^2(\setXhat_K)$ and $\theta_1$ scale with $K/n$, as established by \cref{thm:AFPC} and defined in \eqref{eq:dcf-params}, respectively.

It remains to show that the inner product terms of \eqref{eq:Eapprox-4} and $\norm{\fni - f_*}_*^2 - \offsetc\Eapprox$ in \eqref{eq:error-decomp} preserve the $\Ordo(K/n)$ rate.
To this end, we rely on concentration inequalities from empirical process theory.

\subsubsection{Technical Preparations}
\label{sec:tech-preps}

Since $\E[\vX]$ might not lie within $\setX_*$ for $(\vX,\cdot) \sim P_*$, we need a reference point to leverage the Lipschitz continuity of $f_*$.
To that end, fix $\vx_0 \in \argmin_{\vxhat\in\setX_*}\norm{\vxhat - \E[\vX]}$ independently of the data $\data$,\footnote{If the minimum is not attained, one may choose $\vx_0$ arbitrarily close to the infimum and shrink the gap to zero at the end of the analysis.} and set $y_0 \defeq f_*(\vx_0)$.
For any fixed $\gamma \in (0,1)$, we condition the entire proof on the event $\event_\gamma$ defined in \cref{thm:event}.

\begin{lemma} \label{thm:event}
  Let $\data$ be an \iid\ subgaussian sample as in \eqref{eq:subgaussian} for the regression function $f_* \in \setF_{\lip_*,\setX_*}$.
  Fix $\rrho \defeq \rho\sqrt{\ln(2n/\gamma)}$ and $\rsigma \defeq \sigma\sqrt{\ln(2n/\gamma)}$, and define the event
  {
    \setlength{\abovedisplayskip}{\dimexpr\abovedisplayskip-2mm\relax}
    \setlength{\belowdisplayskip}{\dimexpr\belowdisplayskip-2mm\relax}
    \setlength{\abovedisplayshortskip}{\dimexpr\abovedisplayshortskip-2mm\relax}
    \setlength{\belowdisplayshortskip}{\dimexpr\belowdisplayshortskip-2mm\relax}
    \begin{equation*}
      \event_\gamma \defeq \Big\{
      \max_{i\in[n]}\norm{\vX_i - \E[\vX]} \le \rrho
      ,\, \max_{i\in[n]}|f_*(\vX_i) - Y_i| \le \rsigma \Big\} .
    \end{equation*}
  }%
  Then $\Prob\{\event_\gamma\} \ge 1 - 2\gamma$.
  Furthermore, $\event_\gamma$ implies $R_{\setX_n} \le 2\rrho$, $R_{\setY_n} \le 2(\rsigma + \lip_*\rrho)$, and
  {
    \setlength{\abovedisplayskip}{\dimexpr\abovedisplayskip-3mm\relax}
    \setlength{\abovedisplayshortskip}{\dimexpr\abovedisplayshortskip-3mm\relax}
    \begin{equation*}
      \max_{i\in[n]}\norm{\vX_i-\vx_0} \le 2\rrho
      ,\qquad
      \max_{i\in[n]}|Y_i - y_0| \le \rsigma + 2\lip_*\rrho
      ,\qquad
      \frac{R_{\setY_n}}{\max\{1,R_{\setX_n}\}} \le 2(\rsigma + \lip_*)
      .
    \end{equation*}
  }%
\end{lemma}
\begin{proof}
  The result $\Prob\{\event_\gamma\} \ge 1 - 2\gamma$ follows from the subgaussian property \eqref{eq:subgaussian} of $P_*$, using the union and Chernoff bounds.
  The implications of $\event_\gamma$ follow from $f_* \in \setF_{\lip_*,\setX_*}$, $\setX_n \cup \{\vx_0\} \subseteq \setX_*$, and $y_0 = f_*(\vx_0)$, using the triangle and Jensen's inequalities.
\end{proof}

\cref{thm:event} shows that, with high probability, the data $\data$ lies inside a ball of bounded radius centered at $(\vx_0,y_0)$.
This is needed for deriving upper bounds on the magnitudes of the estimator parameters.

In the following sections, we work with parametric function sets and construct covers via their bounded parameter sets, as summarized in the next lemma:
\begin{lemma} \label{thm:func-cover}
  Let $(\setF,\psi)$ be a metric space, where $\setF \defeq \{f_{p_1,\ldots,p_m} : p_j\in\setP_j, j\in[m]\}$ with $\setP_j \subseteq \setB(\vz_j,r_j)$ for some $t_j\in\setN$, $\vz_j \in \setR^{t_j}$, and $r_j > 0$ for all $j\in[m]$.
  Suppose there exist constants $s_1,\ldots,s_m > 0$ such that $\psi(f_{p_1,\ldots,p_m},f_{\hat{p}_1,\ldots,\hat{p}_m}) \le \max_{j\in[m]}s_j\norm{p_j-\hat{p}_j}$ for all $p_j,\hat{p}_j\in\setP_j$, $j\in[m]$.
  Then, for all $\epsilon > 0$, we have $N_\psi(\setF,\epsilon) \le \max\big\{1,(3r/\epsilon)^{t}\big\}$, where $r \ge \max_{j\in[m]}s_jr_j$ and $t \defeq \sum_{j\in[m]}t_j$.
\end{lemma}
\begin{proof}
  By using the conditions on $\setF$ and $\psi$, the result follows directly from \cref{thm:volume-argument} as $N_\psi(\setF,\epsilon) \le \prod_{j=1}^m N_{\norm{\cdot}}(\setP_j,\epsilon / s_j) \le \prod_{j=1}^m\max\big\{1,(3s_jr_j / \epsilon)^{t_j}\big\} \le \max\{1, (3r/\epsilon)^t\}$.
\end{proof}

We will make extensive use of the following concentration inequality, which generalizes Lemma~9 of \citet{Balazs2022}.
\begin{lemma} \label{thm:conc-ineq-inner}
  Let $\data$ be an \iid\ subgaussian sample as in \eqref{eq:subgaussian} for the regression function $f_* : \setR^d \to \setR$.
  Let $\setH_n \subseteq \{f \,|\, f : \setR^d \to \setR\}$, and $\psi_n$ be a metric on $\setH_n$ such that $\norm{h - \hat{h}}_n \le \psi_n(h,\hat{h})$ holds for all $h, \hat{h} \in \setH_n$.
  Assume that $(\setH_n,\psi_n)$ and $\setY_n$ are conditionally independent given $\setX_n$.
  Let $h_n \in \setH_n$ be a function, which may depend on the entire sample $\data$.
  Then, for any $\gamma, \delta > 0$, with probability at least $1-\gamma$ over the randomness of $\setY_n | \setX_n$, it holds that
  \begin{equation*}
    \langle h_n , f_* - y \rangle_n
    \le 3\sigma \, \big(\norm{h_n}_n + \delta\big) \sqrt{\ln\big(N_{\psi_n}(\setH_n,\delta) / \gamma\big) / n} + \delta \norm{f_* - y}_n .
  \end{equation*}
\end{lemma}
\begin{proof}
  The claim is trivial when $N_{\psi_n}(\setH_n,\delta)$ is infinite, so suppose $N_{\psi_n}(\setH_n,\delta) < \infty$.
  The claimed inequality always holds if $\norm{h_n}_n = 0$.
  Therefore, without loss of generality, assume that $\norm{h}_n > 0$ for all $h \in \setH_n$.

  Let $\setH_\delta$ be a $\delta$-cover of $\setH_n$ \wrt\ $\psi_n$ of minimal cardinality, so that $|\setH_\delta| = N_{\psi_n}(\setH_n,\delta)$.
  Note that $\setH_\delta$ and $\setY_n$ are conditionally independent given $\setX_n$.
  Further, leveraging the $\delta$-covering property, choose $\hat{h}_n \in \setH_\delta$ to be such that $\norm{h_n - \hat{h}_n}_n \le \psi_n(h_n, \hat{h}_n) \le \delta$.
  Then, by using the Cauchy-Schwarz inequality, we get
  \begin{equation} \label{eq:proof-conc-ineq-1}
    \langle h_n, f_* - y \rangle_n
    = \langle \hat{h}_n, f_* - y \rangle_n + \langle h_n - \hat{h}_n, f_* - y \rangle_n
    \le \langle \hat{h}_n, f_* - y \rangle_n + \delta \norm{f_* - y}_n .
  \end{equation}
  Define $t_i(h) \defeq h(\vX_i) / \norm{h}_n$ for all $i \in [n]$ and $h \in \setH_n$, and let $c_1, c_2 > 0$ be constants to be chosen later.
  Note that $t_i(\cdot)$ only depends on $\setX_n$ for all $i\in[n]$.
  Then, using the union and Chernoff bounds, the independence of the samples $\data$, the subgaussian property \eqref{eq:subgaussian} expressed as $\sup_{s\in\setR}\E\big[e^{s(f_*(\vX_i)-Y_i)-2s^2\sigma^2}\big|\vX_i\big] \le 1$ \as\ \citep[Section~2.3]{BoucheronEtAl2013}, and $\sum_{i\in[n]}t_i^2(h) = n$ for any $h \in \setH_n$, we obtain
  \begin{equation} \label{eq:proof-conc-ineq-2} \begin{split}
    \Prob\Big\{ \big\langle \hat{h}_n , f_* - y \big\rangle_n > c_1 c_2 \norm{\hat{h}_n}_n \,\Big|\, \setX_n \Big\}
    &\le \Prob\bigg\{ \max_{\hat{h}\in\setH_\delta} \Big\langle \frac{\hat{h}}{\norm{\hat{h}}_n} , f_* - y\Big\rangle_n > c_1 c_2 \,\bigg|\, \setX_n\bigg\}
    \\
    &\le \sum_{\hat{h}\in\setH_\delta} \Prob\bigg\{ \frac1n \sum_{i=1}^n t_i(\hat{h})\big(f_*(\vX_i)-Y_i\big) > c_1 c_2 \,\bigg|\, \setX_n \bigg\}
    \\
    &\le \sum_{\hat{h}\in\setH_\delta} e^{-c_1} \prod_{i=1}^n \E\bigg[ \exp\Big(\frac{t_i(\hat{h})}{n\,c_2}\big(f_*(\vX_i)-Y_i\big)\Big) \bigg| \setX_n \bigg]
    \\
    &\le \sum_{\hat{h}\in\setH_\delta} \exp\bigg(\frac{2\sigma^2 \sum_{i\in[n]}t_i^2(\hat{h})}{(n\,c_2)^2} - c_1\bigg)
    \\
    &= |\setH_\delta| \exp\Big( \frac{2\sigma^2}{n\,c_2^2} - c_1 \Big)
    \\
    &= \gamma ,
  \end{split} \end{equation}
where we set $c_1 \defeq 3\sigma^2 / (n\,c_2^2)$ and $c_2 \defeq \sigma / \sqrt{n \ln(|\setH_\delta|/\gamma)}$.
Note that we also obtain $c_1 c_2 = 3\sigma^2 / (n\,c_2) = 3\sigma \sqrt{\ln(|\setH_\delta|/\gamma) / n}$.
At last, the triangle inequality implies $\norm{\hat{h}_n}_n \le \norm{h_n}_n + \norm{\hat{h}_n - h_n}_n \le \norm{h_n}_n + \delta$, and we get the result by combining \eqref{eq:proof-conc-ineq-1} and \eqref{eq:proof-conc-ineq-2}.
\end{proof}

Finally, consider the following bound on the pointwise distance between functions in $\setFn(\setXhat)$ or $\setG_\somenorm(\setXhat)$.

\begin{lemma} \label{thm:pointwise-distance}
  Let $k_0 \in \setN$ and $\setXhat \defeq \{\vxhat_1,\ldots,\vxhat_{k_0}\} \subset \setR^d$.
  Furthermore, let $h, \hat{h} \in \setH$ for $\setH \in \big\{\setFn(\setXhat), \setG_\somenorm(\setXhat)\big\}$.
  Then it holds for all $\vx\in\setR^d$ that
  \begin{equation*}
    \big|h(\vx) - \hat{h}(\vx)\big|
    \le
    \max_{k\in[k_0]} |b_{h,k} - b_{\hat{h},k}| + \cphi \norm{\vx - \vxhat_k} \norm{\vw_{h,k} - \vw_{\hat{h},k}} .
  \end{equation*}
\end{lemma}
\begin{proof}
  Let $f, \hat{f} \in \setFn(\setXhat)$, and $g, \hat{g} \in \setG_\somenorm(\setXhat)$.
  Then, we have
  \begin{equation*} \begin{split}
    \big|f(\vx) - \hat{f}(\vx)\big|
    &\le \max_{k\in[k_0]}\Big|b_{f,k} - b_{\hat{f},k} + \phin(\vx,\vxhat_k)^\T\big(\vw_{f,k} - \vw_{\hat{f},k}\big)\Big| ,
    \\
    \big|g(\vx) - \hat{g}(\vx)\big|
    &= \bigg|\sum_{k\in[k_0]} \ind\big\{\vx\in\setC_k(\setXhat)\big\} \Big(b_{g,k} - b_{\hat{g},k} + \phin(\vx,\vxhat_k)^\T\big(\vw_{g,k} - \vw_{\hat{g},k}\big)\Big)\bigg| .
    \\
  \end{split} \end{equation*}
From \cref{thm:phi-props}, we also get $\norm{\phin(\vx,\vxhat_k)} \le \cphi \norm{\vx - \vxhat_k}$ for all $k\in[k_0]$.
Then, the claim for $\setFn(\setXhat)$ follows from the triangle and the Cauchy-Schwarz inequalities.
The claim for $\setG_\somenorm(\setXhat)$ follows similarly, where we also upper bound the sum by a $\max$ using that $\setC_1(\setXhat),\ldots,\setC_{k_0}(\setXhat)$ are disjoint sets.
\end{proof}

\subsubsection{Bounding the Approximation Error}
\label{sec:Eapprox-bound}

The goal of this section is to bound the approximation error $\Eapprox$.
To achieve this, we use \eqref{eq:Eapprox-4} and upper bound its inner product terms.

Let $g_n$ be as in \cref{sec:dcf-derivation}, and recall that the initial DCF estimator satisfies $f_n = \proj(g_n)$.
Hence, $f_n$ and $g_n$ share the same parameters, and we have $b_{n,k} = b_{f_n,k} = b_{g_n,k}$ and $\vw_{n,k} = \vw_{f_n,k} = \vw_{g_n,k}$ for all $k \in [K]$.
First, consider the following bound on the regularized empirical risk of $g_n$:
\begin{lemma} \label{thm:risk-gn-ub}
  Suppose that event $\event_\gamma$ holds.
  Then, $\risk_n(g_n) + \regz_{\vtheta}(g_n) \le (\rsigma + 2 \lip_* \rrho)^2$.
\end{lemma}
\begin{proof}
  Define the constant function $g_0$ by $g_0(\vx) \defeq y_0$ for all $\vx\in\setR^d$.
  Observe that $g_0 \in \setG_\somenorm(\setXhat_K)$, $g_0 = \proj(g_0)$, and $\regz_{\vtheta}(g_0) = 0$.
  Since $g_n$ is a solution to \eqref{eq:erm-G}, we have $\risk_n(g_n) + \regz_{\vtheta}(g_n) \le \risk_n(g_0) = \frac1n\hspace{-0.5mm}\sum_{i\in[n]}(Y_i \hspace{-0.5mm}-\hspace{-0.5mm} y_0)^2$.
  Then, $\event_\gamma$ implies the claim by \cref{thm:event}.
\end{proof}

To bound the inner product terms in \eqref{eq:Eapprox-4}, we apply the concentration inequality of \cref{thm:conc-ineq-inner}.
For this, we need the following bound on the parameter space of DCF estimators:

\begin{lemma} \label{thm:loose-param-bounds}
  Suppose that event $\event_\gamma$ and \eqref{eq:dcf-params} hold.
  Then, $g_n \in \setGbar_\somenorm(\setXhat_K)$, where
  \begin{equation*}
    \setGbar_\somenorm(\setXhat_K)
    \defeq \Big\{
    g \in \setG_\somenorm(\setXhat_K)
    : \max_{k\in[K]} \max\big\{|b_{g,k}-y_0| , \sqrt{d} R_{\setX_n} \norm{\vw_{g,k}}\big\} \le \bnd_1
    \Big\} ,
  \end{equation*}
  and $\bnd_1$ is a constant satisfying $\rsigma\sqrt{d} \le \bnd_1 = \Theta\big(\sqrt{dn}(\rsigma + \lip_* \rrho)\big)$.
\end{lemma}
\begin{proof}
  Since $\setXhat_K \subseteq \setX_n$, we can define $i_k \in [n]$ for each $k \in [K]$ such that $(\vXhat_k, Y_{i_k}) \in \data$.
  By the definition of $g_n$, we have $b_{n,k} = g_n(\vXhat_k)$.
  Hence, by \cref{thm:risk-gn-ub}, and $\theta_1 \ge R_{\setX_n}^2 (d/n)$ from \eqref{eq:dcf-params}, it follows that
  \begin{equation*}
    \frac1n \max_{k\in[K]}\Big\{(b_{n,k}-Y_{i_k})^2, d R_{\setX_n}^2 (\norm{\vw_{n,k}} - \theta_0)_+^2\Big\}
    \le \risk_n(g_n) + \regz_{\vtheta}(g_n)
    \le (\rsigma + 2 \lip_* \rrho)^2 .
  \end{equation*}
  Then, by the triangle inequality and \cref{thm:event}, $\event_\gamma$ implies $|b_{n,k}-y_0| \le (1+\sqrt{n})(\rsigma + 2 \lip_* \rrho)$.
  Further, we also get $\sqrt{d} R_{\setX_n} \norm{\vw_{n,k}} \le \sqrt{d} R_{\setX_n} \theta_0 + \sqrt{n}(\rsigma + 2 \lip_* \rrho)$.
  Therefore, $\bnd_1$ can be chosen as claimed as $R_{\setX_n} \theta_0 = \Ordo\big(R_{\setY_n}\ln(n)\big)$ by \eqref{eq:dcf-params}, $\ln(n) \le \sqrt{n}$, and $R_{\setY_n} = \Ordo(\rsigma + \lip_*\rrho)$ from \cref{thm:event}.
\end{proof}

The bound $\bnd_1$ in \cref{thm:loose-param-bounds} scales with $\sqrt{n}$, making it too loose to directly bound the Lipschitz constant of $f_n$ for establishing the near-minimax rate.
However, we can still use the bounded class $\setGbar_\somenorm(\setXhat_K)$ in the concentration inequality of \cref{thm:conc-ineq-inner} to upper bound the inner product terms of \eqref{eq:Eapprox-4} in the next result.
The key observation is that $\bnd_1$ appears only inside the logarithmic term.

\begin{lemma} \label{thm:inner-product-terms}
  Let $g_*$ be an approximation of $f_* \in \setF_{\lip_*,\setX_*}$ as in \cref{thm:gstar}.
  For any $\delta \in (0,\bnd_1]$, define event $\event_{\gamma,\delta}$ as
  \begin{equation*} \begin{split}
    \event_{\gamma,\delta} \defeq \event_\gamma \cap \bigg\{
    \langle f_n - g_n, f_* - y \rangle_n
    &= \Ordo\Big(\sigma \big(\norm{f_n-g_n}_n + \delta\big) \sqrt{d K \ln\big(\bnd_1/(\delta\gamma)\big) / n} + \rsigma \delta \Big) ,
    \\
    \langle g_* - f_* , f_* - y \rangle_n
    &= \Ordo\Big( \sigma \norm{g_* - f_*}_n \sqrt{\ln(1/\gamma)/n} \Big),
    \\
    \langle g_* - g_n , f_* - y \rangle_n
    &= \Ordo\Big(\sigma \big(\norm{g_* - g_n}_n + \delta\big) \sqrt{d K \ln\big(\bnd_1/(\delta\gamma)\big)/n} + \rsigma \delta \Big)
      \bigg\} .
  \end{split} \end{equation*}
Then, $\Prob\{\event_{\gamma,\delta}\} \ge 1 - 5\gamma$.
\end{lemma}
\begin{proof}
  Define the metric $\psi$ by $\psi(h,\hat{h}) \defeq \max_{k\in[K]}|b_{h,k} - b_{\hat{h},k}| + 2 R_{\setX_n} \sqrt{1+d} \norm{\vw_{h,k} - \vw_{\hat{h},k}}$ for all $h, \hat{h} \in \setH$, where $\setH \in \{\setFn(\setXhat_K), \setG_\somenorm(\setXhat_K)\}$.
  Note that $\cphi^2 \le 1+d$ by definition, and $\norm{\vX_i - \vXhat_k} \le 2R_{\setX_n}$ holds for all $i\in[n]$ and $k\in[K]$ by the triangle inequality.
  Therefore, by \cref{thm:pointwise-distance}, we have $\norm{h - \hat{h}}_n \le \max_{i\in[n]}|h(\vX_i) - \hat{h}(\vX_i)| \le \psi(h,\hat{h})$.
  Additionally, note that $\psi(g,\hat{g}) = \Ordo(\bnd_1)$ and $\psi\big(\proj(g),\proj(\hat{g})\big) = \Ordo(\bnd_1)$ for all $g, \hat{g} \in \setGbar_\somenorm(\setXhat_K)$.
  Furthermore, event $\event_{\gamma}$ implies $\norm{f_* - y}_n \le \rsigma$.

Let $\setH_1 \defeq \{f-g : g \in \setGbar_\somenorm(\setXhat_K), f = \proj(g)\}$, where $f_n - g_n \in \setH_1$ by $f_n = \proj(g_n)$ and \cref{thm:loose-param-bounds}.
For all $h, \hat{h} \in \setH_1$ with $h = f - g$ and $\hat{h} = \hat{f} - \hat{g}$, define the metric $\psi_1(h,\hat{h}) \defeq \psi(f,\hat{f}) + \psi(g,\hat{g})$.
By using \cref{thm:func-cover}, the bound $\bnd_1$ on the parameter magnitudes within $\setGbar_\somenorm(\setXhat_K)$, and since the parameters of the functions $g$ and $\proj(g)$ are the same, we have $\ln N_{\psi_1}(\setH_1,\delta) = \Ordo\big(d K \ln(\bnd_1/\delta)\big)$ for all $\delta \in (0,\bnd_1]$.
Then, the first inequality in $\event_{\gamma,\delta}$ holds with probability at least $1-\gamma$ by \cref{thm:conc-ineq-inner}, using $\setH_n \leftarrow \setH_1$, and $\psi_n \leftarrow \psi_1$.

Define the singleton set $\setH_2 \defeq \{g_* - f_*\}$.
Then, the second inequality in $\event_{\gamma,\delta}$ holds with probability at least $1-\gamma$ by \cref{thm:conc-ineq-inner}, using $\setH_n \leftarrow \setH_2$, the zero constant function for $\psi_n$, and taking the limit $\delta \to 0$.

Let $\setH_3 \defeq \{g_* - g : g \in \setGbar_\somenorm(\setXhat_K)\}$, where $g_* - g_n \in \setH_3$ by \cref{thm:loose-param-bounds}.
For all $h, \hat{h} \in \setH_3$ with $h = g_* - g$ and $\hat{h} = g_* - \hat{g}$, define $\psi_3(h,\hat{h}) \defeq \psi(g,\hat{g})$.
By \cref{thm:func-cover}, and the bound $\bnd_1$ on the parameter magnitudes within $\setGbar_\somenorm(\setXhat_K)$, we have $\ln N_{\psi_3}(\setH_3,\delta) = \Ordo\big(d K \ln(\bnd_1/\delta)\big)$ for all $\delta \in (0,\bnd_1]$.
Then, the third inequality in $\event_{\gamma,\delta}$ holds with probability at least $1-\gamma$ by \cref{thm:conc-ineq-inner}, using $\setH_n \leftarrow \setH_3$, and $\psi_n \leftarrow \psi_3$.

Finally, the result follows by combining the three cases with \cref{thm:event}.
\end{proof}

We now bound $\norm{g_n - g_*}_n$ and $\lipf_{f_n}$ by combining the result of \cref{thm:inner-product-terms} with the ``basic inequality'' \eqref{eq:erm-ineq-rearranged}.
This, in turn, yields a bound on the approximation error $\Eapprox$ via \eqref{eq:Eapprox-4}.

\begin{lemma} \label{thm:Eapprox}
  Let $g_*$ be an approximation of $f_* \in \setF_{\lip_*,\setX_*}$ as in \cref{thm:gstar}.
  Set $\delta_n \defeq \rsigma \sqrt{d} K/n$, and suppose that $\event_{\gamma,\delta_n}$ and \eqref{eq:dcf-params} hold.
  Then, for some $\lipf_0 > 0$, the following bounds hold:
  \begin{equation*} \begin{split}
    \lipf_{f_n}^2 &\le \lipf_0^2 = \Theta\big(\theta_0^2 + \clip^2 \lip_*^2 + \sigma^2 \ln(\bnd_2/\gamma)\big) ,\\
    \norm{g_n - g_*}_n^2 &= \Ordo\bigg(\frac{d K}{n} \Big( \clip^2 (1+\rrho^2)\lip_*^2 + \sigma^2\ln(\bnd_2/\gamma) \Big)\bigg) , \\
    \Eapprox &= \Ordo\Big(\frac{d K}{n} (1 + \rrho^2) \lipf_0^2\Big) ,
  \end{split} \end{equation*}
  where $\bnd_2 \defeq n \bnd_1 / (\rsigma\sqrt{d})$ satisfies $n \le \bnd_2 = \Theta\big(n^{3/2}(1 + \lip_* \rho / \sigma)\big)$.
\end{lemma}
\begin{proof}
  Notice that $\delta_n \in (0,\bnd_1]$ by definition since $K \le n$.
  Using $ab = (a/c)(cb) \le a^2/(2c^2) + b^2(c^2/2)$ for all $a, b \in \setR$ and $c > 0$, we obtain from \cref{thm:inner-product-terms} and $\bnd_1/\delta_n = \Ordo(\bnd_2)$ that
  \begin{equation} \label{eq:Eapprox-proof-1} \begin{split}
    \langle g_* - g_n , f_* - y \rangle_n
    &= \frac18\norm{g_* - g_n}_n^2 + \Ordo\Big(\delta_n^2 + \frac{d K \sigma^2}{n}\ln\big(\bnd_1/(\delta_n \gamma)\big) + \rsigma \delta_n\Big)
    \\
    &= \frac18\norm{g_* - g_n}_n^2 + \Ordo\Big(\frac{d K \sigma^2}{n} \ln(\bnd_2/\gamma)\Big) .
  \end{split}\end{equation}
From \cref{thm:gstar}, we have $\norm{g_* - f_*}_n^2 = \Ordo\big(d \lip_*^2 \epsilon_n^2(\setXhat_K)\big)$ by $\capx^2 = \Ordo(d)$, and $\regz_{\vtheta}(g_*) = \Ordo\big(\theta_1 \clip^2 \lip_*^2\big)$ by $\theta_1 \ge K\theta_2$ from \eqref{eq:dcf-params}.
Recall that $\theta_1 = \Theta\big(\max\{1,R_{\setX_n}^2\} d K / n\big)$ from \eqref{eq:dcf-params}.
Further, $\epsilon_n^2(\setXhat_K) = \Ordo(R_{\setX_n}^2 K / n)$ from \cref{thm:AFPC}.
Therefore, using $\regz_{\vtheta}(g_n) \ge \theta_1 (\lipf_{f_n} - \theta_0)_+^2$, $\lipf_{f_n} = \lipf_{g_n}$, $\clip \ge 1$, and combining \eqref{eq:Eapprox-proof-1} with the ``basic inequality'' \eqref{eq:erm-ineq-rearranged}, we obtain
\begin{equation} \label{eq:Eapprox-proof-2}
  \frac14 \norm{g_* - g_n}_n^2 + \max\{1,R_{\setX_n}^2\}\frac{d K}{n} (\lipf_{f_n} - \theta_0)_+^2
  = \Ordo\bigg(\hspace{-0.5mm}\frac{d K}{n}\Big(\clip^2 \max\{1,R_{\setX_n}^2\} \lip_*^2 + \sigma^2 \ln(\bnd_2/\gamma)\Big)\hspace{-1mm}\bigg) .
\end{equation}
Then, the claims for $\lipf_{f_n}^2$ and $\norm{g_n - g_*}_n^2$ follow from \eqref{eq:Eapprox-proof-2} after rearranging terms, and applying $R_{\setX_n} = \Ordo(\rrho)$ which follows from $\event_{\gamma,\delta_n} \subset \event_\gamma$ by \cref{thm:event}.

Similarly to the derivation of \eqref{eq:Eapprox-proof-1}, \cref{thm:inner-product-terms} yields for $\delta_n$ that
\begin{equation} \label{eq:Eapprox-proof-3}
  \langle f_n - g_n, f_* - y \rangle_n + \langle g_* - f_*, f_* - y \rangle_n
  = \Ordo\Big( \norm{f_n - g_n}_n^2 + \norm{g_* - f_*}_n^2 + \frac{d K }{n} \sigma^2 \ln(\bnd_2 / \gamma) \Big) .
\end{equation}
Since $\norm{f_n - g_n}_n^2 = \Ordo\big(\lipf_{f_n}^2 \lipphi^2 \epsilon_n^2(\setXhat_K)\big)$ by \cref{thm:distance-fngn}, and $\norm{g_* - f_*}_n^2 = \Ordo\big(\capx^2 \lip_*^2 \epsilon_n^2(\setXhat_K)\big)$ by \cref{thm:gstar}, we upper bound the approximation error $\Eapprox$ by combining \eqref{eq:Eapprox-4} with \eqref{eq:Eapprox-proof-1}, the bound on $\norm{g_*-g_n}_n^2$, and \eqref{eq:Eapprox-proof-3} as
\begin{equation} \begin{split}
  \Eapprox
  = \, &\Ordo\bigg(
         \big(\lipf_{f_n}^2\lipphi^2 + \capx^2 \lip_*^2\big) \epsilon_n^2(\setXhat_K) + \theta_1 \clip^2 \lip_*^2
         + \frac{d K}{n}\Big(\clip^2 (1 + \rrho^2) \lip_*^2 + \sigma^2 \ln(\bnd_2/\gamma)\Big)
         \bigg) ,
\end{split} \end{equation}
which proves the claim on $\Eapprox$, as $\capx^2 = \Ordo(d)$ by definition, $\max\{\lip_*, \clip\lip_*, \lipf_{f_n}\} \le \lipf_0$, $\epsilon_n^2(\setXhat_K) = \Ordo(R_{\setX_n}^2K/n)$ by \cref{thm:AFPC}, $R_{\setX_n} = \Ordo(\rrho)$, and $\lipf_{f_n}^2 \lipphi^2 \epsilon_n^2(\setXhat_K) = \Ordo\big(\frac{dK}{n} \rrho^2 \lipf_0^2\big)$ by $\lipphi^2 = \Ordo(d)$.
The bounds on $\bnd_2$ hold by definition.
\end{proof}

In the proof of \cref{thm:Eapprox}, for bounding $\lipf_{f_n}$, we used that $\theta_1$ in \eqref{eq:dcf-params} scales with $\max\{1,R_{\setX_n}^2\}$ rather than just $R_{\setX_n}^2$.
This choice is necessary because, in the latter case, we cannot ensure that $\sigma^2 / R_{\setX_n}^2$ remains upper bounded in the setting of \cref{thm:near-minimax-rate}.

Notice that the bound $\bnd_1$ on the slope parameters from \cref{thm:loose-param-bounds} is improved by the bound on $\lipf_{f_n}$ from \cref{thm:Eapprox}, replacing the earlier $\sqrt{n}$ dependence with $\theta_0 + \sqrt{\ln(n)}$.
This improvement will be important for applying our concentration inequality in the random design setting and obtaining the near-minimax rate.

\subsubsection{Applying the Concentration Inequality in the Random Design}
\label{sec:conc-ineq}

After bounding the approximation error $\Eapprox = \norm{\fni - y}_n^2 - \norm{f_* - y}_n^2$ of the DCF estimator $\fni$ to the regression function $f_*$, it remains to bound the first term in \eqref{eq:error-decomp}, namely \mbox{$\norm{\fni - f_*}_*^2 - \offsetc\Eapprox$}, in order to complete the proof of \cref{thm:near-minimax-rate}.
To this end, we use the following concentration inequality, which builds on the Bernstein inequality and leverages the results of \citet{BalazsGyorgySzepesvari2016}, as applied here to the product of subgaussian random variables.
\begin{lemma} \label{thm:conc-ineq-Bernstein}
  Let $\setF$ be a finite, nonempty set, and $n\in\setN$.
  For each $f\in\setF$ and \mbox{$i\in[n] \cup \{0\}$}, let $Z_{f,i}, W_{f,i}$ be real-valued, subgaussian random variables satisfying $\E\big[e^{Z_{f,i}^2/\omega^2}\big] \le 2$ and $\E\big[e^{W_{f,i}^2/\nu^2}\big] \le 2$ for some $\omega, \nu > 0$.
  Suppose that, for each fixed $f \in \setF$, the random pairs $(Z_{f,0},W_{f,0}), (Z_{f,1},W_{f,1}), \ldots, (Z_{f,n},W_{f,n})$ are \iid,
  and that there exist constants $\alpha, \mu > 0$ such that $\mu^2 \le \E[Z_{f,0}^2] \le \alpha \E[Z_{f,0}W_{f,0}]$ for all $f \in \setF$.
  Then, for all $\gamma \in (0,1)$, it holds with probability at least $1-\gamma$ that
  \begin{equation*}
    \max_{f\in\setF}\Big\{\E\big[Z_{f,0}W_{f,0}\big] - \frac{2}{n}\sum_{i=1}^n Z_{f,i}W_{f,i}\Big\}
    \le 16 (\omega + 2\alpha\nu) \nu \, \frac{\ln(3\omega/\mu)}{n} \, \ln(|\setF|/\gamma) .
  \end{equation*}
\end{lemma}
\begin{proof}
  See \cref{sec:proof-conc-ineq}.
\end{proof}

The condition $\E[Z_{f,0}^2] \le \alpha \E[Z_{f,0}W_{f,0}]$ in \cref{thm:conc-ineq-Bernstein} is related to the widely used Bernstein condition \citep[Definition~2.6]{BartlettMendelson2006}.

In order to apply \cref{thm:conc-ineq-Bernstein}, we need to show that $\fni$ belongs to a (non-random) function class with a bounded covering number.
From \cref{thm:fniL} we already know that $\lipf_{\fni} \le (1+\theta_3) \lipf_{f_n}$, where $\lipf_{f_n}$ is further bounded in \cref{thm:Eapprox} as $\lipf_{f_n} \le \lipf_0$.
The next result provides the missing upper bound on the magnitudes of the bias parameters of $\fni$.
This bound follows from properties of $\fni$ ensured by the final step in the definition of the DCF estimator~\eqref{eq:dcf-estimator}.
Unlike in \cref{sec:Eapprox-bound}, here the bound on the bias terms cannot be allowed to scale polynomially with the sample size~$n$.

\begin{lemma} \label{thm:fni-bias-bound}
  Suppose that event $\event_{\gamma,\delta_n}$ and \eqref{eq:dcf-params} hold.
  Then, there exists $\bnd_0 > 0$ such that $\max_{k\in\inds} |b_{\fni,k} - y_0| \le \bnd_0$, and $\bnd_0 = \Theta\big( (1 + \rrho) (\cphi + \lipphi) \theta_3 \lipf_0\big)$.
  Furthermore, we have $\risk_n(\fni) = \Ordo\big((1 + \rrho^2) \theta_3^2 \lipf_0^2 \lipphi^2\big)$.
\end{lemma}
\begin{proof}
  For each $k\in\inds$, let $i_k \in [n]$ be such that $\fni(\vX_{i_k}) = b_{\fni,k} + \vw_{\fni,k}^\T\phin(\vX_{i_k},\vXhat_k)$, which always exists by the definition of $\inds$ in \eqref{eq:dcf-estimator}.
  The triangle inequality and $\event_{\gamma,\delta_n} \subset \event_\gamma$ yield $\norm{\vX_{i_k} - \vXhat_k} \le 2\rrho$.
  Moreover, \cref{thm:fniL,thm:Eapprox} imply $\lipf_{\fni} \le (1+\theta_3)\lipf_0$.
  Then, by the triangle and Cauchy-Schwarz inequalities, and \cref{thm:phi-props}, we get
  \begin{equation} \label{eq:non-random-func-set-proof-1} \begin{split}
    |b_{\fni,k} - y_0|
    &\le |\fni(\vX_{i_k}) - y_0| + \norm{\vw_{\fni,k}} \norm{\phin(\vX_{i_k},\vXhat_k)}
    \\
    &\le |\fni(\vX_{i_k}) - y_0| + 2 \rrho \cphi (1+\theta_3) \lipf_0 
    .
  \end{split} \end{equation}
  Additionally, from \eqref{eq:dcf-estimator}, we have $\fni(\vX_i) = C_n^+ + \fnih(\vX_i)$ for all $i \in [n]$, which yields
  \begin{equation} \label{eq:centering}
    \frac1n\sum_{i=1}^n\fni(\vX_i) = C_n^+ + \frac1n \sum_{i=1}^n\fnih(\vX_i) = \Ybar .
  \end{equation}
  Since $\lipf_{\fni} \le (1+\theta_3)\lipf_0$, \cref{thm:phi-props} implies that $\fni$ is $((1+\theta_3)\lipf_0\lipphi)$-Lipschitz on $\setR^d$ \wrt\ $\norm{\cdot}$, that is $\fni \in \setF_{(1+\theta_3)\lipf_0\lipphi,\setR^d}$.
  Using this, the triangle and Jensen's inequalities, \eqref{eq:centering}, and \cref{thm:event}, $\event_\gamma$ implies for all $i \in [n]$ that
  {
    \setlength{\abovedisplayskip}{\dimexpr\abovedisplayskip-4mm\relax}
    \setlength{\belowdisplayskip}{\dimexpr\belowdisplayskip-2mm\relax}
    \setlength{\abovedisplayshortskip}{\dimexpr\abovedisplayshortskip-4mm\relax}
    \setlength{\belowdisplayshortskip}{\dimexpr\belowdisplayshortskip-2mm\relax}
    \begin{equation} \label{eq:non-random-func-set-proof-2} \begin{split}
      |\fni(\vX_{i}) - y_0|
      &\le \Big|\fni(\vX_{i}) - \frac1n\sum_{j=1}^n\fni(\vX_j)\Big| + |\Ybar - y_0|
      \\[-1mm]
      &\le 2 (1+\theta_3) \rrho \lipf_0 \lipphi + (\rsigma + 2 \lip_* \rrho)
        .
    \end{split} \end{equation}
  }%
We prove the bound $\bnd_0$ on the bias by combining \eqref{eq:non-random-func-set-proof-1} and \eqref{eq:non-random-func-set-proof-2}, and using $\theta_3 \ge 1$ with $\max\{\lip_*, \rsigma\} = \Ordo(\lipf_0)$.

  By using the triangle inequality and \cref{thm:event}, event $\event_\gamma$ also implies $|\fni(\vX_{i})-Y_{i}| \le |\fni(\vX_{i})-y_0| + \Ordo(\rsigma + \lipf_0\rrho)$.
  The bound on $\risk_n(\fni)$ then follows from \eqref{eq:non-random-func-set-proof-2} by squaring and averaging over all $i \in [n]$.
\end{proof}

\vspace{-4mm}
Having bounded the parameter magnitudes of $\fni$, we now construct a bounded, non-random function class that always contains a function uniformly approximating $\fni$ on the entire space $\setR^d$ to the required accuracy.
This construction relies on the fact that $\phin(\cdot,\cdot)$ is Lipschitz in its second argument (\cref{thm:phi-props}).

We also need the notion of tuples: the set of all tuples of size $k \in \setN$ with elements from a set $\setX \subseteq \setR^d$ is defined as $\setT_k(\setX) \defeq \{\langle \vxtilde_1,\ldots,\vxtilde_k\rangle : \vxtilde_j\in\setX, j\in[k]\}$.
We extend the definition of $\setFn(\setXtilde)$ to allow $\setXtilde$ to be a tuple, so that multiple function parameters may be associated with the same center.

The following result presents the construction mentioned above:
\vspace{-1mm}
\begin{lemma} \label{thm:fni-approx}
  Suppose the conditions of \cref{thm:fni-bias-bound} hold, and let $k_*$ be as in \cref{thm:AFPC}.
  Further, let $\setX_\rho \defeq \{\vx\in\setX_* : \norm{\vx-\E[\vX]} \le \rrho\}$, $\eta > 0$, and $\setXhat_{\rho,\eta}$ be an $\eta$-cover of $\setX_\rho$ \wrt\ $\norm{\cdot}$.
  Define the following (non-random) function class:
  {
    \setlength{\abovedisplayskip}{\dimexpr\abovedisplayskip-3mm\relax}
    \setlength{\belowdisplayskip}{\dimexpr\belowdisplayskip-3mm\relax}
    \setlength{\abovedisplayshortskip}{\dimexpr\abovedisplayshortskip-3mm\relax}
    \setlength{\belowdisplayshortskip}{\dimexpr\belowdisplayshortskip-3mm\relax}
    \begin{equation*}
      \setFbar_{\somenorm,\eta}
      \defeq \bigcup_{k \in [k_*]} \bigcup_{\setXtilde \in \setT_k(\setXhat_{\rho,\eta})} \setFbar_\somenorm(\setXtilde)
    \end{equation*}%
  }%
  where $
  \setFbar_\somenorm(\setXtilde)
  \defeq \big\{ f \in \setFn(\setXtilde)
  : |b_{f,k} - y_0| \le \bnd_0 ,\, \norm{\vw_{f,k}} \le (1+\theta_3) \lipf_0 ,\, k \in [|\setXtilde|]
  \big\}$ for all $\setXtilde \in \setT_k(\setXhat_{\rho,\eta})$ and $k\in\setN$, restricting $\setFn(\setXtilde)$ to functions with bounded parameter magnitudes.
  Then, there exists $f \in \setFbar_{\somenorm,\eta}$ such that $\max_{\vx\in\setR^d}|\fni(\vx)-f(\vx)| \le (1+\theta_3) \lipf_0 \lipphi \eta$.
\end{lemma}
\begin{proof}
  Let $K_+ \defeq |\inds|$ and write $\inds \defeq \{j_1,\ldots,j_{K_+}\}$ with some indices $j_1,\ldots,j_{K_+} \in [K]$.
  By \cref{thm:event}, $\event_\gamma$ implies $\setXhat_K \subseteq \setX_n \subset \setX_\rho$.
  Since $K \le k_*$ by \cref{thm:AFPC}, it follows that $1 \le K_+ \le K \le k_*$, and since $\setXhat_{\rho,\eta}$ is an $\eta$-cover of $\setX_\rho$ \wrt\ $\norm{\cdot}$, we can select a tuple $\setXtilde \defeq \langle\vxtilde_1,\ldots,\vxtilde_{K_+}\rangle \in \setT_{K_+}(\setXhat_{\rho,\eta})$ such that $\setFbar_\somenorm(\setXtilde) \subset \setFbar_{\somenorm,\eta}$ and $\norm{\vXhat_{j_k} - \vxtilde_k} \le \eta$ for all $k\in[K_+]$.
  The reason for using tuples is that we cannot guarantee that distinct elements from $\hat{\setX}_{\rho,\eta}$ can be associated with all of $\vXhat_{j_1},\ldots,\vXhat_{j_{K_+}}$.

  Recall that we have $\lipf_{\fni} \le (1+\theta_3)\lipf_{f_n} \le (1+\theta_3)\lipf_0$ from \cref{thm:fniL,thm:Eapprox}, and $\max_{k\in\inds}|b_{\fni,k} - y_0| \le \bnd_0$ from \cref{thm:fni-bias-bound}.
  Using this, we define function $f \in \setFbar_\somenorm(\setXtilde)$ by setting $b_{f,k} \defeq b_{\fni,j_k}$ and $\vw_{f,k} \defeq \vw_{\fni,j_k}$ for all $k \in [K_+]$.
  Then, for all $\vx\in\setR^d$, we have by the Cauchy-Schwarz inequality and the $\lipphi$-Lipschitzness of $\phin(\vx,\cdot)$ from \cref{thm:phi-props} that
  {
    \setlength{\abovedisplayskip}{\dimexpr\abovedisplayskip-4mm\relax}
    \setlength{\belowdisplayskip}{\dimexpr\belowdisplayskip-3mm\relax}
    \setlength{\abovedisplayshortskip}{\dimexpr\abovedisplayshortskip-4mm\relax}
    \setlength{\belowdisplayshortskip}{\dimexpr\belowdisplayshortskip-3mm\relax}
    \begin{equation*}
      \big|\fni(\vx) - f(\vx)\big|
      \le \max_{k\in[K_+]}\bigg|\Big(\phin(\vx,\vXhat_{j_k})-\phin(\vx,\vxtilde_k)\Big)^\T\vw_{f,k}\bigg|
      \le \max_{k\in[K_+]} \lipphi \norm{\vXhat_{j_k} - \vxtilde_k} \norm{\vw_{f,k}} ,
    \end{equation*}
  }%
  which proves the result as $\norm{\vXhat_{j_k} - \vxtilde_k} \le \eta$ and $\norm{\vw_{f,k}} \le \lipf_{\fni} \le (1+\theta_3) \lipf_0$ for all $k\in[K_+]$.
\end{proof}

As the concentration inequality in \cref{thm:conc-ineq-Bernstein} requires a finite function set, next we construct a cover of $\setFbar_{\somenorm,\eta}$.
\begin{lemma} \label{thm:fni-approx-cover}
  Suppose that the conditions of \cref{thm:fni-approx} hold, and that the cover $\setXhat_{\rho,\eta}$ of $\setX_\rho$ \wrt~$\norm{\cdot}$ is of minimal cardinality.
  For all $k_0\in\setN$ and $\setXtilde \in \setT_{k_0}(\setR^d)$, define a metric between any $f, \hat{f} \in \setFn(\setXtilde)$ as
  \begin{equation*}
    \psi_{k_0}(f,\hat{f}) \defeq \max_{k\in[k_0]} |b_{f,k}-b_{\hat{f},k}| + \rrho\norm{\vw_{f,k}-\vw_{\hat{f},k}} .
  \end{equation*}
  Let $\eta \in (0,\rrho/2]$ and $\delta \in (0,\bnd_0]$.
  For all $k\in\setN$ and $\setXtilde \in \setT_k(\setXhat_{\rho,\eta})$, define $\hat{\setFbar}_{\somenorm,\delta}(\setXtilde)$ to be a $\delta$-cover of $\setFbar_{\somenorm}(\setXtilde)$ \wrt\ $\psi_k$ of minimal cardinality.
  Finally, let
  \begin{equation*} \begin{split}
    \hat{\setFbar}_{\somenorm,\eta,\delta}
    \defeq
    \bigcup_{k \in [k_*]} \bigcup_{\setXtilde \in \setT_k(\setXhat_{\rho,\eta})}
    \hat{\setFbar}_{\somenorm,\delta}(\setXtilde)
    .
  \end{split} \end{equation*}
Then, for every function $f\in\setFbar_{\somenorm,\eta}$ there exists $\hat{f} \in \hat{\setFbar}_{\somenorm,\eta,\delta}$ which satisfies $|f(\vx) - \hat{f}(\vx)| \le 2\delta \cphi \big(1 + \norm{\vx-\E[\vX]}/\rrho\big)$ for all $\vx\in\setR^d$.
Additionally, $\ln\big|\hat{\setFbar}_{\somenorm,\eta,\delta}\big| = \Ordo\big(d k_* \ln\big(\rrho \bnd_0/(\eta\delta)\big)\big)$.
\end{lemma}
\begin{proof}
  Note that the result of \cref{thm:pointwise-distance} extends straightforwardly to $\setFn(\setXtilde)$, where $\setXtilde = \langle \vxtilde_1,\ldots,\vxtilde_{k_0}\rangle$ is a tuple.
  Therefore, we have for all $f, \hat{f} \in \setFn(\setXtilde)$ and for all $\vx\in\setR^d$ that
  \begin{equation} \label{eq:fni-approx-cover-proof-1}
    \big|f(\vx) - \hat{f}(\vx)\big| \le \psi_{k_0}(f,\hat{f}) \Big(1+\max_{k\in[k_0]}\cphi\,\norm{\vx-\vxtilde_k}/\rrho\Big) .
  \end{equation}
  Fix $f \in \setFbar_{\somenorm,\eta}$ arbitrarily, and let $\setXtilde \defeq \langle\vxtilde_1,\ldots,\vxtilde_{k_0}\rangle \in \setT_{k_0}(\setXhat_{\rho,\eta})$ such that $f \in \setFbar_\somenorm(\setXtilde)$.
  By the definition of $\delta$-cover \wrt\ $\psi_{k_0}$, we choose $\hat{f} \in \hat{\setFbar}_{\somenorm,\delta}(\setXtilde) \subset \hat{\setFbar}_{\somenorm,\eta,\delta}$ to be such that $\psi_{k_0}(f,\hat{f}) \le \delta$.
  Now fix any $\vx\in\setR^d$.
  Since $\vxtilde_k \in \setX_\rho$ for all $k\in[k_0]$, we have $\norm{\vxtilde_k-\E[\vX]} \le \rrho$, and the triangle inequality yields $\max_{k\in[k_0]}\norm{\vx-\vxtilde_k} \le \norm{\vx-\E[\vX]} + \rrho$.
  Therefore, \eqref{eq:fni-approx-cover-proof-1} and $\cphi \ge 1$ imply the claimed upper bound on $|f(\vx)-\hat{f}(\vx)|$.

  Define $\setT_{\eta,k} \defeq \setT_k(\setXhat_{\rho,\eta})$ for all $k\in[k_*]$.
  By \cref{thm:volume-argument}, $|\setT_{\eta,k}| = |\setXhat_{\rho,\eta}|^k = \Ordo\big((\rrho/\eta)^{dk}\big)$.
  Further, for all $\setXtilde \in \setT_{\eta,k}$, the bounds on the bias and slope parameters of any $f, \hat{f} \in \setFbar_{\somenorm}(\setXtilde)$ imply $\psi_k(f,\hat{f}) \le 2(\bnd_0 + (1+\theta_3)\rrho \lipf_0) = \Ordo(\bnd_0)$.
  Therefore, $N_{\psi_k}\big(\setFbar_{\somenorm}(\setXtilde),\delta\big) = \Ordo\big((\bnd_0/\delta)^{(1+d_\somenorm)k}\big)$ by \cref{thm:func-cover}.
  Note that $\sum_{k\in[k_*]}t^k = (t^{k_*+1}-t)/(t-1) < 2t^{k_*}$ holds for any $t \ge 2$.
  Then, using $\rrho\bnd_0/(\eta\delta) \ge 2$ for all $\eta \in (0,\rrho/2]$ and $\delta \in (0,\bnd_0]$, and $d \le d_\somenorm$, we obtain
  \begin{equation*} \begin{split}
    |\hat{\setFbar}_{\somenorm,\eta,\delta}|
    \le \sum_{k=1}^{k_*} \sum_{\setXtilde \in \setT_{\eta,k}} \hspace{-1mm} N_{\psi_k}\big(\setFbar_\somenorm(\setXtilde),\delta\big)
    = \Ordo\Big(\sum_{k=1}^{k_*} \big|\setT_{\eta,k}\big| (\bnd_0/\delta)^{(1+d_\somenorm)k}\Big)
    = \Ordo\bigg(\Big(\frac{\rrho\bnd_0}{\eta\delta}\Big)^{(1+d_\somenorm)k_*}\bigg) ,
  \end{split} \end{equation*}
which proves the claimed upper bound on $|\hat{\setFbar}_{\somenorm,\eta,\delta}|$ with $1+d_\somenorm = \Ordo(d)$.
\end{proof}

Finally, we are ready to combine the results and prove \cref{thm:near-minimax-rate}.
\newpage
\begin{proof}\textbf{of \cref{thm:near-minimax-rate}}
  Define $\eta \defeq (\rrho/2) (K/n)$ and choose $\delta \in (0,\bnd_0]$ to be such that $\delta = \Theta\big((1+\rrho) \lipf_0 (K/n)\big)$, which is always possible since $\bnd_0^2 = \Theta\big((1+\rrho)^2(\cphi + \lipphi)^2 \theta_3^2 \lipf_0^2\big)$, $\theta_3 \ge 1$, $\min\{\cphi,\lipphi\} \ge 1$, and $K \le n$.
  Let $\tilde{f}_n^{++} \in \setFbar_{\somenorm,\eta}$ be the approximation to the DCF estimator $\fni$ from \cref{thm:fni-approx}, and $\fniapprox \in \hat{\setFbar}_{\somenorm,\eta,\delta}$ be the approximation to $\tilde{f}_n^{++}$ from \cref{thm:fni-approx-cover}.
  Define $c_{\eta,\delta} \defeq (1+\theta_3) \lipf_0 \lipphi \eta + 4\delta\cphi$, where $c_{\eta,\delta} = \Ordo\big(\sqrt{d}(K/n)(1+\rrho)\theta_3\lipf_0\big)$ by $\theta_3 \ge 1$ and $\max\{\lipphi, \cphi\} = \Ordo(\sqrt{d})$.
  Since $\rrho > \rho$ for all $n \ge 2$ and all $\gamma \in (0,1)$, we have $\E\big[\norm{\vX-\E[\vX]}^2/\rrho^2\big] \le \ln(2) < 1$ by \eqref{eq:subgaussian} and Jensen's inequality.
  Then, $\norm{\fni - \fniapprox}_* \le c_{\eta,\delta}$ follows from the triangle inequality, and \cref{thm:fni-approx,thm:fni-approx-cover}.
  Further, by using $(a+b)^2 \le 2(a^2+b^2)$ for all $a,b\in\setR$, we get
  {
    \setlength{\abovedisplayskip}{\dimexpr\abovedisplayskip-3mm\relax}
    \setlength{\belowdisplayskip}{\dimexpr\belowdisplayskip-2mm\relax}
    \setlength{\abovedisplayshortskip}{\dimexpr\abovedisplayshortskip-3mm\relax}
    \setlength{\belowdisplayshortskip}{\dimexpr\belowdisplayshortskip-2mm\relax}
    \begin{equation} \label{eq:near-minimax-rate-proof-1}
      \norm{\fni - f_*}_*^2
      \le 2\norm{\fniapprox - f_*}_*^2 + 2 c_{\eta,\delta}^2
      = 2\norm{\fniapprox - f_*}_*^2 + \Ordo\Big(\frac{d K}{n} (1+\rrho^2) \theta_3^2 \lipf_0^2\Big) .
    \end{equation}
  }%
  Notice that if $\norm{\fniapprox - f_*}_*^2 \le \bnd_0^2/n$, then the term $\norm{\fniapprox - f_*}_*^2$ in \eqref{eq:near-minimax-rate-proof-1} becomes negligible, since $(\cphi + \lipphi)^2 = \Ordo(d)$.
  Hence, we can assume, without loss of generality, that $\fniapprox \in \hat{\setFbar}_{\somenorm,\eta,\delta,\bnd_0} \defeq \big\{f \in \hat{\setFbar}_{\somenorm,\eta,\delta} : \norm{f - f_*}_*^2 > \bnd_0^2/n\big\}$.

  Combining the error decomposition \eqref{eq:error-decomp} using $\offsetc \defeq 4$, the bound of \cref{thm:Eapprox} on the approximation error $\Eapprox = \norm{\fni - y}_n^2 - \norm{f_* - y}_n^2$, and \eqref{eq:near-minimax-rate-proof-1} yields
  {
    \setlength{\abovedisplayskip}{\dimexpr\abovedisplayskip-2mm\relax}
    \setlength{\belowdisplayskip}{\dimexpr\belowdisplayskip-2mm\relax}
    \setlength{\abovedisplayshortskip}{\dimexpr\abovedisplayshortskip-2mm\relax}
    \setlength{\belowdisplayshortskip}{\dimexpr\belowdisplayshortskip-2mm\relax}
    \begin{equation} \label{eq:near-minimax-rate-proof-2}
      \norm{\fni - f_*}_*^2
      = 2\big(\norm{\fniapprox - f_*}_*^2 - 2\Eapprox\big)
      + \Ordo\Big(\frac{d K}{n} (1 + \rrho^2) \theta_3^2 \lipf_0^2\Big) .
    \end{equation}
  }%
  Event $\event_\gamma$ implies $\max_{i\in[n]}\norm{\vX_i-\E[\vX]}^2/\rrho^2 \le 1$.
  Hence, $\norm{\fni - \fniapprox}_n \le c_{\eta,\delta}$ follows from the triangle inequality, and \cref{thm:fni-approx,thm:fni-approx-cover}.
  Then, using the Cauchy-Schwarz inequality, the bound on $\norm{\fni-y}_n$ from \cref{thm:fni-bias-bound}, and $\max\{\cphi,\lipphi\} = \Ordo(\sqrt{d})$, we obtain
  {
    \setlength{\abovedisplayskip}{\dimexpr\abovedisplayskip-2mm\relax}
    \setlength{\belowdisplayskip}{\dimexpr\belowdisplayskip-2mm\relax}
    \setlength{\abovedisplayshortskip}{\dimexpr\abovedisplayshortskip-2mm\relax}
    \setlength{\belowdisplayshortskip}{\dimexpr\belowdisplayshortskip-2mm\relax}
    \begin{equation} \label{eq:near-minimax-rate-proof-3} \begin{split}
      \Eapprox
      &= \norm{\fni - y}_n^2 - \norm{f_* - y}_n^2
      \\
      &= \norm{\fniapprox - y}_n^2 + 2\langle \fni - \fniapprox , \fni - y \rangle_n - \norm{\fni - \fniapprox}_n^2 - \norm{f_* - y}_n^2
      \\
      &\ge \norm{\fniapprox - y}_n^2 - \norm{f_* - y}_n^2 - c_{\eta,\delta}^2 - 2\norm{\fni - \fniapprox}_n \norm{\fni - y}_n
      \\
      &\ge \norm{\fniapprox - y}_n^2 - \norm{f_* - y}_n^2 - c_{\eta,\delta}^2 - \Ordo\big(c_{\eta,\delta} (1+\rrho) \theta_3 \lipf_0 \lipphi\big)
      \\
      &= \norm{\fniapprox - y}_n^2 - \norm{f_* - y}_n^2 - \Ordo\Big( \frac{dK}{n} (1+\rrho^2) \theta_3^2 \lipf_0^2 \Big).
    \end{split} \end{equation}
}%
For any function $f : \setR^d \to \setR$, define $Z_{f,0} \defeq f(\vX) - f_*(\vX)$, $W_{f,0} \defeq f(\vX) + f_*(\vX) - 2Y$ and $Z_{f,i} \defeq f(\vX_i) - f_*(\vX_i)$, $W_{f,i} \defeq f(\vX_i) + f_*(\vX_i) - 2Y_i$ for all $i \in [n]$.
Then, we have $\E[Z_{f,0}W_{f,0}] = \norm{f - y}_*^2 - \norm{f_* - y}_*^2$ and $\frac1n\sum_{i\in[n]}Z_{f,i}W_{f,i} = \norm{f - y}_n^2 - \norm{f_* - y}_n^2$, since $a^2 - b^2 = (a-b)(a+b)$ for all $a,b\in\setR$.
Note that $\norm{f - f_*}_*^2 = \norm{f - y}_*^2 - \norm{f_* - y}_*^2$ for any function $f : \setR^d \to \setR$ \citep[e.g.,][Section~1.1]{GyorfiEtAl2002}, which can be also expressed as $\E[Z_{f,0}^2] = \E[Z_{f,0}W_{f,0}]$.
Therefore, combining \eqref{eq:near-minimax-rate-proof-2} and \eqref{eq:near-minimax-rate-proof-3} implies
\begin{equation} \label{eq:near-minimax-rate-proof-4} \begin{split}
  \norm{\fni - f_*}_*^2
  &\le 2\Big(\norm{\fniapprox - y}_*^2 - \norm{f_* - y}_*^2 - 2\big(\norm{\fniapprox - y}_n^2 - \norm{f_* - y}_n^2\big)\hspace{-0.5mm}\Big)
    \\ &\hspace{8cm}
    + \Ordo\Big(\frac{d K}{n} (1 + \rrho^2) \theta_3^2 \lipf_0^2\Big)
  \\[-1mm]
  &= \Ordo\bigg(
  \max_{f \in \hat{\setFbar}_{\somenorm,\eta,\delta,\bnd_0}} \Big\{ \E[Z_{f,0}W_{f,0}] - \frac2n\sum_{i=1}^n Z_{f,i}W_{f,i} \Big\}
  + \frac{d K}{n} (1 + \rrho^2) \theta_3^2 \lipf_0^2
  \bigg) .
\end{split} \end{equation}
Note that $(Z_{f,0},W_{f,0}), \ldots, (Z_{f,n},W_{f,n})$ are \iid, and $\E[Z_{f,0}^2] \ge \bnd_0^2/n$, for all $f \in \hat{\setFbar}_{\somenorm,\eta,\delta,\bnd_0}$.

Take any $f \in \hat{\setFbar}_{\somenorm,\eta,\delta,\bnd_0}$, and let $\setXtilde \defeq \langle\vxtilde_1,\ldots,\vxtilde_{k_0}\rangle \in \setT_{k_0}(\setXhat_{\rho,\eta})$ be the centers associated with $f$ for some $k_0 \in [k_*]$, that is $f \in \hat{\setFbar}_{\somenorm,\delta}(\setXtilde)$.
Then, for any $\vx\in\setX_*$, using the triangle and Cauchy-Schwarz inequalities, \cref{thm:phi-props}, $y_0 = f_*(\vx_0)$, $\vx_0 \in \setX_*$, $f_* \in \setF_{\lip_*,\setX_*}$, and $\norm{\vxtilde_k-\E[\vX]} \le \rrho$ since $\vxtilde_k \in \setXhat_{\rho,\eta} \subseteq \setX_\rho$, we get
\begin{equation} \label{eq:near-minimax-rate-proof-5} \begin{split}
  |f(\vx) - f_*(\vx)|
  &\le \max_{k\in[k_0]}|b_{f,k} - y_0| + \norm{\phin(\vx,\vxtilde_k)} \norm{\vw_{f,k}} + |y_0 - f_*(\vx)|
  \\
  &\le \bnd_0 + \cphi(\norm{\vx - \E[\vX]} + \rrho) (1+\theta_3) \lipf_0 + \lip_* \norm{\vx-\vx_0}
    .
\end{split} \end{equation}
Since $\norm{\vx-\vx_0} \le 2\norm{\vx-\E[\vX]}$ for all $\vx \in \setX_*$ by the definition of $\vx_0$ in \cref{sec:tech-preps}, and $\E\big[e^{\norm{\vX-\E[\vX]}^2/\rrho^2}\big] \le 2$ from \eqref{eq:subgaussian}, \eqref{eq:near-minimax-rate-proof-5} yields that the random variable $Z_{f,0}$ satisfies $\E\big[e^{Z_{f,0}^2/\omega^2}\big] \le 2$ for some constant $\omega > 0$ such that $\omega = \Theta\big(\bnd_0 + \cphi \rrho (1+\theta_3) \lipf_0\big) = \Ordo(\bnd_0)$.
Similarly, using \cref{thm:event}, $\event_\gamma$ implies $\E\big[e^{Z_{f,i}^2/\omega^2}\big] \le 2$ for all $i \in [n]$.
Additionally, by writing $W_{f,0} = Z_{f,0} + 2(f_*(\vX)-Y)$ and $W_{f,i} = Z_{f,i} + 2(f_*(\vX_i)-Y_i)$ for all $i \in [n]$, we have $\E\big[e^{W_{f,i}^2/\nu^2}\big] \le 2$ for some $\nu > 0$ satisfying $\nu = \Theta(\omega + \rsigma) = \Ordo(\bnd_0)$, for all $i \in [n] \cup \{0\}$.

Then, by applying \cref{thm:conc-ineq-Bernstein} with $\alpha = 1$, $\mu = \bnd_0/\sqrt{n}$, $\omega \le \nu = \Ordo(\bnd_0)$, and using $\big|\hat{\setFbar}_{\somenorm,\eta,\delta,\bnd_0}\big| \le \big|\hat{\setFbar}_{\somenorm,\eta,\delta}\big|$ with the bound from \cref{thm:fni-approx-cover}, we get with probability at least $1-\gamma$ that
\begin{equation} \label{eq:near-minimax-rate-proof-6} \begin{split}
  \max_{f \in \hat{\setFbar}_{\somenorm,\eta,\delta,\bnd_0}} \hspace{-1mm} \Big\{ \E[Z_{f,0}W_{f,0}] - \frac2n\sum_{i=1}^n Z_{f,i}W_{f,i} \Big\}
  &= \Ordo\bigg(\nu^2 \frac{\ln(n)}{n} d k_* \ln\Big(\frac{\rrho \bnd_0}{\eta \delta \gamma}\Big)\bigg)
  \\
  &= \Ordo\Big( \frac{d k_*}{n} \bnd_0^2 \ln(n) \ln(d n / \gamma) \Big)
  ,
\end{split} \end{equation}
where in the last step we simplified using $\eta \delta = \Omega\big(\rrho(1+\rrho)\lipf_0/n^2\big)$, $\cphi = \Ordo(\sqrt{d})$, $\theta_3 = \Ordo\big(\ln(n)\big) = \Ordo(\sqrt{n})$ from \eqref{eq:dcf-params}, so $(\rrho \bnd_0) / (\eta\delta) = \Ordo\big(n^2 \bnd_0 / ((1+\rrho) \lipf_0)\big) = \Ordo\big(n^2\sqrt{dn}\big)$.

We finally prove \cref{thm:near-minimax-rate} by combining \eqref{eq:near-minimax-rate-proof-4} and \eqref{eq:near-minimax-rate-proof-6}, together with the bound $K \le k_* = \Ordo(n^{d_*/(2+d_*)})$ from \cref{thm:AFPC}.
We conclude the proof by appropriately rescaling $\gamma$, and simplifying the bound by using $\bnd_0^2 = \Ordo\big((1+\rrho^2)(\cphi+\lipphi)^2\theta_3^2 \lipf_0^2\big)$ from \cref{thm:fni-bias-bound}, $\lipf_0^2 = \Ordo\big(\theta_0^2 + \clip^2\lip_*^2+\sigma^2\ln(\bnd_2/\gamma)\big)$ from \cref{thm:Eapprox}, $\theta_0^2 = \Ordo\big((\rsigma^2 + \lip_*^2)\ln^2(n)\big)$ from \eqref{eq:dcf-params} and \cref{thm:event}, $\bnd_2 = \Ordo(\bnd_{\ln}^2)$, and $(\cphi + \lipphi)^2\clip^2 = \Ordo\big(1+d\ind\{\somenorm \ne 2\}\big)$ from \cref{thm:phi-props,thm:gstar}.
\end{proof}

\subsubsection{Proof of the Concentration Inequality}
\label{sec:proof-conc-ineq}

In this section, we present the deferred proof of \cref{thm:conc-ineq-Bernstein}.
To this end, we employ the following variant of the Bernstein inequality, applied to products of subgaussian random variables:
\begin{lemma} \label{thm:bernstein-ineq}
  Let $Z$ and $W$ be real-valued random variables satisfying $\E\big[e^{Z^2/\omega^2}\big] \le 2$ and $\E\big[e^{W^2/\nu^2}\big] \le 2$ for some constants $\omega, \nu > 0$, and $\E[Z^2] > 0$.
  Define the \emph{kurtosis} of $Z$ by $\K[Z] \defeq \E[Z^4]/\E[Z^2]^2$, and let $c \ge 4\ln(4\sqrt{\K[Z]})$.
  Then, the following hold:
  \begin{itemize}
  \item[(a)] $\E\big[|Z W|^k\big] \le (k!/2) \E[Z^2] (2c\,\nu^2) (c\,\omega \nu)^{k-2}$ for all integers $k \ge 2$,
  \item[(b)] $\E\big[e^{s(\E[ZW]-ZW)}\big] \le \exp\Big(\tfrac{c\,\nu^2 s^2 \E[Z^2]}{1 - c\,\omega\nu s}\Big)$ for all $s \in \big(0,1/(c\,\omega\nu)\big)$.
  \end{itemize}
\end{lemma}
\begin{proof}
  Part (a) was proved in Lemma~A.5 of \citet{Balazs2016}.
  Part (b) follows from Theorem~2.10 of \citet{BoucheronEtAl2013}, using part (a).
\end{proof}

Next, we combine \cref{thm:bernstein-ineq} with the ideas of \citet{BalazsGyorgySzepesvari2016} to prove \cref{thm:conc-ineq-Bernstein}.
\medskip
\begin{proof}\textbf{of \cref{thm:conc-ineq-Bernstein}}
Introduce the shorthand notation $Z_f \defeq Z_{f,0}$ and $W_f \defeq W_{f,0}$.
By Lemma~A.2 of \citet{Balazs2016}, we have $\E[Z_f^4] \le 2(2/e)^2\omega^4$.
Since $\E[Z_f^2] \ge \mu^2$, it follows that $\K[Z_f] \le 2(\omega/\mu)^4$ for all $f\in\setF$.
Define $c \defeq 8\ln(3\omega/\mu)$, which satisfies $c \ge \max_{f\in\setF}4\ln(4\sqrt{\K[Z_f]})$.
Then, by applying \cref{thm:bernstein-ineq} with $\E[Z_f^2] \le \alpha \E[Z_f W_f]$, we obtain for all $s \in (0,1)$, $i\in[n]$, and $f\in\setF$ that
\begin{equation} \label{eq:bernstein}
  \E\Big[e^{s(\E[Z_f W_f]-Z_{f,i} W_{f,i})/(c\,\omega\nu)}\Big]
  \le \exp\Big(\frac{c\,\nu^2 s^2 \alpha \E[Z_f W_f]}{(1-s)(c\,\omega\nu)^2}\Big)
  .
\end{equation}
Let $s \in (0,1)$ and $t > 0$ be constants to be chosen later.
Thereby, applying the union and Chernoff bounds, using the independence of $Z_fW_f, Z_{f,1}W_{f,1},\ldots,Z_{f,n}W_{f,n}$ for each fixed $f \in \setF$, and applying \eqref{eq:bernstein}, we obtain
\begin{equation*} \begin{split}
  \Prob\bigg\{ \max_{f\in\setF}\Big\{\E\big[Z_fW_f\big] - \frac{2}{n}\sum_{i=1}^n &Z_{f,i}W_{f,i}\Big\} > \frac{2tc\,\omega\nu}{s\,n} \bigg\}
  \\
  &\le \Prob\bigg\{\max_{f\in\setF}\frac{s}{2c\,\omega\nu}\sum_{i=1}^n\big(\E[Z_fW_f] - 2Z_{f,i}W_{f,i}\big) > t\bigg\}
  \\
  &\le e^{-t} \sum_{f\in\setF} \E\bigg[e^{\frac{s}{2c\,\omega\nu}\big(\sum_{i\in[n]}(\E[Z_fW_f] - 2 Z_{f,i}W_{f,i})\big)}\bigg] 
  \\
  &= e^{-t} \sum_{f\in\setF} e^{-\frac{s\,n\,\E[Z_f W_f]}{2c\,\omega\nu}} \prod_{i=1}^n \E\Big[e^{s(\E[Z_f W_f]-Z_{f,i}W_{f,i})/(c\,\omega\nu)}\Big]
  \\
  &\le e^{-t} \sum_{f\in\setF} \exp\bigg(\frac{s\,n\,\E[Z_fW_f]}{c\,\omega \nu}\Big(-\frac{1}{2} + \frac{s \nu \alpha}{(1-s)\omega}\Big)\bigg)
  \\
  &= \gamma ,
\end{split} \end{equation*}
where we set $s \defeq \omega/(\omega + 2\alpha\nu)$ and $t \defeq \ln(|\setF|/\gamma)$ for the last line.
This proves the claim.
\end{proof}

\section{Approximation Rates of Some DC Classes}

An important part of our analysis is understanding the approximation rate of the chosen function representation for the estimator to the underlying Lipschitz regression function.
To strengthen the connection between our work and the existing literature, we establish uniform approximation results for other delta-convex function classes that have been proposed to extend convex regression techniques to the more general Lipschitz setting of~\eqref{eq:data-model}.

To the best of our knowledge, the approximation rate bounds presented in this section (\cref{thm:max-min-affine-approx,thm:weakly-delta}) have not appeared in the literature.
The closest related result establishes that DC functions are dense in the class of (locally) Lipschitz functions on a closed, bounded, convex domain in $\setR^d$ \citep[Proposition~2.2]{BacakBorwein2011}.

Define the uniform norm of a function $f : \setX \to \setR$ on $\setX \subseteq \setR^d$ by $\norm{f}_{\infty,\setX} \defeq \sup_{\vx\in\setX}|f(\vx)|$.

\subsection{Max-min-affine Functions}
\label{sec:max-min-affine}

For all $k_0, l_0 \in \setN$, define the set of max-min-affine functions by
\begin{equation*} \begin{split}
  \setM_{k_0,l_0}
  \defeq \Big\{
  m : \setR^d \to \setR \,\Big|\, & m(\vx) \defeq \max_{k\in[k_0]}\min_{l\in[l_0]}b_{k,l} + \vx^\T\vw_{k,l} ,\\
   & \vx\in\setR^d ,\, b_{k,l}\in\setR ,\, \vw_{k,l}\in\setR^d ,\, k\in[k_0] ,\, l\in[l_0]
  \Big\} .
\end{split} \end{equation*}
Recall from \eqref{eq:dcf-funcs1} that $v_{f,k}$ denotes the parameter of $f \in \setF_\infty(\setXhat)$ associated with the norm term, for some finite $\setXhat \subset \setR^d$ and any $k \in [|\setXhat|]$.
We then define the class in which this parameter is restricted to be nonpositive by \mbox{$\setF_{\infty-}(\setXhat) \defeq \big\{f\in\setF_\infty(\setXhat) : v_{f,k} \le 0, k \in [|\setXhat|]\big\}$}.
The approximation in \cref{thm:approxL} belongs to $\setF_{\infty-}(\setXhat)$, thereby achieving a uniform approximation rate for max-min-affine functions, as established in the next result.
\begin{corollary} \label{thm:max-min-affine-approx}
  Let $\setX \subset \setR^d$ and suppose there exist $r, t > 0$ such that $N_{\norm{\cdot}}(\setX,\epsilon) \le (r/\epsilon)^t$ for all $\epsilon \in (0,r]$.
  Let $f \in \setF_{\lip,\setX}$ for some Lipschitz constant $\lip > 0$.
  Then, for all $k_0 \in \setN$ and $l_0 \ge 2d$, there exists $m \in \setM_{k_0,l_0}$ such that $\norm{f - m}_{\infty,\setX} \le (1+\sqrt{d}) \lip r k_0^{-1/t}$ and $m \in \setF_{\sqrt{d}\lip,\setR^d}$.
\end{corollary}
\begin{proof}
  Write the max-norm $\norm{\cdot}_\infty$ in max-linear form as $\norm{\vx}_\infty = \max_{j\in[d],s\in\{-1,1\}} s \ve_j^\T\vx$ for all $\vx\in\setR^d$, where $\ve_1,\ldots,\ve_d$ are the canonical basis vectors of $\setR^d$.
  Then, functions in $\setM_{k_0,l_0}$ with $l_0 \ge 2d$ can use the internal minimization to represent the negated max-norm $-\norm{\cdot}_\infty$, thereby allowing implementation of any $f \in \setF_{\infty-}(\setXhat)$ with any $\setXhat \defeq \{\vxhat_1,\ldots,\vxhat_{k_0}\}$ as
  \begin{equation*} \begin{split}
    f(\vx)
    = \max_{k\in[k_0]}b_{f,k} \hspace{-0.5mm}+\hspace{-0.5mm} \vu_{f,k}^\T(\vx \hspace{-0.5mm}-\hspace{-0.5mm} \vxhat_k) \hspace{-0.5mm}+\hspace{-0.5mm} v_{f,k}\norm{\vx \hspace{-0.5mm}-\hspace{-0.5mm} \vxhat_k}_\infty
    = \max_{k\in[k_0]}\hspace{-1mm}\min_{\substack{j\in[d],\\ s\in\{-1,1\}}}\hspace{-2.5mm} b_{f,k} \hspace{-0.5mm}+\hspace{-0.5mm} (\vu_{f,k} \hspace{-0.5mm}+\hspace{-0.5mm} v_{f,k} s \ve_j)\hspace{-0.5mm}^\T\hspace{-0.5mm}(\vx \hspace{-0.5mm}-\hspace{-0.5mm} \vxhat_k)
  \end{split} \end{equation*}
for all $\vx\in\setR^d$.
Hence, the function $\hat{f} \in \setF_{\infty-}(\setXhat)$ from \cref{thm:approxL} belongs to $\setM_{k_0,l_0}$.

  Set $\epsilon \defeq r k_0^{-1/t}$, which ensures that $N_{\norm{\cdot}}(\setX,\epsilon) \le (r/\epsilon)^t = k_0$.
  Choose $\setXhat \subseteq \setX$ such that $|\setXhat| = k_0$ and it contains an $\epsilon$-cover of $\setX$ \wrt\ $\norm{\cdot}$.
  The claims then follow from \cref{thm:approxL} using $\setX$, $\setXhat$, $\epsilon$, $t_0 = 1$, $t_1 = \sqrt{d}$, and $f \in \setF_{\lip,\setX}$.
\end{proof}

For any bounded set $\setX$, the covering condition of \cref{thm:max-min-affine-approx} is satisfied with $t = d$ by \cref{thm:volume-argument}.
This yields the approximation rate ${k_0}^{-1/d}$, which is known to be optimal \citep[Theorem~4.2]{DeVoreEtAl1989}.

An appealing property of the class $\setM_{k_0,l_0}$ is that it does not depend on the choice of center points $\setXhat$, unlike $\setF_{\infty-}(\setXhat)$ or $\setF_{\infty}(\setXhat)$.
In fact, $\setF_{\infty-}(\setXhat) \subset \setM_{k_0,l_0}$ for any $\setXhat \subset \setR^d$ with $k_0 \ge |\setXhat|$ and $l_0 \ge 2d$.
However, $\setM_{k_0,l_0}$ uses at least $d$ times more parameters than $\setF_{\infty}(\setXhat)$, specifically at least $2d|\setXhat|(d+1)$ versus $|\setXhat|(d+2) + |\setXhat|d$.
Moreover, we are not aware of any tractable algorithm for solving the non-convex ERM problem over the full class $\setM_{k_0,l_0}$.
Only heuristic methods have been proposed \citep[e.g.,][]{BagirovEtAl2010,BagirovEtAl2022}.

The DCF algorithm (\cref{alg:DCF}) addresses this gap in the presented nonparametric setting.
Using $\somenorm = \infty$ and additional linear constraints $v_{f,k} \le 0$ for all $k\in[K]$, expressed as $\vw_k^\T = [\vu_k^\T\,\,v_k]$ with $\vu_k\in\setR^d$ and $v_k \le 0$ in \eqref{eq:erm}, DCF computes an estimator $f_n \in \setF_{\infty-}(\setXhat_K)$ in polynomial time, where the set $\setXhat_K$ is computed by AFPC.
Since the worst-case approximation functions of \cref{thm:approxL,thm:gstar} already satisfy the nonpositivity constraints $\{v_k \le 0 : k\in[K]\}$, this $f_n$ estimator achieves the near-minimax rate of \cref{thm:near-minimax-rate} and can be converted to an equivalent representation $m_n \in \setM_{K,2d}$.

The final refinement step \eqref{eq:erm-local} is performed directly over $\setM_{K,2d}$, initialized at $m_n$, and the resulting estimator continues to satisfy the near-minimax rate of \cref{thm:near-minimax-rate}.
Our proof adapts to this case (and in fact simplifies) since $\setM_{K,2d}$ depends not on the random covariates $\setX_n$, but only on $K \le k_*$.
This leads to a substantial simplification of \cref{thm:fni-approx,thm:fni-approx-cover}, as we only need to construct a cover of the (non-random) class $\setM_{k_*,2d}$, which follows straightforwardly from \cref{thm:func-cover} after bounding the parameter space.
The convergence rate bound for $m_n$ scales with an additional factor of $d$ relative to that of \cref{thm:near-minimax-rate}, due to the larger number of parameters in $m_n$ compared to $f_n$.

Finally, the symmetrization of this max-min-affine estimator can be carried out in a similar way as for the other DCF variants, as described in \cref{sec:symmetric}.

\subsection{Approximation of Smooth Functions}
\label{sec:approxS}

We now present an approximation result analogous to \cref{thm:approxL} for smooth functions, which we use in \cref{sec:weakly-convex} to establish uniform approximation results for certain delta-convex classes.

For some $\nu > 0$, we say that a function $f : \setX \to \setR$ is $\nu$-smooth on $\setX$ \wrt\ $\norm{\cdot}$ if it is differentiable on $\setX \subseteq \setR^d$, and its gradient $\nabla f : \setX \to \setR^d$ is $\nu$-Lipschitz on $\setX$ \wrt\ $\norm{\cdot}$, that is $\norm{\nabla f(\vx) - \nabla f(\vxhat)} \le \nu\norm{\vx-\vxhat}$ holds for all $\vx,\vxhat\in\setX$.
Denote the class of $\nu$-smooth functions on $\setX$ \wrt\ $\norm{\cdot}$ by $\setF_{\nu,\setX}^\nabla$.

Then consider the following uniform approximation bound for smooth functions:
\begin{theorem} \label{thm:approxS}
  Let $\setX_\epsilon \subseteq \setX$ be an $\epsilon$-cover of a convex set $\setX \subset \setR^d$ \wrt\ $\norm{\cdot}$.
  Let $f \in \setF_{\nu,\setX}^\nabla$ for some constant $\nu > 0$, and define $\tilde{f}_1(\vx) \defeq \max_{\vxhat\in\setX_\epsilon} f(\vxhat) + \nabla f(\vxhat)^\T(\vx-\vxhat) - \nu\norm{\vx-\vxhat}^2$ for all $\vx\in\setR^d$.
  Then, $0 \le f(\vx) - \tilde{f}_1(\vx) \le 2\nu\epsilon^2$ for all $\vx\in\setX$.
\end{theorem}
\begin{proof}
  Choose $\vx \in \setX$ arbitrarily.
  Because $\setX$ is a convex set, Taylor's theorem and $f \in \setF_{\nu,\setX}^\nabla$ yield for all $\vxhat\in\setX_\epsilon$ that $f(\vx) = f(\vxhat) + \nabla f\big(t_{\vxhat}\vx + (1-t_{\vxhat})\vxhat\big)^\T(\vx-\vxhat)$ for some $t_{\vxhat} \in [0,1]$.
  Then by the Cauchy-Schwarz inequality, we get
  \begin{equation*} \begin{split}
    \tilde{f}_1(\vx) - f(\vx)
    &= \max_{\vxhat\in\setX_\epsilon} f(\vxhat) - f(\vx) + \nabla f(\vxhat)^\T(\vx-\vxhat) - \nu \norm{\vx-\vxhat}^2
    \\
    &\le \max_{\vxhat\in\setX_\epsilon}(t_{\vxhat}-1) \nu \norm{\vx-\vxhat}^2
    \\
    &\le 0
      .
  \end{split} \end{equation*}

  For the other side, we have $\min_{\vxhat\in\setX_\epsilon}\norm{\vx-\vxhat} \le \epsilon$ by the $\epsilon$-covering property.
  Then by Taylor's theorem, $f \in \setF_{\nu,\setX}^\nabla$, and the Cauchy-Schwarz inequality, we get
  {
    \setlength{\abovedisplayskip}{\dimexpr\abovedisplayskip-3mm\relax}
    \setlength{\belowdisplayskip}{\dimexpr\belowdisplayskip-4mm\relax}
    \setlength{\abovedisplayshortskip}{\dimexpr\abovedisplayshortskip-3mm\relax}
    \setlength{\belowdisplayshortskip}{\dimexpr\belowdisplayshortskip-4mm\relax}
    \begin{equation*} \begin{split}
      f(\vx) - \tilde{f}_1(\vx)
      &= \min_{\vxhat\in\setX_\epsilon} f(\vx) - f(\vxhat) - \nabla f(\vxhat)^\T(\vx-\vxhat) + \nu \norm{\vx-\vxhat}^2
      \\
      &\le \min_{\vxhat\in\setX_\epsilon} (t_{\vxhat} + 1) \nu \norm{\vx-\vxhat}^2
      \\
      &\le 2 \nu \epsilon^2
        .
    \end{split} \end{equation*}
  }%
  which proves the claim.
\end{proof}

\vspace{-3mm}
The function $\tilde{f}_1$ of \cref{thm:approxS} provides a lower approximation of $f$.
Similarly, an upper approximation is given by $\breve{f}_1$, defined as $\breve{f}_1(\vx) \defeq \min_{\vxhat\in\setX_\epsilon}f(\vxhat) + \nabla f(\vxhat)^\T(\vx-\vxhat) + \nu \norm{\vx-\vxhat}^2$ for all $\vx\in\setR^d$.

Note that $\tilde{f}_1$ and $\breve{f}_1$ use the quadratic feature $\norm{\cdot}^2$, in contrast to the norm feature $\norm{\cdot}$ used by the functions $\hat{f}$ and $\check{f}$ in \cref{sec:approxL}.
The next result shows that the quadratic feature $\norm{\cdot}^2$ can also be used to approximate non-smooth Lipschitz functions.
\begin{theorem} \label{thm:approxLeps}
  Let $\setX_\epsilon \subseteq \setX$ be an $\epsilon$-cover of a set $\setX \subset \setR^d$ \wrt\ $\norm{\cdot}$.
  Suppose that $f \in \setF_{\lip,\setX}$ for some Lipschitz constant $\lip > 0$, and define $\tilde{f}_0(\vx) \defeq \max_{\vxhat\in\setX_\epsilon} f(\vxhat) - (\lip/\epsilon) \norm{\vx-\vxhat}^2$ for all $\vx \in \setR^d$.
  Then, $-\lip\epsilon/4 \le f(\vx) - \tilde{f}_0(\vx) \le 2 \lip \epsilon$ for all $\vx\in\setX$.
\end{theorem}
\begin{proof}
  Choose $\vx \in \setX$ arbitrarily.
  By $f \in \setF_{\lip,\setX}$, we have
  {
    \setlength{\abovedisplayskip}{\dimexpr\abovedisplayskip-2mm\relax}
    \setlength{\belowdisplayskip}{\dimexpr\belowdisplayskip-2mm\relax}
    \setlength{\abovedisplayshortskip}{\dimexpr\abovedisplayshortskip-2mm\relax}
    \setlength{\belowdisplayshortskip}{\dimexpr\belowdisplayshortskip-2mm\relax}
    \begin{equation*}
      \tilde{f}_0(\vx) - f(\vx)
      = \max_{\vxhat\in\setX_\epsilon}f(\vxhat)-f(\vx) - (\lip/\epsilon)\norm{\vx-\vxhat}^2
      \le \max_{\vxhat\in\setX_\epsilon} \lip\norm{\vx-\vxhat}\Big(1 - \frac{\norm{\vx-\vxhat}}{\epsilon}\Big)
      \le \frac{\lip\epsilon}{4} .
    \end{equation*}
  }%
  For the other side, we have $\min_{\vxhat\in\setX_\epsilon}\norm{\vx-\vxhat} \le \epsilon$ by the $\epsilon$-covering property.
  The claim then follows from $f \in \setF_{\lip,\setX}$ as $f(\vx) - \tilde{f}_0(\vx) = \min_{\vxhat\in\setX_\epsilon} f(\vx) - f(\vxhat) + (\lip/\epsilon)\norm{\vx-\vxhat}^2 \le 2 \lip \epsilon$.
\end{proof}

\vspace{-3mm}
Again, the max-concave approximation $\tilde{f}_0$ has an analogous min-convex variant $\breve{f}_0$, defined as $\breve{f}_0(\vx) \defeq \min_{\vxhat\in\setX_\epsilon} f(\vxhat) + (\lip/\epsilon) \norm{\vx-\vxhat}^2$ for all $\vx\in\setR^d$.
\cref{fig:approxS} illustrates the approximations of this section on the smooth and non-smooth examples from \cref{fig:approxL}.
\begin{figure}[b]
  \begin{center}
    \fbox{\includegraphics[width=0.24\textwidth]{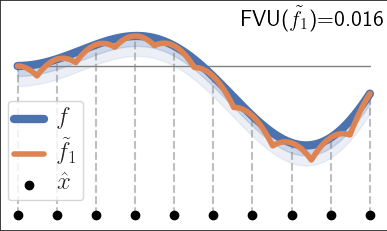}}
    \fbox{\includegraphics[width=0.24\textwidth]{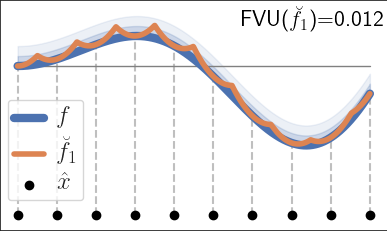}}
    \fbox{\includegraphics[width=0.24\textwidth]{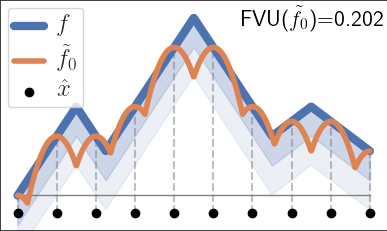}}
    \fbox{\includegraphics[width=0.24\textwidth]{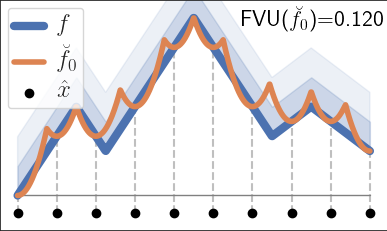}}
  \end{center}
  \vspace{-5mm}
  \caption{
    Approximation of a smooth function (left two plots) by $\tilde{f}_1$ and $\breve{f}_1$ of \cref{thm:approxS}, and of a non-smooth Lipschitz function (right two plots) by $\tilde{f}_0$ and $\breve{f}_0$ of \cref{thm:approxLeps}.
    The setting and notation are the same as in \cref{fig:approxL}, except that in the left two plots, the shaded areas indicate distances of $\nu\epsilon^2$ and $2\nu\epsilon^2$ from $f$.}
  \label{fig:approxS}
  \vspace{-1mm}
\end{figure}

Although it is straightforward to modify the DCF algorithm to use $\norm{\cdot}^2$ instead of $\norm{\cdot}$ in the function representation of $\setF_2(\setXhat_K)$, the near-minimax guarantee of \cref{thm:near-minimax-rate} does not carry over.
Our proof of \cref{thm:near-minimax-rate} does not extend to this case due to the challenge that the approximation functions $\tilde{f}_0$ and $\tilde{f}_1$ are only locally Lipschitz, and the coefficient of the quadratic feature $\norm{\cdot}^2$ in $\tilde{f}_0$ grows as $\lip/\epsilon$, which scales polynomially in $n$ since $\epsilon \approx n^{-1/(2+d_*)}$.
Addressing this issue remains an open direction for future research.

\subsection{Weakly Max-affine and Delta-max-affine Functions}
\label{sec:weakly-convex}

The functions $\tilde{f}_0$ and $\tilde{f}_1$ of \cref{thm:approxS,thm:approxLeps} are weakly convex in the sense of \citet{Vial1983}; that is, there exist $s_0, s_1 \ge 0$ such that $\tilde{f}_0 + s_0 q$ and $\tilde{f}_1 + s_1 q$ are convex functions, where $q$ is a symmetric convex quadratic function given by $q(\vx) \defeq \norm{\vx}^2$ for all $\vx\in\setR^d$.

Define the class of max-affine functions with at most $k_0 \in \setN$ hyperplanes by
\begin{equation*}
  \setM_{k_0} \defeq \Big\{ m : \setR^d \to \setR \,\Big|\, m(\vx) \defeq \max_{k\in[k_0]}b_k + \vw_k^\T\vx ,\, \vx\in\setR^d ,\, b_k\in\setR ,\, \vw_k\in\setR^d ,\, k\in[k_0]\Big\} .
\end{equation*}
Let the class of weakly max-affine functions be $\setM_{k_0}^w \defeq \{f \,|\, f \defeq m - s q ,\, m\in\setM_{k_0} ,\, s\in\setR\}$.
Furthermore, consider the closely related class of delta-max-affine functions, defined by $\setM_{k_0}^\Delta \defeq \{f \,|\, f \defeq m_1 - m_2 ,\, m_1,m_2\in\setM_{k_0}\}$.
The next result provides uniform approximation bounds for both of these classes.
For convenience, let $\setF_{\infty,\setX} \defeq \setF_{\infty,\setX}^\nabla \defeq \{ f \,|\, f : \setX \to \setR \}$.

\begin{corollary} \label{thm:weakly-delta}
  Let $\setX \subset \setR^d$ be a convex set, and suppose there exist $r, t > 0$ such that $N_{\norm{\cdot}}(\setX,\epsilon) \le (r/\epsilon)^t$ for all $\epsilon \in (0,r]$.
  Let $f \in \setF_{\lip,\setX} \cap \setF_{\nu,\setX}^\nabla$ for some constants $\lip, \nu \in (0,\infty]$.
  Then, for all $k_0 \in \setN$, there exist functions $f_w \in \setM_{k_0}^w$ and $f_\Delta \in \setM_{k_0}^\Delta$ such that
  \begin{equation*}
    \norm{f_w - f}_{\infty,\setX} \le 2 \epsilon \min\{\lip, \nu\epsilon\}
    ,\quad\qquad
    \norm{f_\Delta - f}_{\infty,\setX} \le 3 \epsilon \min\{\lip, \nu\epsilon\}
    ,\quad\qquad
    \epsilon \defeq r k_0^{-1/t} .
  \end{equation*}
\end{corollary}
\begin{proof}
  The choice $\epsilon = r k_0^{-1/t}$ ensures that $N_{\norm{\cdot}}(\setX,\epsilon) \le (r/\epsilon)^t = k_0$, which allows us to choose $\setXhat \subseteq \setX$ with $|\setXhat| = k_0$ such that it contains an $\epsilon$-cover of $\setX$ \wrt\ $\norm{\cdot}$.
  Suppose that $\{\lip,\nu\} \ne \{\infty\}$, otherwise the claim is trivial and vacuous.

  For finite $\lip$ and $\nu$, define the following weakly max-affine functions for all $\vx\in\setR^d$:
  \begin{equation*} \begin{aligned}
    \tilde{f}_0(\vx) &\defeq m_0(\vx) \hspace{-0.5mm}-\hspace{-0.5mm} (\lip/\epsilon) q(\vx) , & m_0(\vx) &\defeq \max_{\vxhat\in\setXhat}f(\vxhat) - (\lip/\epsilon)\norm{\vxhat}^2 + 2(\lip/\epsilon)\vxhat^\T\vx ,
    \\
    \tilde{f}_1(\vx) &\defeq m_1(\vx) \hspace{-0.5mm}-\hspace{-0.5mm} \nu q(\vx) , & m_1(\vx) &\defeq \max_{\vxhat\in\setXhat}f(\vxhat) - \nu\norm{\vxhat}^2 - \nabla f(\vxhat)^{\T}\vxhat + \big(\nabla f(\vxhat) \hspace{-0.5mm}+\hspace{-0.5mm} 2\nu\vxhat\big)\hspace{-0.5mm}^{\T}\hspace{-0.5mm}\vx .
  \end{aligned} \end{equation*}
Clearly, $\tilde f_0, \tilde f_1 \in \setM_{k_0}^w$, where $\tilde f_0$ coincides with the function constructed in \cref{thm:approxLeps} and $\tilde f_1$ with the one in \cref{thm:approxS}.
Set $\tilde{f}_1 \defeq \tilde{f}_0$ for $\nu = \infty$ and $\tilde{f}_0 \defeq \tilde{f}_1$ for $\lip = \infty$.
Then, the claimed upper bound on $\norm{f_w - f}_{\infty,\setX}$ follows directly for some $f_w \in \{\tilde{f}_0, \tilde{f}_1\}$.

Next, we approximate the quadratic function $q$ by a max-affine function formed as the maximum of the first-order Taylor approximations of $q$ at the points in $\setXhat$.
Define $\hat{m}(\vx) \defeq \max_{\vxhat\in\setXhat}-\norm{\vxhat}^2 + 2\vxhat^\T\vx$ for all $\vx\in\setR^d$.
Then, by the $\epsilon$-covering property of $\setXhat$, we have $q(\vx) - \hat{m}(\vx) = \min_{\vxhat\in\setXhat}\norm{\vx-\vxhat}^2 \in [0,\epsilon^2]$ for all $\vx\in\setX$.
Let $\delta \defeq \min\{\lip, \nu\epsilon\}$ and $m_w \in \{m_0, m_1\}$ such that $f_w \defeq m_w - (\delta/\epsilon)q \in \{\tilde{f}_0, \tilde{f}_1\}$ satisfies $\norm{f_w - f}_{\infty,\setX} \le 2\epsilon \delta$.
Define $f_\Delta \defeq m_w - (\delta/\epsilon)\hat{m}$, where $f_\Delta \in \setM_{k_0}^\Delta$.
The second claim then follows by using the triangle inequality as $\norm{f_\Delta - f}_{\infty,\setX} \le \norm{f_\Delta - f_w}_{\infty,\setX} + \norm{f_w - f}_{\infty,\setX} \le (\delta/\epsilon)\epsilon^2 + 2\epsilon \delta = 3 \epsilon \delta$.
\end{proof}

For estimator design in regression, weakly max-affine functions were studied by \citet{SunYu2019}, although without establishing convergence rates.
Using delta-max-affine functions, \citet{SiahkamariEtAl2020} proposed an estimator and proved a suboptimal convergence rate for the case in which the regression function is delta-convex.
Our proof techniques do not apply to either of these function classes, for the reasons explained at the end of \cref{sec:approxS}.
Nevertheless, \cref{thm:weakly-delta} may be useful for extending these developments and designing near-minimax estimators for smooth regression functions, which lies beyond the scope of this work.

\section{Variants of DCF}
\label{sec:variants}

As mentioned in \cref{sec:near-minimax-rate}, one can train over any of the function classes $\setFn(\cdot)$, $\setFn^-(\cdot)$, or $\setFn^\Delta(\cdot)$ with $\somenorm \in \{1,2,\infty,+\}$ using DCF (\cref{alg:DCF}).
While all these settings enjoy the near-minimax guarantee of \cref{thm:near-minimax-rate}, they can differ significantly in empirical performance, as shown in \cref{sec:experiments}.
We further illustrate this difference in \cref{fig:approx-fni} for the 1-dimensional examples discussed earlier in \cref{fig:approxL,fig:approxS}.
\begin{figure}[b!]
  \begin{center}
    $\setFn(\cdot)$ \hspace{32mm} $\setFn^-(\cdot)$ \hspace{30mm} $\setFn^\Delta(\cdot)$
    \\ \vspace{2mm}
    \includegraphics[width=0.305\textwidth]{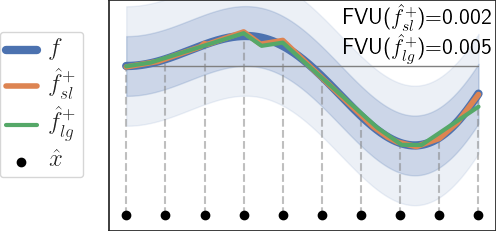}
    \hspace{4mm}
    \fbox{\includegraphics[width=0.24\textwidth]{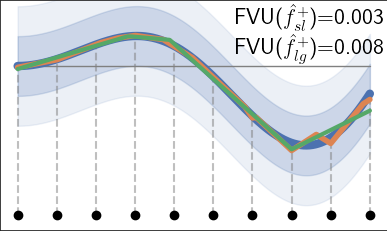}}
    \hspace{4mm}
    \fbox{\includegraphics[width=0.24\textwidth]{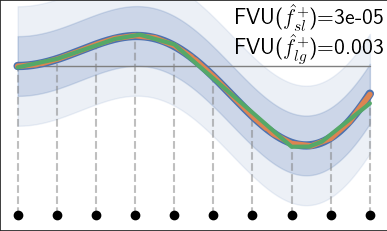}}
    \\ \vspace{2mm}
    \includegraphics[width=0.305\textwidth]{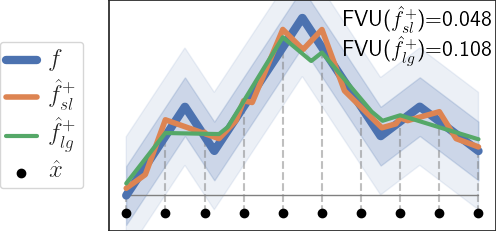}
    \hspace{4mm}
    \fbox{\includegraphics[width=0.24\textwidth]{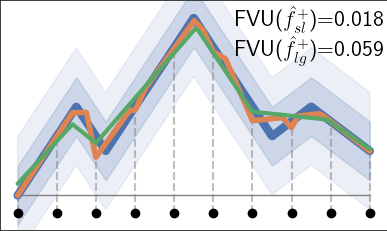}}
    \hspace{4mm}
    \fbox{\includegraphics[width=0.24\textwidth]{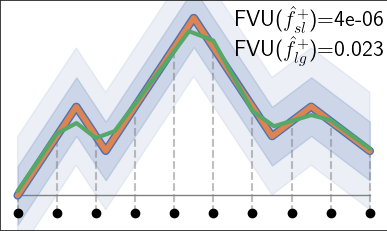}}
  \end{center}
  \vspace{-5mm}
  \caption{
    DCF approximations, each column showing the result for $\setFn(\cdot)$, $\setFn^-(\cdot)$, $\setFn^\Delta(\cdot)$, respectively.
    We used the norms $\somenorm \in \{1,2,\infty\}$ which are equivalent for $d = 1$.
    The settings and the notations are the same as in \cref{fig:approxL}.
    DCF uses the same parameters as in \cref{sec:experiments}, with $\theta_2 = (R_{\setX_n}/n)^2$ for $\hat{f}^+_{\textrm{sl}}$ and $\theta_2 = R_{\setX_n}^2/n$ for $\hat{f}^+_{\textrm{lg}}$.
  }
  \label{fig:approx-fni}
\end{figure}
The plots show that $\setFn(\cdot)$ struggles to approximate the concave regions, while $\setFn^{-}(\cdot)$ struggles with convex regions.
In contrast, the symmetric representation in $\setFn^\Delta(\cdot)$ overcomes both issues and achieves significantly better approximation accuracy.
In all cases, the accuracy is an order of magnitude better than that of the worst-case optimal approximation functions from \cref{thm:approxL} in \cref{fig:approxL}.
The plots also illustrate how the choice of the regularizer $\theta_2$ controls the ``smoothness'' of the estimator.
In these noise-free settings, smaller values of $\theta_2$ yield better results; however, in the noisy problems discussed in \cref{sec:experiments}, such choices can lead to overfitting.

Recall the definition $\setFn^-(\setXhat_K) \defeq \{-f : f \in \setFn(\setXhat_K)\}$ from \cref{sec:dcf-discuss}.
Training over $\setFn^-(\setXhat_K)$ requires only a minor modification to DCF: one can flip the sign of the response variables $Y_1,\ldots,Y_n$ during training, and then flip the sign of the final estimator $\fni$ at the end.
The case of $\setFn^\Delta$ is slightly more involved, and we discuss it in detail in the next section.

\subsection{Symmetric Representations}
\label{sec:symmetric}

We now describe the modifications to DCF (\cref{alg:DCF}) needed to work with the symmetric class $\setFn^\Delta(\setXhat_K) \defeq \{f\,|\,f \defeq f_1 - f_2 ,\, f_1,f_2\in\setFn(\setXhat_K)\}$.

As before, we compute a single set of center points $\setXhat_K$ using AFPC (\cref{alg:AFPC}).
The main difference lies in modifying \eqref{eq:erm} to use two sets of parameters, $\langle (b_k,\vw_k) : k\in[K] \rangle$ and $\langle (b'_k,\vw'_k) : k\in[K] \rangle$, as shown in the following:
\begin{equation} \label{eq:erm-symm} \begin{split}  \raisetag{4mm}
  &\min_{\substack{\reg \in \setR, \\ b_1,\ldots,b_K \in \setR, \\ \vw_1,\ldots,\vw_K \in \setR^{d_\somenorm}, \\ b'_1,\ldots,b'_K \in \setR, \\ \vw'_1,\ldots,\vw'_K \in \setR^{d_\somenorm}}}
  \hspace{-5mm}
  \theta_1 \reg^2 + \sum_{k\in[K]} \theta_2\big(\norm{\vw_k}^2 + \norm{\vw'_k}^2\big) \\[-12mm]
  & \hspace{35mm} + \frac1n \sum_{i\in[n]} \ind\big\{\vX_i \in \setC_k(\setXhat_K)\big\} \Big(b_k - b_k' + \phin(\vX_i,\vXhat_k)^\T(\vw_k - \vw_k') - Y_i\Big)^2  \\
  &\qquad \textrm{such that for all }\, k,l\in[K]:
    b_k \ge b_l + \phin(\vXhat_k,\vXhat_l)^\T\vw_l ,\quad \norm{\vw_k} \le \reg + \theta_0 ,\\
  & \hspace{55.25mm} b'_k \ge b'_l + \phin(\vXhat_k,\vXhat_l)^\T\vw'_l ,\quad \norm{\vw'_k} \le \reg + \theta_0 .
\end{split} \end{equation}

Let the solution of \eqref{eq:erm-symm} be $\big(\reg_n, \langle (b_{n,k},\vw_{n,k},b'_{n,k},\vw'_{n,k}) : k\in[K] \rangle\big)$, and define the (initial) estimator by $f_n \defeq f_{n,1} - f_{n,2}$, where $f_{n,1}(\vx) \defeq \max_{k\in[K]}b_{n,k} + \phin(\vx,\vXhat_k)^\T\vw_{n,k}$, and $f_{n,2}(\vx) \defeq \max_{k\in[K]}b'_{n,k} + \phin(\vx,\vXhat_k)^\T\vw'_{n,k}$, for all $\vx\in\setR^d$.
Clearly, $f_{n,1}, f_{n,2} \in \setFn(\setXhat_K)$, and we have $f_n \in \setFn^\Delta(\setXhat_K)$.

Next, the final step refines the estimator $f_n$ to $\fnih \,\in\, \setFn^\Delta(\setXhat_K)$ that satisfies \eqref{eq:erm-local}, where the regularizer
{
  \setlength{\abovedisplayskip}{\dimexpr\abovedisplayskip-2mm\relax}
  \setlength{\abovedisplayshortskip}{\dimexpr\abovedisplayshortskip-2mm\relax}
  \begin{equation} \label{eq:erm-local-symm} \begin{split}
    \regz_{c_0,c_1,c_2}(f) \defeq c_1\max_{k\in[K],j\in[2]}\big(\norm{\vw_{f_j,k}} - c_0\big)_+^2 + c_2\sum_{k\in[K],j\in[2]}\norm{\vw_{f_j,k}}^2
  \end{split} \end{equation}
}%
is defined for all $c_0, c_1, c_2 \ge 0$, and for all $f \defeq f_1 - f_2 \in \setFn^\Delta(\setXhat_K)$ with $f_1, f_2 \in \setFn(\setXhat_K)$.
Here, we use $\lipf_{f} \defeq \max_{k\in[K],j\in[2]}\norm{\vw_{f_j,k}}$ to define $\regz_n(\cdot)$ in \eqref{eq:erm-local}.

Suppose the refined estimator $\fnih$ is expressed as $\fnih \defeq \hat{f}_{n,1}^+ - \hat{f}_{n,2}^+$ with some functions $\hat{f}_{n,1}^+, \hat{f}_{n,2}^+ \in \setFn(\setXhat_K)$.
Then, the final estimator $\fni \defeq f_{n,1}^+ - f_{n,2}^+$ is defined for all $\vx\in\setR^d$ as
\begin{equation} \label{eq:dcf-estimator-symm} \begin{split}
  f_{n,j}^+(\vx) \defeq \big(\tfrac{3}{2}\hspace{-0.5mm}-\hspace{-0.5mm}j\big)C_n^+ + \hspace{-0.5mm}\max_{k\in\setI_{n,j}^+} \hspace{-0.5mm} b_{\hat{f}_{n,j}^+,k}\hspace{-0.5mm}-\hspace{-0.5mm}C_\Delta^+ + \phin(\vx,\vXhat_k)^\T\vw_{\hat{f}_{n,j}^+,k} ,\qquad j\in[2] ,
\end{split} \end{equation}
where $\setI_{n,j}^+ \defeq \setI_n(\hat{f}_{n,j}^+)$, $C_n^+$ and $\setI_n(\cdot)$ are defined as in \eqref{eq:dcf-estimator}, and the average bias term $C_\Delta^+$ is given by $C_\Delta^+ \defeq  \frac12 \sum_{j\in[2]}\frac{1}{|\setI_{n,j}^+|}\sum_{k\in\setI_{n,j}^+}b_{\hat{f}_{n,j}^+,k}$.

Let $b^+_{j,k} \defeq (\frac32-j)C_n^+ + b_{\hat{f}_{n,j}^+,k} - C_\Delta^+$ be the bias parameters used in defining $\fni$, for all $j\in[2]$ and $k\in\setI_{n,j}^+$.
Define the average bias terms as $\bar{b}_{n,j}^+ \defeq \frac{1}{|\setI_{n,j}^+|}\sum_{k\in\setI_{n,j}^+} b^+_{j,k}$ for all $j\in[2]$.
Note that subtracting the constant $C_\Delta^+$ from all the bias parameters leaves the function $\fni$ unchanged and ensures that $\bar{b}_{n,1}^+ + \bar{b}_{n,2}^+ = 0$.

We claim that the theoretical guarantee of \cref{thm:near-minimax-rate} extends to DCF estimators based on the symmetric set $\setFn^\Delta(\cdot)$.
The derivation in \cref{sec:analysis} can be straightforwardly adapted to this case, except for one detail concerning the bounding of the unregularized bias parameters in \cref{thm:loose-param-bounds,thm:fni-bias-bound}.
The techniques presented in \cref{sec:proof-near-minimax-rate} yield bounds on differences between pairs of bias parameters, rather than on individual ones.
To decouple these bounds, we apply the following result.

\vspace{-2mm}
\begin{lemma} \label{thm:bias-sep}
  Let $n, m \in \setN$, and $a_1, \ldots, a_n , b_1, \ldots, b_m\in \setR$.
  Define $\bar{a} \defeq (1/n)\sum_{i=1}^n a_i$ and $\bar{b} \defeq (1/m)\sum_{j=1}^m b_j$.
  Suppose that $\bar{a} + \bar{b} = 0$, and that there exist $c \in \setR$ and $\bnd > 0$ such that $\max_{i\in[n],j\in[m]}|a_i - b_j - c| \le \bnd$.
  Then, $\max\big\{\max_{i\in[n]}|a_i - c/2|, \max_{j\in[m]}|b_j+c/2|\big\} \le 3\bnd/2$.
\end{lemma} \vspace{-2mm} %
\begin{proof}
  By Jensen's inequality, we get $|\bar{a} - \bar{b} - c| \le \frac{1}{n m}\sum_{i=1}^n\sum_{j=1}^m|a_i - b_j - c| \le \bnd$.
  Combining this with $\bar{a} = -\bar{b}$ yields $|\bar{a} - c/2| \le \bnd/2$ and $|\bar{b} + c/2| \le \bnd/2$.
  Then, by using the reverse triangle and Jensen's inequalities, we get for all $i\in[n]$ that $|a_i - c/2| - |\bar{b} + c/2| \le |a_i - \bar{b} - c| \le \bnd$, which implies $\max_{i\in[n]}|a_i - c/2| \le 3\bnd/2$.
  The other claim, $\max_{j\in[m]}|b_j+c/2| \le 3\bnd/2$, follows analogously.
\end{proof}

\vspace{-4mm}
Similarly to the proof of \cref{thm:fni-bias-bound}, one can show that for every $k \in \setI_{n,1}^+$ there exists $l \in \setI_{n,2}^+$ such that $|b^+_{1,k} - b^+_{2,l} - y_0| \le \bnd_0$, and symmetrically with the roles of the two components interchanged.
To lift this to all index pairs, we bound the spread of the bias parameters within each component.
By the definition of $\setI_{n,j}^+$, for all $j\in[2]$ and $k\in\setI_{n,j}^+$ there exists $i_k \in [n]$ with $f_{n,j}^+(\vX_{i_k}) = b^+_{j,k} + \vw_{f_{n,j}^+,k}^\T\phin(\vX_{i_k},\vXhat_k)$.
Hence, for all $j\in[2]$ and $k, l \in \setI_{n,j}^+$, the triangle and Cauchy-Schwarz inequalities, \cref{thm:phi-props}, the Lipschitz continuity of $f_{n,j}^+$, and $\setXhat_K \subseteq \setX_n$ yield
{
  \setlength{\abovedisplayskip}{\dimexpr\abovedisplayskip-2mm\relax}
  \setlength{\belowdisplayskip}{\dimexpr\belowdisplayskip-3mm\relax}
  \setlength{\abovedisplayshortskip}{\dimexpr\abovedisplayshortskip-2mm\relax}
  \setlength{\belowdisplayshortskip}{\dimexpr\belowdisplayshortskip-3mm\relax}
  \begin{equation*} \begin{split}
    |b^+_{j,k} - b^+_{j,l}|
    &\le |f_{n,j}^+(\vX_{i_k}) - f_{n,j}^+(\vX_{i_l})| +
      \norm{\vw_{f_{n,j}^+,k}}\norm{\phin(\vX_{i_k},\vXhat_k)} +
      \norm{\vw_{f_{n,j}^+,l}}\norm{\phin(\vX_{i_l},\vXhat_l)}
    \\[-1mm]
    &\le (\lipphi + 2\cphi) \max_{k\in[K]}\norm{\vw_{f_{n,j}^+,k}} \max_{i,i'\in[n]}\norm{\vX_i - \vX_{i'}}
    \\[-3mm]
    &= \Ordo(\bnd_0) .
  \end{split} \end{equation*}
}%
Combining the two bounds through the triangle inequality gives $|b^+_{1,k} - b^+_{2,l} - y_0| = \Ordo(\bnd_0)$ for all $k\in\setI_{n,1}^+$ and $l\in\setI_{n,2}^+$, which we then separate using \cref{thm:bias-sep} with $\bar{b}_{n,1}^+ + \bar{b}_{n,2}^+ = 0$.

The bias parameters of $f_n$ can be centered to satisfy $0 = \sum_{k\in[K]} b_{n,k} + b'_{n,k}$ without changing the function.
Analogously to the proof of \cref{thm:loose-param-bounds}, we get $|b_{n,k} - b'_{n,k} - y_0| = \Ordo(\bnd_1)$ for all $k \in [K]$.
To extend this to all pairs, we use the constraints in \eqref{eq:erm-symm}.
Specifically, for each pair $k, l \in [K]$, we have $b_k \ge b_l + \phin(\vXhat_k,\vXhat_l)^\T\vw_l$ and $b_l \ge b_k + \phin(\vXhat_l,\vXhat_k)^\T\vw_k$, which imply $|b_{n,k} - b_{n,l}| = \Ordo\big(\cphi R_{\setX_n} \max_{k'\in[K]}\norm{\vw_{n,k'}}\big) = \Ordo(\bnd_1)$.
The same reasoning applies to $\{b'_{n,k} : k\in[K]\}$.
We then use the triangle inequality to bound the distance between all pairs as $|b_{n,k}-b'_{n,l}-y_0| = \Ordo(\bnd_1)$ for all $k,l\in[K]$, and invoke \cref{thm:bias-sep}.
The remaining details are straightforward and omitted for brevity.

\vspace{-2mm}
\subsection{Adapting to Convex Shape-restricted Regression}
\label{sec:convex-regression}

We now consider the setting of convex (shape-restricted) regression, where the $\lip_*$-Lipschitz regression function $f_*$ is known to be convex.
The goal is to estimate $f_*$ with a convex function.

Let $\setFcvx_{\lip,\setX} \defeq \big\{f \in \setF_{\lip,\setX} : f \textrm{ is convex}\big\}$ denote the set of $\lip$-Lipschitz, convex functions on a convex set $\setX \subseteq \setR^d$.
Let $\nabla f(\vx)$ denote an arbitrary but fixed subgradient of $f$ at $\vx\in\setR^d$.
We consider the statistical model \eqref{eq:data-model}, modified so that $f_* \in \setFcvx_{\lip_*,\setX_*}$, for an unknown Lipschitz constant $\lip_* > 0$ and an unknown convex domain $\setX_* \subseteq \setR^d$.
In this setting, DCF (\cref{alg:DCF}) can be applied by restricting $\setFn(\setXhat_K)$ to convex functions, where the center points $\setXhat_K \defeq \{\vxhat_1,\ldots,\vxhat_K\}$ are computed by AFPC (\cref{alg:AFPC}).

The max-affine representation is by far the most widely used in convex regression.
DCF can be readily adapted to this case by disabling the ``norm feature'' and restricting $\setFn(\setXhat_K)$ to the class of max-affine functions as
{
  \setlength{\abovedisplayskip}{\dimexpr\abovedisplayskip-2mm\relax}
  \setlength{\belowdisplayskip}{\dimexpr\belowdisplayskip-2mm\relax}
  \setlength{\abovedisplayshortskip}{\dimexpr\abovedisplayshortskip-2mm\relax}
  \setlength{\belowdisplayshortskip}{\dimexpr\belowdisplayshortskip-2mm\relax}
  \begin{equation} \label{eq:max-affine-match}
    \setM_K = \Big\{f \in \setFn(\setXhat_K)
    : \vw_{f,k}^\T = [\vu_{f,k}^\T \,\, v_{f,k}] ,\, \vu_{f,k} \in \setR^d ,\, v_{f,k} = 0 ,\, k\in[K]\Big\}
    ,
  \end{equation}
}%
which holds for all $\somenorm \in \{1,2,\infty\}$.

To adapt the proof of \cref{thm:near-minimax-rate} to the convex regression setting, we replace \cref{thm:approxL} with the following approximation result, which slightly extends the result of \citet{Balazs2022}.

\begin{theorem} \label{thm:approxLcvx}
  Let $\setX_\epsilon$ be an $\epsilon$-cover of a convex set $\setX \subset \setR^d$.
  Suppose $f \in \setFcvx_{\lip,\setX}$ for some Lipschitz constant $\lip > 0$, and let $\hat{m}(\vx) \defeq \max_{\vxhat\in\setX_\epsilon} f(\vxhat) + \nabla f(\vxhat)^\T(\vx-\vxhat)$ for all $\vx\in\setR^d$.
  Then, for all $\vx\in\setX$ and $\vxtilde\in\setX_\epsilon$,
  \begin{equation*}
    0 \le f(\vx) - \hat{m}(\vx) \le 2\lip\epsilon
    ,\quad\qquad \hat{m}(\vxtilde) = f(\vxtilde)
    ,\quad\qquad \hat{m} \in \setFcvx_{\lip,\setR^d} \cap \setM_{|\setX_\epsilon|}.
  \end{equation*}
\end{theorem}
\begin{proof}
  The convexity of $f$ directly implies $0 \le f - \hat{m}$.
  Moreover, since $f \in \setF_{\lip,\setX}$, we get for all $\vx\in\setX$ that
  \begin{equation} \label{eq:cvx-m-ineq}
    f(\vx) - \hat{m}(\vx) = \min_{\vxhat\in\setX_\epsilon}f(\vx) - f(\vxhat) - \nabla f(\vxhat)^\T(\vx-\vxhat) \le 2\lip\min_{\vxhat\in\setX_\epsilon}\norm{\vx-\vxhat} ,
  \end{equation}
  which implies $f(\vx)-\hat{m}(\vx) \le 2\lip\epsilon$ for all $\vx\in\setX$ by the $\epsilon$-covering property of $\setX_\epsilon$.
  Furthermore, choosing $\vx = \vxtilde \in \setX_\epsilon$ in \eqref{eq:cvx-m-ineq}, we get $f(\vxtilde) - \hat{m}(\vxtilde) \le 0$, implying the second claim.
  The third claim follows because the max function is $1$-Lipschitz, and $\norm{\nabla f(\vxhat)} \le \lip$ for all $\vxhat \in \setX_\epsilon$ as $\setX_\epsilon \subseteq \setX$ and $f \in \setF_{\lip,\setX}$.
\end{proof}

Since \cref{thm:approxLcvx} delivers the same $\Ordo(\lip\epsilon)$ approximation accuracy as \cref{thm:approxL}, one may apply DCF with the class of max-affine functions $\setM_K$ to achieve the near-minimax rate of \cref{thm:near-minimax-rate} in the convex regression setting.
This matches the theoretical guarantees of the APCNLS algorithm of \citet{Balazs2022}, but unlike APCNLS, DCF does not require knowledge of the Lipschitz constant $\lip_*$.
Moreover, the optimization problem \eqref{eq:erm} in DCF uses $K^2$ constraints, which provides a substantial reduction in computational burden compared to the $nK$ constraints used in APCNLS.

Motivated by the experimental results of \cref{sec:experiments}, it may be beneficial to use a richer function class than the max-affine one.
In particular, one can allow $v_{f,k} \ge 0$ for all $k\in[K]$ instead of enforcing $v_{f,k} = 0$ as in \eqref{eq:max-affine-match}, which still ensures that the set $\setFn(\setXhat_K)$ is restricted to convex functions for all $\somenorm \in \{1,2,\infty\}$.
Moreover, one can restrict the set $\setF_+(\setXhat_K)$ to convex functions using the following simple result:

\begin{lemma} \label{thm:cvx-plus}
    Fix $\hat{x}, u, v \in \setR$, and define $f(x) \defeq (x-\hat{x})_+ u + (\hat{x}-x)_+ v$ for all $x \in \setR$.
  Then $f$ is convex on $\setR$ if and only if $u \ge -v$.
\end{lemma}
\begin{proof}
  Write $f(x) = (x-\hat{x})_+(u+v) - \big((x-\hat{x})_+ - (\hat{x}-x)_+\big)v = (x-\hat{x})_+(u+v) - (x-\hat{x})v$ which is convex on $\setR$ if $u + v \ge 0$ and concave otherwise by the convexity of $(\cdot)_+$.
\end{proof}

Then, by \cref{thm:cvx-plus}, the above-mentioned restriction to convex functions can be formalized by using an extra linear constraint as
\begin{equation*}
  \setF_+^{\textrm{cvx}}(\setXhat_K) \defeq \Big\{ f \in \setF_+(\setXhat_K) : \vw_{f,k}^\T = [\vu_{f,k}^\T\,\,\vv_{f,k}^\T],\, \vu_{f,k} \ge -\vv_{f,k},\, \vu_{f,k},\vv_{f,k}\in\setR^d,\,k\in[K]\Big\} .
\end{equation*}
A detailed analysis of when these extended convex function classes provide performance gains compared to max-affine functions is left for future research.

\section{Conclusions}
\label{sec:conclusion}

We introduced the polynomial-time DCF algorithm, which decomposes the nonparametric estimation of a Lipschitz function into three steps: a partitioning step, an initial convex fitting step over the resulting partition, and an optional refinement step applied to the initial solution.
As shown in \cref{thm:near-minimax-rate}, DCF achieves the adaptive near-minimax rate in our setting, capturing the intrinsic dimension of the covariate space without relying on external model selection procedures.

Our empirical results show that DCF, when equipped with an appropriately chosen regularization parameter $\theta_2$, can be competitive with state-of-the-art methods and can outperform other theoretically justified algorithms.
However, its sensitivity to $\theta_2$, together with its computationally intensive training procedure, highlights an important direction for future research.

\acks{%
We thank Csaba Szepesv\'ari for many helpful comments on an early draft of the paper.
This work was funded by \textsc{G\&G} (Mariana Gema and the author).
}

\vskip 0.2in
\bibliography{dcf}

\end{document}